%% file: main.tex
\documentclass[11pt]{article}
\pdfoutput=1
\usepackage{smile}

\usepackage{amsmath,amssymb}
\usepackage{bm}
\usepackage{natbib}
\usepackage[usenames]{color}
\usepackage{amsthm}

\usepackage{multirow} 
\usepackage{enumitem}

\usepackage[colorlinks,
linkcolor=red,
anchorcolor=blue,
citecolor=blue
]{hyperref}

\usepackage{setspace}
\usepackage[left=1in, right=1in, top=1in, bottom=1in]{geometry}

\usepackage{xcolor}

\def\Diag{\mathrm{Diag}}
\def\bbvec{\mathrm{\mathbf{b}}}

\def\cN{\mathcal{N}}

\newcommand{\la}{\langle}
\newcommand{\ra}{\rangle}

\begin{document}

\title{\huge Stochastic Gradient Descent Optimizes\\ Over-parameterized Deep ReLU Networks}

\author
{
Difan Zou\thanks{Equal contribution}~\thanks{Department of Computer Science, University of California, Los Angeles, CA 90095, USA; e-mail: {\tt knowzou@cs.ucla.edu}}
	~~~and~~~
Yuan Cao\footnotemark[1]~\thanks{Department of Computer Science, University of California, Los Angeles, CA 90095, USA; e-mail: {\tt yuancao@cs.ucla.edu}}
~~~and~~~
	Dongruo Zhou\thanks{Department of Computer Science,
  University of California,
  Los Angeles, CA 90095, USA;
  e-mail: {\tt drzhou@cs.ucla.edu}} 
	~~~and~~~
	Quanquan Gu\thanks{Department of Computer Science,
  University of California,
  Los Angeles, CA 90095, USA; e-mail: {\tt qgu@cs.ucla.edu}}
}
%\date{}

% \author
% {
% Difan Zou\thanks{Equal contribution}~\thanks{Department of Computer Science, University of California, Los Angeles, CA 90095, USA; e-mail: {\tt knowzou@cs.ucla.edu}}
% 	~~~and~~~
% Yuan Cao\footnotemark[1]~\thanks{Department of Computer Science, University of California, Los Angeles, CA 90095, USA; e-mail: {\tt yuancao@cs.ucla.edu}}
% ~~~and~~~
% 	Dongruo Zhou\thanks{Department of Computer Science,
%   University of California,
%   Los Angeles, CA 90095, USA;
%   e-mail: {\tt drzhou@cs.ucla.edu}} 
% 	~~~and~~~
% 	Quanquan Gu\thanks{Department of Computer Science,
%   University of California,
%   Los Angeles, CA 90095, USA; e-mail: {\tt qgu@cs.ucla.edu}}
% }
%\date{May 18, 2018}
\date{}
\maketitle

\begin{abstract}%
We study the problem of training deep neural networks with Rectified Linear Unit (ReLU) activation function using gradient descent and stochastic gradient descent. In particular, we study the binary classification problem and show that for a broad family of loss functions, with proper random weight initialization, both gradient descent and stochastic gradient descent can find the global minima of the training loss for an over-parameterized deep ReLU network, under mild assumption on the training data. The key idea of our proof is that Gaussian random initialization followed by (stochastic) gradient descent produces a sequence of iterates that stay inside a small perturbation region centering around the initial weights, in which the empirical loss function of deep ReLU networks enjoys nice local curvature properties that ensure the global convergence of (stochastic) gradient descent. Our theoretical results shed light on understanding the optimization for deep learning, and pave the way for studying the optimization dynamics of training modern deep neural networks.
%objective function enjoys the XXX achieves good convergence rate
\end{abstract}
% \begin{keywords}
%   deep neural networks, ReLU activation, stochastic gradient descent, over-parameterization, random initialization
% \end{keywords}

\input{intro}

\section{Problem Setup and Preliminaries}\label{sec:preliminaries}

\subsection{Notation}
%\CC{need to revise this part based on the notations used in the paper}
We use lower case, lower case bold face, and upper case bold face letters to denote scalars, vectors and matrices respectively. For a positive integer $n$, we denote $[n] = \{1,\dots,n\}$. For a vector $\xb = (x_1,\dots,x_d)^\top$, we denote by $\|\xb\|_p=\big(\sum_{i=1}^d |x_i|^p\big)^{1/p}$ the $\ell_p$ norm of $\xb$, $\|\xb\|_\infty = \max_{i=1,\dots,d} |x_i|$ the $\ell_\infty$ norm of $\xb$, and $\|\xb\|_0 = |\{x_i:x_i\neq 0,i=1,\dots,d\}|$ the $\ell_0$ norm of $\xb$. We use $\text{Diag}(\xb)$ to denote a square diagonal matrix with the elements of vector $\xb$ on the main diagonal. For a matrix $\Ab = (A_{ij})\in \RR^{m\times n}$, we use $\|\Ab\|_F$ to denote the Frobenius norm of $\Ab$, $\|\Ab\|_2$ to denote the spectral norm (maximum singular value), and $\|\Ab\|_0$ to denote the number of nonzero entries. We denote by $S^{d-1} = \{ \xb\in\RR^d:\| \xb\|_2 =1\}$ the unit sphere in $\RR^d$.

For two sequences $\{a_n\}$ and $\{b_n\}$, we use $a_n = O(b_n)$ to denote that $a_n\le C_1 b_n$ for some absolute constant $C_1> 0$, and use $a_n = \Omega (b_n)$ to denote that $a_n\ge C_2 b_n$ for some absolute constant $C_2>0$. In addition, we also use $\tilde O(\cdot)$ and $\tilde \Omega(\cdot)$ to hide some logarithmic terms in Big-O and Big-Omega notations.
% For a matrix $\Ab = [A_{ij}] \in \RR^{d\times d}$, we define $\|\Ab\|_{1,1} = \sum_{i,j=1}^d |A_{ij}|$. 
% Given two sequences $\{a_n\}$ and $\{b_n\}$, we write $a_n = O(b_n)$ if there exists a constant $0 < C < +\infty$ such that $a_n \leq C\, b_n$, and we write $a_n = \Omega(b_n)$ if there exists a constant $0 < c < +\infty$ such that $a_n \geq c\, b_n$.  We use notation $\tilde{O}(\cdot)$ and $\tilde{\Omega}(\cdot)$ to hide logarithmic factors in Big-O and Big-Omega notations.
We also use the following matrix product notation. For indices $l_1,l_2$ and a collection of matrices $\{\Ab_r\}_{r\in \ZZ_+}$, we denote
\begin{align}\label{eq:matrixproductnotation}
    \prod_{r = l_1}^{l_2} \Ab_r :=\left\{
    \begin{array}{ll}
        \Ab_{l_2}\Ab_{l_2-1} \cdots  \Ab_{l_1} & \text{if }l_1\leq l_2 \\
        \Ib & \text{otherwise.}
    \end{array}
    \right.
\end{align}

\subsection{Problem Setup}
Let $\{(\xb_1,y_1),\ldots,(\xb_n,y_n)\} \in (\RR^d\times \{-1,1\})^n$ be a set of $n$ training examples. %and $y_1,\ldots,y_n\in \{0,1\}$ be the corresponding labels. 
Let $m_0 = d$. We consider $L$-hidden-layer neural networks as follows:
\begin{align*}
    f_{\Wb}(\xb) = \vb^\top \sigma ( \Wb_{L}^{\top} \sigma ( \Wb_{L-1}^{\top} \cdots \sigma( \Wb_{1}^{\top} \xb )\cdots)),
\end{align*}
where $\sigma(\cdot)$ is the entry-wise ReLU activation function, $\Wb_{l} = (\wb_{l,1},\ldots,\wb_{l,m_l}) \in \RR^{m_{l-1}\times m_{l}}$, $l=1,\ldots,L$ are the weight matrices, and $\vb\in \{-1,+1\}^{m_L}$ is the fixed output layer weight vector with half $1$ and half $-1$ entries. Let $\Wb=\{\Wb_l\}_{l=1,\dots,L}$ be the collection of matrices $\Wb_1,\dots,\Wb_L$, we consider solving the following empirical risk minimization problem:
\begin{align}\label{eq:minimizationproblemdefinition}
    \min_{\Wb} L_S(\Wb) = \frac{1}{n} \sum_{i=1}^n \ell(y_i \hat{y}_i), 
\end{align}
where $\hat{y}_i = f_{\Wb}(\xb_i)$. Regarding the loss function $\ell(\xb)$, we make the following assumptions.

\begin{assumption}\label{assump:loss}
The loss function $\ell(\cdot)$ is continuous, and satisfies $\ell'(x)\le 0$, $\lim_{x\rightarrow\infty}\ell(x)=0$ as well as $\lim_{x\rightarrow\infty}\ell'(x)=0$.
\end{assumption}
Assumption~\ref{assump:loss} has been widely made in the studies of training binary classifiers \citep{soudry2017implicit,nacson2018stochastic,ji2018gradient}. In addition, we require the following assumption which provides an upper bound on the derivative of $\ell(\cdot)$.
\begin{assumption}\label{assump:derivative_loss}
There exists positive constants $\alpha_0$ and $\alpha_1$, such that for any $x\in\RR$ we have
\begin{align*}
-\ell'(x)\ge \min\{\alpha_0,\alpha_1\ell^p(x)\}
\end{align*}
where $p\le 1$ is a positive constant. Note that $a_0$ can be $+\infty$.
\end{assumption}
This assumption holds for a large class of loss functions including hinge loss, cross-entropy loss and exponential loss. It is worthy noting that when $\alpha_0 = +\infty$ and $p=1/2$, this reduces to Polyak-{\L{}uk}ojasiewicz (PL) condition \citep{polyak1963gradient}.

% This assumption holds for most of the popular loss functions, such as hinge loss $\ell(x) =\max\{0,1-x\}$, squared hinge loss $\ell(x) = \big(\max\{0,1-x\}\big)^2$, exponential loss $\ell(x) = \exp(-x)$, and cross entropy loss $\ell(x) = \log\big(1+\exp(-x)\big)$.

% \begin{assumption}\label{assump:Lipschitz}
% The loss function $\ell(\cdot)$ is $\rho$-Lipschitz, i.e., $|\ell'(x)|\le \rho $ for all $x$.
% \end{assumption}
 
\begin{assumption}\label{assump:smooth}
The loss function $\ell(\cdot)$ is $\lambda$-smooth, i.e., $|\ell''(x)|\le \lambda $ for all $x$.
\end{assumption}
% The above assumption is general for analyzing the convergence property of (stochastic) gradient descent, which has also been made in \cite{li2018learning}. 

In addition, we make the following assumptions on the training data.
\begin{assumption}\label{assump:normalizeddata}
$\| \xb_i \|_2 = 1$ and $(\xb_i)_d = \mu$ for all $i\in \{1,\ldots,n\}$, where $\mu \in ( 0, 1)$ is a constant.
\end{assumption}
As is shown in the assumption above, the last entry of input $\xb$ is considered to be a constant $\mu$, which introduces the bias term in the input layer of the network. 

\begin{assumption}\label{assump:separateddata}
For all $i,i'\in \{1,\ldots,n\}$, if $y_i \neq y_{i'}$, then $\| \xb_i - \xb_{i'} \|_2  \geq \phi$ for some $\phi>0$.
\end{assumption}
Assumption~\ref{assump:separateddata} is a weaker version of Assumption~2.1 in \cite{allen2018convergence}, which assumes that every two different data points are separated by a constant. In comparison, Assumption~\ref{assump:separateddata} only requires that inputs with different labels are separated, which is a much more practical assumption since it holds for all data distributions with margin $\phi$, while the data separation distance in \cite{allen2018convergence} is usually dependent on the sample size $n$ when the examples are generated independently.% \CC{XX}

%Assumption \ref{assump:separateddata} was inspired by the assumption made in \citet{li2018learning}, where they assumed that the training data in each class are supported on a mixture of components, and any two components are separated by a certain distance. In this assumption, we treat each training examples as a component and define by $\phi$ the minimum distance between any two examples. The same assumption was also made in \citet{allen2018convergence}.

\begin{assumption}\label{assump:m_scaling}
Define $M = \max\{ m_1,\ldots, m_L\}$, $m = \min\{ m_1,\ldots, m_L\}$. We assume that $M \leq 2m$.
\end{assumption}
Assumption~\ref{assump:m_scaling} states that the number of nodes at all layers are of the same order. The constant $2$ is not essential and can be replaced with an arbitrary constant greater than or equal to $1$.

\subsection{Optimization Algorithms}
In this paper, we consider training a deep neural network with Gaussian initialization followed by gradient descent/stochastic gradient descent. 

\noindent\textbf{Gaussian initialization.} We say that the weight matrices $\Wb_{1}, \ldots, \Wb_{L}$ are generated from Gaussian initialization if each column of $\Wb_{l}$ is generated independently from the Gaussian distribution $N(\mathbf{0},2/m_l \Ib)$ for all $l=1,\ldots,L$.

\noindent\textbf{Gradient descent.} We consider solving the empirical risk minimization problem \eqref{eq:minimizationproblemdefinition} with gradient descent  with Gaussian initialization: let $ \Wb_{1}^{(0)},\ldots,\Wb_{L}^{(0)} $ be weight matrices generated from Gaussian initialization, we consider the following gradient descent update rule:
\begin{align*}
    \Wb_{l}^{(k)} = \Wb_{l}^{(k-1)} - \eta \nabla_{\Wb_l} L_{S}(\Wb_{l}^{(k-1)}),~l=1,\ldots,L, 
\end{align*}
where $\eta>0$ is the step size (a.k.a., learning rate).

\noindent\textbf{Stochastic gradient descent.} We also consider solving \eqref{eq:minimizationproblemdefinition} using stochastic gradient descent with Gaussian initialization. Again, let $\{ \Wb_{l}^{(0)} \}_{l=1}^L$ be generated from Gaussian initialization. At the $k$-th iteration, a minibatch $\cB^{(k)}$ of training examples with batch size $B$ is sampled from the training set, and the  stochastic gradient is calculated as follows:
\begin{align*}
    & \Gb_l^{(k)} = \frac{1}{B} \sum_{i \in \cB^{(k)}} \nabla_{\Wb_{l}} \ell \big[y_i f_{\Wb^{(k)}}(\xb_i)\big], ~l=1,\ldots,L.
\end{align*}
The update rule for stochastic gradient descent is then defined as follows:
\begin{align*}
    \Wb_{l}^{(k+1)} = \Wb_{l}^{(k)} - \eta \Gb_{l}^{(k)},~l=1,\ldots,L,
\end{align*}
where $\eta>0$ is the step size.

\subsection{Preliminaries}
Here we briefly introduce some useful notations and provide some basic calculations regarding the neural network under our setting.
\begin{itemize}
    \item \textbf{Output after the $l$-th layer:} Given an input $\xb_i$, the output of the neural network after the $l$-th layer is
    \begin{align*}
        \xb_{l,i} &= \sigma ( \Wb_{l}^{\top} \sigma ( \Wb_{l-1}^{\top} \cdots \sigma( \Wb_{1}^{\top} \xb_i )\cdots))=\bigg(\prod_{r=1}^l\bSigma_{r,i}\Wb_{r}^{\top}\bigg)\xb_i,
    \end{align*}
    where $\bSigma_{1,i} = \Diag\big( \ind\{ \Wb_1^{\top} \xb_i > 0 \} \big)$\footnote{Here we slightly abuse the notation and denote $\ind\{\ab >0\} = (\ind\{\ab_1 >0\},\ldots,\ind\{\ab_m >0\})^\top$ for a vector $\ab\in \RR^m$.}, and $\bSigma_{l,i} = \Diag [\ind \{ \Wb_{l}^{\top} (\prod_{r=1}^{l-1}\bSigma_{r,i}\Wb_{r}^{\top}) \xb_i > 0  \} ]$ for $l=2,\ldots,L$.
    \item \textbf{Output of the neural network: } The output of the neural network with input $\xb_i$ is as follows:
    \begin{align*}
    f_{\Wb}(\xb_i) = \vb^\top \sigma ( \Wb_{L}^{\top} \sigma ( \Wb_{L-1}^{\top} \cdots \sigma( \Wb_{1}^{\top} \xb_i )\cdots)) = \vb^\top\Bigg(\prod_{r=l}^L\bSigma_{r,i}\Wb_{r}^{\top}\Bigg)\xb_{l-1,i},
    \end{align*}
where we define $\xb_{0,i} = \xb_i$ and the last equality holds for any $l\ge 1$.
    \item \textbf{Gradient of the neural network:} The partial gradient of the training loss $L_S(\Wb)$ with respect to $\Wb_l$ is as follows:
    \begin{align*}
        \nabla_{\Wb_l}L_S(\Wb) = \frac{1}{n}\sum_{i=1}^n\ell'(y_i\hat y_i)\cdot y_i\cdot \nabla_{\Wb_l}[f_\Wb(\xb_i)],
    \end{align*}
    where
    \begin{align*}
    \nabla_{\Wb_l}[ f_{\Wb}(\xb_i)] =\xb_{l-1,i}\vb^\top \bigg(\prod_{r=l+1}^L\bSigma_{r,i}\Wb_{r}^\top\bigg)\bSigma_{l,i}.
\end{align*}

\end{itemize}

\input{convergence}

\input{proof_main}

\section{Conclusions and Future Work}\label{sec:conclusions}
In this paper, we studied training deep neural networks by gradient descent and stochastic gradient descent. We proved that both gradient descent and stochastic gradient descent can achieve global minima of over-parameterized deep ReLU networks with random initialization, for a general class of loss functions, with only mild assumption on training data. Our theory sheds light on understanding why stochastic gradient descent can train deep neural networks very well in practice, and paves the way to study the optimization dynamics of training more sophisticated deep neural networks.

In the future, %we will improve the dependence of our results on the number of layers $L$. 
we will sharpen the polynomial dependence of our results on those problem-specific parameters.

\section{Acknowledgment}
We would like to thank Spencer Frei for helpful comments on the first version of this paper.

\appendix

\input{appendix}

\input{appendix2}

\input{Appendix3}

\bibliography{relu}
\bibliographystyle{ims}

\end{document}

%% file: intro.tex
\section{Introduction}

Deep neural networks have achieved great success in many applications like image processing \citep{krizhevsky2012imagenet}, speech recognition \citep{hinton2012deep} and Go games \citep{silver2016mastering}. However, the reason why deep networks work well in these fields remains a mystery for long time. Different lines of research try to understand the mechanism of deep neural networks from different aspects. For example, a series of work tries to understand how the expressive power of deep neural networks are related to their architecture, including the width of each layer and depth of the network \citep{telgarsky2015representation, telgarsky2016benefits, lu2017expressive, liang2016deep, yarotsky2017error, yarotsky2018optimal, hanin2017universal, hanin2017approximating}. %\CC{need more references here, as suggested by the reviewers of Lingxiao Wang's AISTATS submission} 
These work shows that multi-layer networks with wide layers can approximate arbitrary continuous function. 

In this paper, we mainly focus on the optimization perspective of deep neural networks. It is well known that without any additional assumption, even training a shallow neural network is an NP-hard problem \citep{blum1989training}. Researchers have made various assumptions to get a better theoretical understanding of training neural networks, such as Gaussian input assumption \citep{brutzkus2017sgd, du2017convolutional, zhong2017learning} and independent activation assumption \citep{choromanska2015loss, kawaguchi2016deep}. A recent line of work tries to understand the optimization process of training deep neural networks from two aspects: over-parameterization and random weight initialization. It has been observed that over-parameterization and proper random initialization can help the optimization in training neural networks, and various theoretical results have been established \citep{safran2017spurious, du2018power, arora2018convergence, allen2018rnn, du2018gradient, li2018learning}. More specifically,  \citet{safran2017spurious} showed that over-parameterization can help reduce the spurious local minima in one-hidden-layer neural networks with Rectified Linear Unit (ReLU) activation function. \citet{du2018power} showed that with over-parameterization, all local minima in one-hidden-layer networks with quardratic activation function are global minima. \citet{arora2018optimization} showed that over-parameterization introduced by depth can accelerate the training process using gradient descent (GD). \citet{allen2018rnn} showed that with over-parameterization and random weight initialization, both gradient descent and stochastic gradient descent (SGD) can find the global minima of recurrent neural networks.

The most related work to ours are \citet{li2018learning} and \citet{du2018gradient}. \citet{li2018learning} showed that for a one-hidden-layer network with ReLU activation function using over-parameterization and random initialization, GD and SGD can find the near global-optimal solutions in polynomial time with respect to the accuracy parameter $\epsilon$, training sample size $n$ and the data separation parameter $\delta$\footnote{More precisely, \citet{li2018learning} assumed that each data point is generated from distributions $\{\cD_{l}\}$, and $\delta$ is defined as $\delta : = \min_{i,j \in [l]}\{\text{dist}(\text{supp}(\cD_i),\text{supp}(\cD_j) )\}$.}. \citet{du2018gradient} showed that under the assumption that the population Gram matrix is not degenerate\footnote{More precisely, \citet{du2018gradient} assumed that the minimal singular value of $\Hb^\infty$ is greater than a constant, where $\Hb^\infty$ is defined as $\Hb^\infty_{i,j}: = \EE_{\wb \sim N(0,\Ib)} [\xb_i^\top\xb_j \ind\{\wb^\top\xb_i \geq 0, \wb^\top \xb_j \geq 0\}]$ and $\{\xb_i\}$ are data points.}, randomly initialized GD converges to a globally optimal solution of a one-hidden-layer network with ReLU activation function %with a linear convergence rate 
and quadratic loss function. However, both \citet{li2018learning} and \citet{du2018gradient} only characterized the behavior of gradient-based method on one-hidden-layer shallow neural networks rather than on deep neural networks that are widely used in practice.

% A central question we would like to address in this work is: 
% \begin{center}
% \emph{For deep nonlinear neural networks, can gradient-based methods converge to the global minima under mild assumptions? }
% \end{center}

In this paper, we aim to advance this line of research by studying the optimization properties of gradient-based methods for deep ReLU neural networks. %We answer the above question in the affirmative way. 
In specific, we consider an $L$-hidden-layer fully-connected neural network with ReLU activation function. Similar to the one-hidden-layer case studied in \citet{li2018learning} and \citet{du2018gradient}, we study binary classification problem and show that both GD and SGD can achieve global minima of the training loss for any $L \geq 1$, with the aid of  over-parameterization and random initialization. At the core of our analysis is to show that Gaussian random initialization followed by (stochastic) gradient descent generates a sequence of iterates within a small perturbation region centering around the initial weights. In addition, we will show that the empirical loss function of deep ReLU networks has very good local curvature properties inside the perturbation region, which guarantees the global convergence of (stochastic) gradient descent.  More specifically, our main contributions are summarized as follows: 
\begin{itemize}
    \item We show that with Gaussian random initialization on each layer, when the number of hidden nodes per layer is at least $\tilde \Omega\big(\text{poly}(n,\phi^{-1},L)\big)$, GD can achieve  zero training error within $\tilde O\big(\text{poly}(n,\phi^{-1}, L)\big)$ iterations, where $\phi$ is the data separation distance, $n$ is the number of training examples, and $L$ is the number of hidden layers. Our result can be applied to a broad family of loss functions, as opposed to cross entropy loss studied in \citet{li2018learning} and quadratic loss considered in \citet{du2018gradient}.
    \item We also prove a similar convergence result for SGD. We show that with Gaussian random initialization on each layer, when the number of hidden nodes per layer is at least $\tilde \Omega\big(\text{poly}(n,\phi^{-1}, L)\big)$, SGD can also achieve zero training error within  $\tilde O\big(\text{poly}(n,\phi^{-1}, L)\big)$ iterations. 
    %We want to emphasize that this result is highly nontrivial, since in general case, linear convergence rate cannot be achieved by SGD  even in the convex optimization setting.
    \item 
    %Our proof only makes the so-called data separation assumption on input data, which is more general than the assumption made in \cite{du2018gradient} and 
%    \CC{XXX}.
%Our theory relies on the so-called data separation assumption on input data, which is more general than the assumption made in \cite{du2018gradient}, and is weaker than the distribution assumption made in \citet{li2018learning}.
    In terms of data distribution, we only make the so-called data separation assumption, which is more realistic than the assumption on the gram matrix made in \cite{du2018gradient}. The data separation assumption in this work is similar, but slightly milder, than that in \citet{li2017convergence} in the sense that %it only requires that 
    it holds as long as the data are sampled from a distribution with a constant margin separating different classes. %cite{allen2018convergence} requires that each data example is separated from others, which can potentially introduce additional dependency on the sample size $n$ when the data are independently sampled. 
    %In terms of assumptions on data distribution, we only make a data separation assumption, which is more realistic than the assumptions based on the smallest eigenvalue of the gram matrix made in \cite{du2018gradient}.
    %Our proof only makes data separation assumption on input data, 
    % The data separation assumption in this work is also much milder, and much more realistic than that in \citet{li2017convergence}. More precisely, our result only requires that the distance between data instances from different classes is lower bounded by a positive constant $\phi$ and has no requirement on the data instances from the sample class. In comparison, \citet{li2017convergence} assumes the data distribution consists of multiple disjoint 
    % data instances from each class consist of  not only requires that the data instances from different classes are separated, but also assume that the size of each cluster of data instances from the same class is small enough, which may not be realistic in real-world applications.
\end{itemize}

%When we were preparing this manuscript, we were informed that two concurrent work \citet{allen2018convergence} and \citet{du2018gradientdeep} have appeared on-line very recently.
When we were preparing this manuscript, we were informed that two concurrent work \citep{allen2018convergence,du2018gradientdeep} has appeared on-line very recently.
%Compared with \citet{allen2018convergence}, our work bears a similar proof idea at a high level, by extending the results for two-layer ReLU networks in \citet{li2018learning} to deep ReLU networks. 
Our work bears a similarity to \citet{allen2018convergence} in the high-level proof idea, which is to extend the results for two-layer ReLU networks in \citet{li2018learning} to deep ReLU networks. However, while \citet{allen2018convergence} mainly focuses on the regression problems with least square loss, we study the classification problems for a broad class of loss functions based on a milder data distribution assumption. 
\citet{du2018gradientdeep} also studies the regression problem. % with \citet{du2018gradientdeep},
Compared to their work, our work is based on a different assumption on the training data and is able to deal with the nonsmooth ReLU activation function. %Moreover, \citet{du2018gradientdeep} studies the regression problem while we study classification problem.

% We are aware of two concurrent work \citep{allen2018convergence, du2018gradientdeep}.
% %Compared with \citet{allen2018convergence}, 
% Our work bears a similarity to \citet{allen2018convergence} in the high-level proof idea, which is to extend the results for two-layer ReLU networks in \citet{li2018learning} to deep ReLU networks. However, while \citet{allen2018convergence} mainly focuses on regression problems with least square loss, we study the classification problems based on a broad class of loss functions. Compared with \citet{du2018gradientdeep}, our work is based on a different assumption on the training data and is able to deal with the nonsmooth ReLU activation function.

The remainder of this paper is organized as follows: In Section \ref{sec:related work}, we discuss the literature that are most related to our work. In Section \ref{sec:preliminaries}, we introduce the problem setup and preliminaries of our work. In Sections \ref{sec:main theory} and \ref{sec:proof of main theory}, we present our main theoretical results and their proofs respectively. We conclude our work and discuss some future work in Section \ref{sec:conclusions}.

\section{Related Work}\label{sec:related work}
Due to the huge amount of literature on deep learning theory, we are not able to include all papers in this big vein here. Instead, we review the following three major lines of research, which are most related to our work.

\noindent\textbf{One-hidden-layer neural networks with ground truth parameters} 
Recently a series of work \citep{tian2017analytical,brutzkus2017globally, li2017convergence, du2017convolutional, du2017gradient, zhang2018learning} study a specific class of shallow two-layer (one-hidden-layer) neural networks, whose training data are generated by a ground truth network called ``teacher network". This series of work aim to provide recovery guarantee for gradient-based methods to learn the teacher networks based on either the population or empirical loss functions. More specifically, \citet{tian2017analytical} proved that for two-layer ReLU networks with only one hidden neuron, GD with arbitrary initialization on the population loss is able to recover the hidden teacher network. \citet{brutzkus2017globally} proved that GD can learn the true parameters of a two-layer network with a convolution filter. \citet{li2017convergence} proved that SGD can recover the underlying parameters of a two-layer residual network in polynomial time. Moreover, \citet{du2017convolutional,du2017gradient} proved that both GD and SGD can recover the teacher network of a two-layer CNN with ReLU activation function. \citet{zhang2018learning} showed that GD on the empirical loss function can recover the ground truth parameters of one-hidden-layer ReLU networks at a linear rate. 
    
\noindent\textbf{Deep linear networks} 
Beyond shallow one-hidden-layer neural networks, a series of recent work \citep{hardt2016identity,kawaguchi2016deep,bartlett2018gradient, arora2018convergence, arora2018optimization} focus on the optimization landscape of deep linear networks. More specifically, \citet{hardt2016identity} showed that deep linear residual networks have no spurious local minima. \citet{kawaguchi2016deep} proved that all local minima are global minima in deep linear networks. \citet{arora2018optimization} showed that depth can accelerate the optimization of deep linear networks. \citet{bartlett2018gradient} proved that with identity initialization and proper regularizer, GD can converge to the least square solution on a residual linear network with quadratic loss function, while \citet{arora2018convergence} proved the same properties for general deep linear networks.

\noindent\textbf{Generalization bounds for deep neural networks}
The phenomenon that deep neural networks generalize better than shallow neural networks have been observed in practice for a long time \citep{langford2002not}. Besides classical VC-dimension based results \citep{vapnik2013nature, anthony2009neural}, a vast literature have recently studied the connection between the generalization performance of deep neural networks and their architectures \citep{neyshabur2015norm, neyshabur2017exploring, neyshabur2017pac, bartlett2017spectrally, golowich2017size, arora2018stronger, allen2018generalization}. More specifically, \citet{neyshabur2015norm} derived Rademacher complexity for a class of norm-constrained feed-forward neural networks with ReLU activation function. \citet{bartlett2017spectrally}  derived margin bounds for deep ReLU networks based on Rademacher complexity and covering number.  \citet{neyshabur2017exploring, neyshabur2017pac} also derived similar spectrally-normalized margin bounds for deep neural networks with ReLU activation function using PAC-Bayes approach.  \citet{golowich2017size} studied size-independent sample complexity of deep neural networks and showed that the sample complexity can be independent of both depth and width under additional assumptions. \citet{arora2018stronger} proved generalization bounds via compression-based framework. \citet{allen2018generalization} showed that an over-parameterized one-hidden-layer neural network can learn a one-hidden-layer neural network with fewer parameters using SGD up to a small generalization error, while similar results also hold for over-parameterized two-hidden-layer neural networks.

%% file: convergence.tex
\section{Main Theory}\label{sec:main theory}

In this section, we show that with random Gaussian initialization, over-parameterization helps gradient based algorithms, including gradient descent and stochastic gradient descent, converge to the global minimum, i.e., find some points with arbitrary small training loss.  

\subsection{Gradient Descent}
We provide the following theorem which characterizes the required numbers of hidden nodes and iterations such that the gradient descent can attain the global minimum of the empirical training loss function.
\begin{theorem}\label{thm:GD}
Suppose $\Wb_1^{(0)}, \dots, \Wb_L^{(0)}$ are generated by Gaussian initialization. Then under Assumptions \ref{assump:loss}-\ref{assump:m_scaling}, if set the step size $\eta = O(n^{-3}L^{-9}m^{-1})$, the number of hidden nodes per layer 
\begin{align*}
m=\left\{\begin{array}{ll}
\tilde \Omega\big(\text{poly}(n,\phi^{-1},L)\big)& 0\le p <\frac{1}{2}\\
\tilde \Omega\big(\text{poly}(n,\phi^{-1},L)\big)\cdot \Omega\big(\log(1/\epsilon)\big)& p=\frac{1}{2}\\
\tilde \Omega\big(\text{poly}(n,\phi^{-1},L)\big)\cdot \Omega(\epsilon^{1-2p})& \frac{1}{2}<p \le 1,
\end{array} \right.
\end{align*}
and the maximum number of iteration
\begin{align*}
K=\left\{\begin{array}{ll}
 \tilde O\big(\text{poly}(n,\phi^{-1},L)\big)& 0\le p <\frac{1}{2}\\
\tilde O\big(\text{poly}(n,\phi^{-1},L)\big)\cdot O\big(\log(1/\epsilon)\big)& p=\frac{1}{2}\\
\tilde O\big(\text{poly}(n,\phi^{-1},L)\big)\cdot O(\epsilon^{1-2p})& \frac{1}{2}<p \le 1,
\end{array} \right.
\end{align*}
then with high probability, gradient descent can find a point $\Wb^{(K)}$ such that  $L_S(\Wb^{(K)})\le \epsilon$.
\end{theorem}
\begin{remark}
Theorem \ref{thm:GD} suggests that the required number of hidden nodes and the number of iterations are both polynomial in the number of training examples $n$, and the separation parameter $\phi$. This is consistent with the recent work on the global convergence in training neural networks \citep{li2017convergence,du2018gradient,allen2018rnn,du2018gradientdeep,allen2018convergence}. Moreover, we prove that the dependence on the number of hidden layers $L$ is also polynomial, which is similar to \citet{allen2018convergence} and strictly better than \citet{du2018gradientdeep}, where the dependence on $L$ is proved to be $e^{\Omega(L)}$.
Regarding different loss functions (depending on $\alpha_0$ and $p$ according to Assumption \ref{assump:derivative_loss}), the dependence in $\epsilon$ ranges from $\tilde O(\epsilon^{-1})$ to $\tilde O(1)$. 
%It is worth noting that, both $m$ and $K$ have an exponential dependence on the number of hidden layers $L$. The same dependence on $L$ has also appeared in \citep{du2018gradientdeep} for training  deep neural networks with smooth activation functions. For deep ReLU networks, \citet{allen2018convergence} proved that the number of hidden nodes and the number of iterations that guarantee global convergence only have a polynomial dependence on $L$. We will improve the dependence of our results on $L$ in the future.
\end{remark}

Based on the results in Theorem \ref{thm:GD}, we are able to characterize the required number of hidden nodes per layer that gradient descent can find a point with zero training error in the following corollary.

\begin{corollary}\label{coro:gd}
Under the same assumptions as in Theorem \ref{thm:GD}, if $\ell(0) > 0$, then gradient descent can find a point with zero training error if the number of hidden nodes per layer is at least $m =\tilde \Omega\big(\text{poly}(n,\phi^{-1},L)\big) $.
\end{corollary}

% \begin{lemma}[Gradient upper bound]
% \begin{align*}
% \|\nabla_{\Wb_l}L_S(\Wb)\|_F\le \frac{-C_l}{n}\sum_{i=1}^{n}\ell'(y_i\hat y_i)
% \end{align*}
% \end{lemma}

\subsection{Stochastic Gradient Descent}

Regarding stochastic gradient descent, we make the following additional assumption on the derivative of the loss function $\ell(x)$, which is necessary to control the optimization trajectory of SGD.
\begin{assumption}\label{assump:derivative_loss_upper}
There exist positive constants $\alpha_0,\alpha_1,\rho_0,\rho_1$ with $\alpha_0 < \rho_0$, $\alpha_1 < \rho_1$ such that for any $x\in \RR$, we have
\begin{align*}
\min\{\alpha_0,\alpha_1\ell^p(x)\}\le|\ell'(x)|\le \min\{\rho_0,\rho_1\ell^p(x)\},
\end{align*}
where $p\le 1$ is a positive constant and $\rho_0/\alpha_0 = O(1)$.
\end{assumption}
Apparently, this assumption is stronger than Assumption \ref{assump:derivative_loss}, since in addition to the lower bound of $|\ell'(x)|$, we also require that $|\ell'(x)|$ is upper bounded by a function of the loss $\ell(x)$ with the same order $p$ as the lower bound. Moreover, if $\alpha_0 = \infty$, it follows that $\rho_0 = \infty$, and the assumption reduces to $\alpha_1\ell^p(x)\le |\ell'(x)|\le \rho_1\ell^p(x)$. If $\alpha_0<\infty$, we have $\rho_0<\infty$, which implies that $\ell(x)$ is $\rho_0$-Lipschitz.

\begin{theorem}\label{thm:SGD}
Suppose $\Wb_1^{(0)}, \dots, \Wb_L^{(0)}$ are generated by Gaussian random. Then under Assumptions \ref{assump:loss}-\ref{assump:m_scaling} and \ref{assump:derivative_loss_upper}, if the step size $\eta = O(n^{-3}L^{-9}m^{-1})$, the number of hidden nodes per layer satisfies
\begin{align*}
m=\left\{\begin{array}{ll}
\tilde \Omega\big(\text{poly}(n,\phi^{-1},L)\big)& 0\le p <\frac{1}{2}\\
\tilde \Omega\big(\text{poly}(n,\phi^{-1},L)\big)\cdot \Omega\big(\log^2(1/\epsilon)\big)& p=\frac{1}{2}\\
\tilde \Omega\big(\text{poly}(n,\phi^{-1},L)\big)\cdot \Omega(\epsilon^{2-4p})& \frac{1}{2}<p \le 1,
\end{array} \right.
\end{align*}
and the number of iteration
\begin{align*}
K=\left\{\begin{array}{ll}
 \tilde O\big(\text{poly}(n,\phi^{-1},L)\big)& 0\le p <\frac{1}{2}\\
\tilde O\big(\text{poly}(n,\phi^{-1},L)\big)\cdot O\big(\log(1/\epsilon)\big)& p=\frac{1}{2}\\
\tilde O\big(\text{poly}(n,\phi^{-1},L)\big)\cdot O(\epsilon^{1-2p})& \frac{1}{2}<p \le 1,
\end{array} \right.
\end{align*}
then with high probability, stochastic gradient descent can find a point $\Wb^{(K)}$ such that  $L_S(\Wb^{(K)})\le \epsilon$.
\end{theorem}

% \begin{remark}
% Compared with \citet{zou2018stochastic}, Theorem \ref{thm:sgd} improves the dependency of $L$ from $e^{\Omega(L)}$ to $\Omega \big(\text{poly}(L)\big)$, which is consistent with \citet{allen2018convergence}. In addition, similar to gradient descent, if $\ell(0)>0$, by setting $\epsilon = O(n^{-1})$, the stochastic gradient can find a point with zero training error within $\tilde O\big(\text{poly}(n,\phi^{-1},L)\big)$ iterations if the neural network has $m = \tilde\Omega\big(\text{poly}(n,\phi^{-1},L)\big)$ nodes per hidden layer. Moreover, as it can not be directly observed in Theorems \ref{thm:GD} and \ref{thm:sgd}, we remark here that compared with gradient descent, the required hidden nodes and iterations of stochastic gradient to achieve zero training error is worse by a factor ranging from $O(n^2)$ to $O(n^4)$.
% \end{remark}

Similar to gradient descent, the following corollary characterizes the required number of hidden nodes per layer that stochastic gradient descent can achieve zero training error.
\begin{corollary}\label{coro:sgd}
Under the same assumptions as in Theorem \ref{thm:SGD}, if $\ell(0) > 0$, then stochastic gradient descent can find a point with zero training error if the number of hidden nodes per layer is at least $m =\tilde \Omega\big(\text{poly}(n,\phi^{-1},L)\big) $.
\end{corollary}
\begin{remark}
Theorem \ref{thm:SGD} suggests that, to find the global minimum, both the required number of hidden nodes and the number of iterations for stochastic gradient descent are also polynomial in $n$, $\phi$ and $L$, which matches the result in \citet{allen2018convergence} for the regression problem. In addition, as it cannot be directly observed in Corollaries \ref{coro:gd} and \ref{coro:sgd}, we remark here that compared with gradient descent, the required numbers of hidden nodes and iterations of stochastic gradient to achieve zero training error is worse by a factor ranging from $O(n^2)$ to $O(n^4)$.
The detailed comparison can be found in the proofs of Theorems \ref{thm:GD} and \ref{thm:SGD}.
\end{remark}

% \begin{lemma}\label{lemma:lip_GD_sparse}
% With high probability
% \begin{align*}
% \bigg\|\vb^\top\bigg(\prod_{r=l+1}^L\bSigma_{r,i}^{(k)}\Wb_{r}^{(k)\top}\bigg)\tilde \bSigma_{l,i}^{(k+1)}\bigg\|_2\le \sqrt{s\log(M)}.
% \end{align*}
% \end{lemma}

%% file: proof_main.tex
\section{Proof of the Main Theory}\label{sec:proof of main theory}
In this section, we provide the proof of the main theory, including Theorems \ref{thm:GD} and \ref{thm:SGD}. Our proofs for these two theorems can be decomposed into the following five steps:
\begin{enumerate}
    \item We prove the basic properties for Gaussian random matrices $\{\Wb_l\}_{l,\dots,L}$ in Theorem~\ref{thm:randinit}, which constitutes a basic structure of the neural network after Gaussian random initialization.
    \item Based on Theorem \ref{thm:randinit}, we analyze the effect of $\|\cdot\|_2$-perturbations on Gaussian initialized weight matrices within a perturbation region with radius $\tau$, 
    and show that the neural network enjoys good local curvature properties in Theorem~\ref{thm:perturbation}.
    %with perturbation level parameter $\tau$. We perform $\|\cdot\|_2$-perturbation on each matrix $\Wb_l$, which formulates an $\ell_2$ perturbation region. Then based on Theorem \ref{thm:randinit}, we  within the perturbation region in Theorem \ref{thm:perturbation}. The goal of this step is to establish a perturbation region to cover the potential optimization trajectory in the training process.
     
    %We derive an upper bound, denoted by $T$, and show that for any iteration number $k$ and step size $\eta$ satisfying $k\eta\le T$, after $k$-step (stochastic) gradient descent the iterate $\Wb^{k}$ remains in the perturbation region centering at $\Wb^{(0)}$.
    \item Based on the assumption that all iterates are within the perturbation region centering at $\Wb^{(0)}$ with radius $\tau$, we establish the convergence results in Lemmas \ref{lemma:gd_converge} and \ref{lemma:sgd_converge}, and derive conditions on the product of iteration number $k$ and step size $\eta$ that guarantees convergence.  
    
    \item We show that as long as the product of iteration number $k$ and step size $\eta$ is smaller than some quantity $T$, (stochastic) gradient descent with $k$ iterations remains in the perturbation region centering around the Gaussian initialization $\Wb^{(0)}$, which justifies the application of Theorem~\ref{thm:perturbation} to the iterates of (stochastic) gradient descent.
    \item We finalize the proof by ensuring that (stochastic) gradient descent converges before $k\eta$ exceeds $T$ by setting on the  number of hidden nodes in each layer $m$ to be large enough.
\end{enumerate}

The following theorem summarizes some high probability results of neural networks with Gaussian random initialization, which is pivotal to establish the subsequent theoretical analyses.
\begin{theorem}\label{thm:randinit}  
%Suppose that for any $l=1,\ldots, L$, $\wb_{l,1},\ldots,\wb_{l,m_{l}}$ are generated independently from $N(\mathbf{0},2/m_l\Ib)$. 
Suppose that $\Wb_{1},\ldots,\Wb_{L}$ are generated by Gaussian initialization. Then under Assumptions~\ref{assump:normalizeddata}, \ref{assump:separateddata} and \ref{assump:m_scaling}, there exist absolute constants $\overline{C},\overline{C}',\underline{C},\underline{C}',\underline{C}''> 0 $ such that
for any $\delta >0$, $\beta>0$ and positive integer $s$, as long as 
\begin{align} \label{eq:randinit_m_conditions}
m \geq \overline{C} \max\{  L^4\phi^{-4}\log(n^2L/\delta) , [\beta^{-1}\log(nL/\delta)]^{2/3}, n^2 \phi^{-1} \log(n^2/\phi) \},~ s\geq \overline{C}\log(nL^2/\delta), 
\end{align}
and $\phi\leq \underline{C} \min\{L^{-1}, \mu \}$, with probability at least $1-\delta$, all the following results hold:
\begin{enumerate}[label=(\roman*)]
    \item $ \big| \| \xb_{l,i} \|_2 - 1\big| \leq \overline{C}'L\sqrt{\log(nL/\delta)/m}, \| \Wb_{l} \|_2\leq \overline{C}'$ for all $l=1\ldots,L$ and $i=1,\ldots,n$. \label{ThmResult:randinit_normbounds}
    \item $ \big\| \| \xb_{l,i} \|_2^{-1} \xb_{l,i}  - \| \xb_{l,i'} \|_2^{-1}\xb_{l,i'}   \big\|_2 \geq \phi/2 $ for all $l=1,\ldots,L$ and $i, i' \in \{1,\ldots,n\}$ such that $y_i \neq y_{i'}$.\label{ThmResult:randinit_dataseparationbound}
    \item $ |\hat{y}_i| \leq  \overline{C}'\sqrt{\log( n / \delta)}$ for all $i=1,\ldots,n$. \label{ThmResult:randinit_outputbound}
    \item $\big|\{j\in[m_l]: |\la\wb_{l,j},\xb_{l-1,i}\ra|\le \beta \} \big| \leq 2 m_l^{3/2}\beta $ for all $l=1,\ldots,L$ and $i=1,\ldots,n$.\label{ThmResult:randinit_activationthreshold}
    \item $\big\| \Wb_{l_2}^\top \big( \prod_{r = l_1}^{l_2 - 1} \bSigma_{r,i}\Wb_{r}^\top \big)\big\|_2 \leq \overline{C}' L$ for all $1\leq l_1 < l_2 \leq L$ and $i=1,\ldots,n$.\label{ThmResult:randinit_matproductnorm}
    \item $\vb^\top \big( \prod_{r = l}^{L } \bSigma_{r,i}\Wb_{r}^\top \big) \ab \leq \overline{C}' L\sqrt{s\log(M)}$ for all $l=1,\ldots,L$, $i=1,\ldots,n$ and all $\ab\in S^{m_{l-1} -1} $ with $ \| \ab \|_0 \leq s$ .\label{ThmResult:randinit_vecmatproductsparsenorm}
    \item $ \mathrm{\mathbf{b}}^\top \Wb_{l_2}^\top \big(\prod_{r= l_1 }^{l_2 - 1} \bSigma_{r,i}\Wb_{r}^\top\big) \ab \leq \overline{C}' L\sqrt{s\log(M)/m}$ for all $l=1,\ldots,L$, $i=1,\ldots,n$ and all $\ab\in S^{m_{l_1 - 1} - 1}, \mathrm{\mathbf{b}} \in S^{m_{l_2} - 1}$ with $ \| \ab \|_0,\| \mathrm{\mathbf{b}} \|_0\leq s$.\label{ThmResult:randinit_matproductdoublesparsenorm}
    \item For any $\ab = (a_1,\ldots,a_n)^\top \in \RR_+^n$, there exist at least $ \underline{C}' m_L \phi / n$ nodes satisfying \label{ThmResult:randinit_gradientuniformlowerbound}
    \begin{align*}
    \Bigg\|\frac{1}{n}\sum_{i=1}^n a_i\sigma'(\la\wb_{L,j},\xb_{L-1,i}\ra)\xb_{L-1,i}\Bigg\|_2\ge \underline{C}''\|\ab \|_\infty/n.
\end{align*}
\end{enumerate}
\end{theorem}
\begin{remark} Theorem~\ref{thm:randinit} summarizes all the properties we need for Gaussian initialization. In the sequel, we always assume that results \ref{ThmResult:randinit_normbounds}-\ref{ThmResult:randinit_gradientuniformlowerbound} hold for the Gaussian initialization. The parameters $\beta$ and $s$ in Theorem~\ref{thm:randinit} are introduced to characterize the activation pattern of the ReLU activation functions in each layer. Their values that directly help the final convergence proof is derived during the proof of Theorem~\ref{thm:perturbation} as $\beta = O(L^{4/3} \tau^{2/3} m^{-1/2})$ and $s = O(L^{4/3}\tau^{2/3} m)$, where $\tau = \tilde O(n^{-9}L^{-17}\phi^3)$ is the perturbation level. Therefore, the condition on the rate of $m_1,\ldots,m_L$ given by \eqref{eq:randinit_m_conditions} is satisfied under the final assumptions on $m$ given in Theorem~\ref{thm:GD} and Theorem~\ref{thm:SGD}.
%The parameter $\beta$ in Theorem~\ref{thm:randinit} is chosen in later proofs as $\beta = 4^L \tau^{2/3} m_l^{-1/2}$. With $\tau = \tilde O(n^{-15}\phi^9)\cdot e^{-O(L)}$, it is easy to check that the condition \eqref{eq:randinit_m_conditions} is satisfied under the conditions of Theorem~\ref{thm:GD} and Theorem~\ref{thm:SGD}.
\end{remark}

%%%%%%%%%%%%%%%%%%%%%%%%%%%%%%%%%%%%%%%%%%%%%%%%%%%%%%%%%%%%%%%%%%%%%%%%%%%%%%

% \begin{proof}[Proof of Lemma~\ref{lemma:lip_GD_sparse0}]
% The proof of Lemma~\ref{lemma:lip_GD_sparse0} is almost the same as Lemma~\ref{lemma:lipschitzcontinuity_sparsevector}: instead of applying an $\epsilon$-net argument for the vector $\mathrm{\mathbf{b}}$, we directly prove the result for the fixed vector $\vb$. The result follows since $\| \vb \|_2 = \sqrt{m}$. 
% \end{proof}

% For the rest of the proof, we always assume that the results of Lemmas~ \ref{lemma:lipschitzcontinuity_vector}, \ref{lemma:lipschitzcontinuity_sparsevector}, \ref{lemma:lip_GD_sparse0} and Corollaries~\ref{cor:randinit_normfinalbound}, \ref{cor:randinit_distancelowerbound} hold. 

We perform $\|\cdot\|_2$-perturbation on the collection of random matrices $\{\Wb_l\}_{l=1,\dots,L}$ with perturbation level $\tau$, which formulates a perturbation region centering at $\{\Wb_l\}_{l=1,\dots,L}$ with radius $\tau$.
Let $\tilde{\Wb} = \{ \tilde{\Wb}_1,\ldots,\tilde{\Wb}_L \}$ and $\hat{\Wb} = \{ \hat{\Wb}_1,\ldots,\hat{\Wb}_L \}$ be two collections of weight matrices. For $l=1,\ldots,L$, denote $\tilde{\xb}_{l,i}$, $\hat{\xb}_{l,i}$ be the output of the $l$-th hidden layer of the ReLU network with input $\xb_i$ and weight matrices $\tilde{\Wb}$ and $\hat{\Wb}$ respectively. Define $\tilde{\xb}_{0,i}=\tilde{\xb}_{0,i} = \xb_i$, and
\begin{align*}
\tilde{\bSigma}_{l,i} = \Diag\big(\ind\{ \tilde{\Wb}_{l,1}^\top \tilde{\xb}_{l-1,i} \},\ldots,\ind\{ \tilde{\Wb}_{l,L}^\top \tilde{\xb}_{l-1,i} \}\big),
\hat{\bSigma}_{l,i} = \Diag\big(\ind\{ \hat{\Wb}_{l,1}^\top \hat{\xb}_{l-1,i} \},\ldots,\ind\{ \hat{\Wb}_{l,L}^\top \hat{\xb}_{l-1,i} \}\big)
\end{align*}
for all $l=1,\ldots,L$.
% For the perturbed matrices $\{\tilde\Wb_l\}_{l=1,\dots,L}$ within such perturbation region, let $\tilde \bSigma_{1,i} = \text{Diag}\big(\ind\big\{\tilde\Wb_1\xb_i>0\big\}\big)$, $\tilde \bSigma_{l,i} = \text{Diag}\big[\ind\big\{\tilde \Wb_l^\top(\prod_{r=1}^{l-1}\tilde \bSigma_{r,i}\tilde \Wb_r^\top)\big\}\xb_i>0\big]$, and $\tilde \xb_{l,i} = \big(\prod_{r=1}^{l}\tilde \bSigma_{r,i}\tilde \Wb_r^\top\big)\xb_i$.
We summarize their properties in the following theorem. 
\begin{theorem}\label{thm:perturbation} Suppose that $\Wb_1,\ldots,\Wb_L$ are generated via Gaussian initialization, and all results \ref{ThmResult:randinit_normbounds}-\ref{ThmResult:randinit_gradientuniformlowerbound} in Theorem \ref{thm:randinit} hold. Let $\{\tilde \Wb_l\}_{l=1,\dots,L}$, $\{\hat \Wb_l\}_{l=1,\dots,L}$ be perturbed weight matrices satisfying $\|\tilde \Wb_l-\Wb_l\|_2, \|\hat \Wb_l-\Wb_l\|_2\le \tau$, $l=1,\ldots,L$. Then under Assumptions \ref{assump:separateddata} and \ref{assump:m_scaling}, there exist absolute constants $\overline{C},\overline{C}',\underline{C},\underline{C}'$, such that as long as $\tau \leq \underline{C} \{ [ L^{-11} (\log(M))^{-3/2} ] \land (\phi^{3/2}n^{-3}L^{-2}) \}$, the following results hold:
\begin{enumerate}[label=(\roman*)]
    \item $\|\tilde \Wb_l\|_2\le \overline{C}$ for all $l\in[L]$.\label{item:bound_tilde_W}
    \item $\| \hat \xb_{l,i} - \tilde \xb_{l,i} \|_2\le  \overline{C} L \cdot \sum_{r=1}^l \| \hat{\Wb}_{r} - \tilde{\Wb}_{r} \|_2$ for all $l\in[L]$ and $i\in[n]$.\label{item:difference_xli}
    \item $ \| \hat{\bSigma}_{l,i} - \tilde{\bSigma}_{l,i} \|_0 \leq \overline{C} L^{4/3}\tau^{2/3} m_l$ for all $l\in[L]$ and $i\in[n]$.\label{item:difference_sigmali}
    \item $\big| \{j\in[m_L]: \text{there exists } i\in[n] \mbox{ such that } (\tilde\bSigma_{L,i} - \bSigma_{L,i})_{jj}\neq 0 \}\big| \leq \overline{C}n L^{4/3}\tau^{2/3} m_L$.\label{item:difference_sigmali_forall}
    % \item $\big|\{j\in[m_l]: |\la\wb_{l,j},\xb_{l-1,i}\ra|\le 4^L\tau^{2/3}/m_l^{1/2} \} \big| \leq 2\cdot 4^L m_l \tau^{2/3}$ for all $i=1,\ldots,L$ and $i=1,\ldots,n$.
    % \item $\big\| \Wb_{l_2}^\top \big( \prod_{r = l_1}^{l_2 - 1} \bSigma_{r,i}\Wb_{r}^\top \big)\big\|_2 \leq CL$ for all $1\leq l_1 < l_2 \leq L$ and $i=1,\ldots,n$.
    \item $ \big\| \prod_{r=l_1}^{l_2}\tilde \bSigma_{r,i}\tilde \Wb_r^\top \big\|_2 \leq \overline{C} L $ for all $1 \leq l_1 < l_2 \leq L $.\label{item:perturb_matrixnormbound}
    \item $ \vb^\top \big(\prod_{r=l}^{L}\tilde \bSigma_{r,i}\tilde \Wb_r^\top \big) \ab \leq \overline{C}'  L^{5/3} \tau^{1/3}\sqrt{M\log(M)} $ for all $\ab\in \RR^{m_{l - 1}}$ satisfying $\| \ab \|_2 = 1$, $\| \ab \|_0 \leq \overline{C} L^{4/3}\tau^{2/3} m_l$ and any $1 \leq  l \leq L $.
    \label{item:perturbe_lip_sparse}
    \item The squared Frobenius norm of the partial gradient with respect to the weight matrix in the last hidden layer has the following lower bound:
    \begin{align*}
\|\nabla_{\Wb_L}[ L_S(\tilde \Wb)]\|_F^2\ge \underline{C}' \frac{m_L\phi}{n^5}\bigg(\sum_{i=1}^n\ell'(y_i\tilde y_i)\bigg)^2.
\end{align*}
where $\tilde y_i = f_{\tilde \Wb}(\xb_i)$.\label{item:grad_lowerbound}
    \item The spectral norms of gradients and stochastic gradients at each layer have the following upper bounds:
    \begin{align*}
    \big\|\nabla_{\Wb_l}[L_S(\tilde \Wb)]\big\|_2 \le -\frac{\overline{C}L^2M^{1/2}}{n}\sum_{i=1}^n\ell'(y_i\tilde y_i) \mbox{ and } \big\|\tilde \Gb_l\big\|_2 \le -\frac{\overline{C}L^2M^{1/2}}{B}\sum_{i\in\cB}\ell'(y_i\tilde y_i),
    \end{align*}
where $\tilde y_i = f_{\tilde \Wb}(\xb_i)$, $B = |\cB|$ denotes the minibatch size in SGD.
%     \item
%     For all $l\in[L]$, the Frobenius norm of the partial gradient of $L_S(\tilde \Wb)$ with respect to the weight matrix in the $l$-th hidden layer is upper bounded by
%         \begin{align*}
%     \big\|\nabla_{\Wb_l}[L_S(\tilde \Wb)]\big\|_F \le -\frac{4^LM^{1/2}}{n}\sum_{i=1}^n\ell'(y_i\tilde y_i).
% \end{align*}
% Moreover, the Frobenius norm of the stochastic partial gradient $\tilde \Gb_l$ is upper bounded by
% \begin{align*}
%     \big\|\tilde \Gb_l\big\|_F \le -\frac{4^LM^{1/2}}{B}\sum_{i\in\cB}\ell'(y_i\tilde y_i),
% \end{align*}
% where $\cB$ denotes the mini-batch of indices queried for computing stochastic gradient, and $B=|\cB|$ denotes the minibatch size.
\label{item:grad_upperbound}
\end{enumerate}
\end{theorem}
The gradient lower bound provided in \ref{item:grad_lowerbound} implies that within the perturbation region, the empirical loss function of deep neural network enjoys good local curvature properties, which play an essential role in the convergence proof of (stochastic) gradient descent. The gradient upper bound in \ref{item:grad_upperbound} %will be utilized to 
quantifies how much the weight matrices of the neural network would change during (stochastic) gradient descent, which is utilized to guarantee that the weight matrices won't escape from the perturbation region during the training process. 

\subsection{Proof of Theorem \ref{thm:GD}}\label{subsec:convergence_gd}
We organize our proof as the following three steps: (1) we first assume that during gradient descent, each iterate is in the preset perturbation region centering at $\Wb^{(0)}$ with radius $\tau$, and use the results in Theorem \ref{thm:perturbation} to establish the convergence guarantee; (2) we prove the upper bound of the number of iteration such that the distance between the iterate and the initial point does not exceed $\tau$; (3) we compute the minimum number of hidden nodes such that gradient descent achieves the target accuracy before exceeding the upper bound derived in step (2). 

For step (1), the following lemma provides the convergence guarantee of gradient descent while assuming all iterates are in the preset perturbation region, i.e., $\|\Wb_l^{(k)}-\Wb_l^{(0)}\|_2\le \tau$ for all $l=1,\ldots,L$ and $k=1,\ldots,K$.

\begin{lemma}\label{lemma:gd_converge}
Suppose that $\Wb_1^{(0)},\dots,\Wb_L^{(0)}$ are generated via Gaussian initialization, and all results \ref{ThmResult:randinit_normbounds}-\ref{ThmResult:randinit_gradientuniformlowerbound} in Theorem~\ref{thm:randinit} hold. Under Assumptions \ref{assump:loss}-\ref{assump:m_scaling}, if $\|\Wb_l^{(k)}-\Wb_l^{(0)}\|_2\le \tau$ for all $l=1,\ldots,L$ and $k=1,\ldots,K$ with perturbation level $\tau = O(n^{-9}L^{-17}\phi^3)$, the step size $\eta = O( n^{-3}L^{-9}m^{-1}\phi)$ and
\begin{align*}
K\eta=\left\{\begin{array}{ll}
\tilde \Omega\big(n^{5-2p}/(m\phi)\big)& 0\le p <\frac{1}{2}\\
\tilde \Omega(n^4/(m\phi)\big)\cdot \Omega\big(\log(1/\epsilon))& p=\frac{1}{2}\\
\tilde \Omega\big(n^{5-2p}\epsilon^{1-2p}/(m\phi)\big)& \frac{1}{2}<p \le 1
\end{array} \right.  
\end{align*}
when $\alpha_0 = \infty$,
\begin{align*}
K\eta=\left\{\begin{array}{ll}
\tilde \Omega\big(n^5/(m\phi)\big)& 0\le p <\frac{1}{2}\\
\tilde \Omega\big(n^5/(m\phi)\big)+\tilde \Omega(n^4/(m\phi)\big)\cdot \Omega\big(\log(1/\epsilon))& p=\frac{1}{2}\\
\tilde \Omega\big(n^5/(m\phi)\big)+\tilde \Omega\big(n^{5-2p}/(m\phi)\big)\cdot\Omega\big(\epsilon^{1-2p}\big)& \frac{1}{2}<p \le 1
\end{array} \right.  
\end{align*}
when $\alpha_0 < \infty$, then gradient descent is able to find a point $\Wb^{(K)}$ such that $L_S(\Wb^{(K)})\le \epsilon$.
\end{lemma}

The following lemma provides the upper bound of the iteration number such that the distance between the iterate and the initial point does not exceed the perturbation radius $\tau$.

\begin{lemma}\label{lemma:upperbound_tau}
Suppose that $\Wb_1^{(0)},\dots,\Wb_L^{(0)}$ are generated by Gaussian initialization, and all results \ref{ThmResult:randinit_normbounds}-\ref{ThmResult:randinit_gradientuniformlowerbound} in Theorem~\ref{thm:randinit} hold. Then there exist a constant $T = O(L^{-4}n^{-3}\tau^2\phi)$ such for all iteration number $k$ and step size $\eta$ satisfying $k\eta \le T$, it holds that $\|\Wb_l^{(k)}-\Wb_l^{(0)}\|_2\le \tau$ for all $l\in[L]$.
\end{lemma}

Now we are ready to prove Theorem~\ref{thm:GD} based on Lemmas \ref{lemma:gd_converge} and \ref{lemma:upperbound_tau}.
\begin{proof}[Proof of Theorem \ref{thm:GD}]
The proof is straightforward. By Lemma \ref{lemma:upperbound_tau}, it suffices to show that the lower bound of $K\eta$ derived in Lemma \ref{lemma:gd_converge} is smaller than $T = O(L^{-4}n^{-3}\tau^2\phi) = O(L^{-38}n^{-21}\phi^7)$, where we plug in the assumption that $\tau = O(n^{-9}L^{-17}\phi^3)$. Therefore, we can derive the following lower bound on the number of hidden nodes per layer, i.e., $m$: 
\begin{itemize}
    \item $\alpha_0=\infty$: \begin{align*}
m=\left\{\begin{array}{ll}
\tilde \Omega\big(n^{26-2p}L^{38}/\phi^8\big)& 0\le p <\frac{1}{2}\\
\tilde \Omega(n^{25}L^{38}/\phi^8\big)\cdot \Omega\big(\log(1/\epsilon)\big)& p=\frac{1}{2}\\
\tilde \Omega\big(n^{26-2p}L^{38}/\phi^8\big)\cdot\Omega\big(\epsilon^{1-2p}\big)& \frac{1}{2}<p \le 1.
\end{array} \right.  
\end{align*}
\item $\alpha_0 < \infty$:
\begin{align*}
m=\left\{\begin{array}{ll}
\tilde \Omega\big(n^{26}L^{38}/\phi^8\big)& 0\le p <\frac{1}{2}\\
\tilde \Omega\big(n^{26}L^{38}/\phi^8\big)+\tilde \Omega\big(n^{25}L^{38}/\phi^8\big)\cdot \Omega\big(\log(1/\epsilon)\big)& p=\frac{1}{2}\\
\tilde \Omega\big(n^{26-2p}L^{38}/\phi^8\big)+\tilde \Omega \big(n^{26}L^{38}/\phi^8\big)\cdot\Omega(\epsilon^{1-2p})& \frac{1}{2}<p \le 1.
\end{array} \right.  
\end{align*}

\end{itemize}
Moreover, the required number of iterations, i.e., $K$ can be directly derived by combining the results of $K\eta$ in Lemma \ref{lemma:gd_converge} and the choice of the step size $\eta = O(n^{-3}L^{-9}m^{-1}\phi)$, thus we omit the detail here.
\end{proof}

\subsection{Proof of Theorem \ref{thm:SGD}}\label{subsec:convergence_sgd}
Similar to the proof for gradient descent, we first deliver the following lemma which characterizes the convergence of stochastic gradient descent for the training of ReLU network under the assumption that all iterates are in the preset perturbation region.

\begin{lemma}\label{lemma:sgd_converge}
Suppose that $\Wb_1^{(0)},\dots,\Wb_L^{(0)}$ are generated by Gaussian initialization, and all results \ref{ThmResult:randinit_normbounds}-\ref{ThmResult:randinit_gradientuniformlowerbound} in Theorem~\ref{thm:randinit} hold. Under Assumptions \ref{assump:loss}-\ref{assump:m_scaling} and \ref{assump:derivative_loss_upper}, if $\|\Wb_l^{(k)}-\Wb_l^{(0)}\|_2\le \tau$ for all $l=1,\ldots,L$ and $k=1,\ldots,K$ with perturbation level $\tau = O(n^{-9}L^{-17}\phi^3)$, the step size $\eta = O( Bn^{-4}L^{-9}m^{-1}\phi)$, and
\begin{align*}
K\eta=\left\{\begin{array}{ll}
\tilde \Omega\big(n^{5-2p}/(m\phi)\big)& 0\le p <\frac{1}{2}\\
\tilde \Omega(n^4/(m\phi)\big)\cdot \Omega\big(\log(1/\epsilon))& p=\frac{1}{2}\\
\tilde \Omega\big(n^{5-2p}\epsilon^{1-2p}/(m\phi)\big)& \frac{1}{2}<p \le 1
\end{array} \right.  
\end{align*}
when $\alpha_0 = \infty$, 
\begin{align*}
K\eta=\left\{\begin{array}{ll}
\tilde \Omega\big(n^5/(m\phi)+n^{5-2p}/(m\phi)\big)& 0\le p <\frac{1}{2}\\
\tilde \Omega\big(n^5/(m\phi)\big)+\tilde \Omega(n^4/(m\phi)\big)\cdot \Omega\big(\log(1/\epsilon))& p=\frac{1}{2}\\
\tilde \Omega\big(n^5/(m\phi)+n^{5-2p}\epsilon^{1-2p}/(m\phi)\big)& \frac{1}{2}<p \le 1
\end{array} \right.  
\end{align*}
when $\alpha_0 < \infty$, 
then stochastic gradient descent is able to find a point $\Wb^{(K)}$ such that $L_S(\Wb^{(K)})\le \epsilon$.
\end{lemma}
The following lemma provides the upper bound of the iteration number such that the distance between the iterate and the initial point does not exceed the perturbation radius $\tau$.

\begin{lemma}\label{lemma:upperbound_tau_sgd}
Suppose that $\Wb_1^{(0)},\dots,\Wb_L^{(0)}$ are generated by Gaussian initialization, and all results \ref{ThmResult:randinit_normbounds}-\ref{ThmResult:randinit_gradientuniformlowerbound} in Theorem~\ref{thm:randinit} hold. If the step size $\eta = O(\phi m^{-1}n^{2p-7}L^{-8}B^{2})$, then there exists $T$ with rate $T = \tilde O(L^{-2}n^{-1}Bm^{-1/2}\tau)$ when $\alpha_0 = \infty$ and $T =  O(L^{-2}m^{-1/2}\tau)$ when $\alpha<\infty$ such that for all iteration number $k$ satisfying $k\eta \le T$, with high probability it holds that $\|\Wb_l^{(k)}-\Wb_l^{(0)}\|_2\le \tau$ for all $l\in[L]$.
\end{lemma}

\begin{proof}[Proof of Theorem \ref{thm:SGD}]
Similar to the proof of Theorem \ref{thm:GD}, the minimum required number of hidden nodes per layer can be derived by setting the lower bound of $K\eta$ in Lemma \ref{lemma:sgd_converge} to be smaller than $T = \tilde O(L^{-2}n^{-1}m^{-1/2}\tau) = \tilde O(n^{-10}L^{-19}m^{-1/2}\phi^{3})$ when $\alpha_0=\infty$ or $T = O(L^{-2}m^{-1/2}\tau) = O(n^{-9}L^{-19}m^{-1/2}\phi^3)$ when $\alpha_0 < \infty$. Therefore the minimum required number of hidden nodes per layer satisfies
\begin{itemize}
    \item $\alpha_0=\infty$: \begin{align*}
m=\left\{\begin{array}{ll}
\tilde \Omega\big(n^{30-4p}L^{38}/\phi^8\big)& 0\le p <\frac{1}{2}\\
\tilde \Omega(n^{28}L^{38}/\phi^8\big)\cdot \Omega\big(\log(1/\epsilon)\big)& p=\frac{1}{2}\\
\tilde \Omega\big(n^{30-4p}L^{38}/\phi^8\big)\cdot\Omega\big(\epsilon^{2-4p}\big)& \frac{1}{2}<p \le 1.
\end{array} \right.  
\end{align*}
\item $\alpha_0 < \infty$:
\begin{align*}
m=\left\{\begin{array}{ll}
\tilde \Omega\big(n^{28}L^{38}/\phi^8\big)& 0\le p <\frac{1}{2}\\
\tilde \Omega\big(n^{28}L^{38}/\phi^8\big)+\tilde \Omega\big(n^{26}L^{38}/\phi^8\big)\cdot \Omega\big(\log^2(1/\epsilon)\big)& p=\frac{1}{2}\\
\tilde \Omega\big(n^{28}L^{38}/\phi^8\big)+\tilde \Omega \big(n^{28-4p}L^{38}/\phi^8\big)\cdot\Omega(\epsilon^{2-4p})& \frac{1}{2}<p \le 1.
\end{array} \right.  
\end{align*}
\end{itemize}
We now proceed to derive the required iteration numbers. By Lemmas \ref{lemma:sgd_converge} and \ref{lemma:upperbound_tau_sgd}, we can set the step size to be $\eta = O(\phi m^{-1}n^{2p-7}L^{-9}B^2)$. Then using the bound of $K\eta$ derived in Lemma \ref{lemma:sgd_converge}, we have
% The proof in terms of the required iteration number is slightly different from that for gradient descent, since in addition to the lower bound of $K\eta$ we derived in Lemma \ref{lemma:sgd_converge}, the iteration number $K$ should also satisfies $K = \tilde \Omega(n^{12-4p} B^{-2}\phi^{-2})$. Therefore, we have
\begin{itemize}
    \item $\alpha_0 = \infty$:
 \begin{align*}
K=\left\{\begin{array}{ll}
\tilde O\big(n^{12-4p}B^{-2}L^9\phi^{-2}\big)& 0\le p <\frac{1}{2}\\
\tilde O\big(n^{10}B^{-2}L^9\phi^{-2}\big)\cdot O\big(\log(1/\epsilon))& p=\frac{1}{2}\\
\tilde O\big(n^{12-4p}B^{-2}L^9\phi^{-2}\big)\cdot O\big(\epsilon^{1-2p}\big)& \frac{1}{2}<p \le 1.
\end{array} \right.  
\end{align*}
\item $\alpha_0 < \infty$:
 \begin{align*}
K=\left\{\begin{array}{ll}
\tilde O\big(n^{12-2p}B^{-2}L^9\phi^{-2}\big)& 0\le p <\frac{1}{2}\\
\tilde O\big(n^{11}B^{-2}L^9\phi^{-2}\big)+\tilde O\big(n^{10}B^{-2}L^9\phi^{-2}\big)\cdot O\big(\log(1/\epsilon))& p=\frac{1}{2}\\
\tilde O\big(n^{12-2p}B^{-2}L^9\phi^{-2}\big)+\tilde O\big(n^{12-4p}B^{-2}L^9\phi^{-2}\big)\cdot O\big(\epsilon^{1-2p}\big)& \frac{1}{2}<p \le 1.
\end{array} \right.  
\end{align*}
\end{itemize}
This completes the proof.
\end{proof}

%% file: appendix.tex
\section{Proof of Theorem~\ref{thm:randinit}}
In this section we provide the proof of Theorem~\ref{thm:randinit}. 
The bound for $\| \Wb_l \|_{2}$ given by result \ref{ThmResult:randinit_normbounds} in Theorem~\ref{thm:randinit} follows from standard results for Gaussian random matrices with independent entries (See Corollary~5.35 in \cite{vershynin2010introduction}). We split the rest results into several lemmas and prove them separately. 

We first give the bound for the norms of the outputs of each layer. Intuitively, since the columns of $\Wb_{l}$ are sampled independently from $N(\mathbf{0}, 2/m_l\Ib)$, given the output of the previous layer $\xb_{l-1,i}$, the expectation of $\|\Wb_l^{\top} \xb_{l-1,i}\|_2^2$ is $2\| \xb_{l-1,i} \|_2^2$. Moreover, the ReLU activation function truncates roughly half of the entries of $\Wb_{l}^\top \xb_{l-1,i}$ to zero, and therefore $\| \xb_{l,i} \|_2^2$ should be approximately equal to $\| \xb_{l-1,i} \|_2^2$. This leads to Lemma~\ref{lemma:randinit_normconcentration} and Corollary~\ref{cor:randinit_normfinalbound}.

\begin{lemma}\label{lemma:randinit_normconcentration}
Denote by $\wb_{l,j}$ the $j$-th column of $\Wb_{l}$. Suppose that for any $l=1,\ldots, L$, $\wb_{l,1},\ldots,\wb_{l,m_{l}}$ are generated independently from $N(\mathbf{0},2/m_l\Ib)$. Then there exists an absolute constant $C$ such that
for any $\delta >0$, as long as $m_l \geq C^2\log(nL/\delta)$, with probability at least $1-\delta$, 
\begin{align*}
    \big|\| \xb_{l,i} \|_2^2 - \| \xb_{l-1,i} \|_2^2 \big| \leq C  \| \xb_{l-1,i} \|_2^2 \cdot \sqrt{\frac{\log(nL/\delta)}{m_l}} 
\end{align*}
for all $i=1,\ldots,n$ and $l=1,\ldots,L$.
\end{lemma}
\begin{proof}[Proof of Lemma~\ref{lemma:randinit_normconcentration}]
For any fixed $i\in \{1,\ldots,n\}$, $l\in \{1,\ldots,L\}$ and $j\in \{1,\ldots, m_{l}\}$, condition on $\xb_{l-1,i}$ we have $\wb_{l,j}^\top \xb_{l-1,i}\sim N(\mathbf{0},2\| \xb_{l-1,i} \|_2^2/m_l)$. Therefore,
\begin{align*}
    \EE [\sigma^2(\wb_{l,j}^\top \xb_{l-1,i}) | \xb_{l-1,i}] = \frac{1}{2}\EE [(\wb_{l,j}^\top \xb_{l-1,i})^2 | \xb_{l-1,i}] = \frac{1}{m_l}\| \xb_{l-1,i} \|_2^2.
\end{align*}
Since $\| \xb_{l,i} \|_2^2 = \sum_{j=1}^{m_l} \sigma^2(\wb_{l,j}^\top \xb_{l-1,i}) $ and condition on $\xb_{l-1}$, $\|\sigma^2(\wb_{l,j}^\top \xb_{l-1,i})\|_{\psi_1} \leq C_1 \| \xb_{l-1,i} \|_2^2/m_l$ for some absolute constant $C_1$, by Bernstein inequality (See Proposition~5.16 in \cite{vershynin2010introduction}), for any $\xi\geq 0$ we have
\begin{align*}
    \PP\Big(\big| \| \bSigma_{l,i} \Wb_{l}^\top \xb_{l-1,i} \|_2^2 - \| \xb_{l-1,i} \|_2^2 \big| \geq \| \xb_{l-1,i} \|_2^2 \xi \Big|\xb_{l-1,i} \Big) \leq 2\exp ( -C_2 m_l \min\{ \xi^2, \xi \} ).
\end{align*}
Taking union bound over $l$ and $i$ gives
\begin{align*}
    \PP\Big(\big| \| \xb_{l,i} \|_2^2 - \| \xb_{l-1,i} \|_2^2 \big| \leq \| \xb_{l-1,i} \|_2^2\xi,i=1,\ldots,n,l=1,\ldots,L \Big) \geq 1 - 2nL\exp ( -C_2 m_l \min\{ \xi^2, \xi \} ).
\end{align*}
The inequality above further implies that if $m_l \geq C_3^2\log(nL/\delta)$, then with probability at least $1-\delta$, we have
\begin{align*}
    \big|\| \xb_{l,i} \|_2^2 - \| \xb_{l-1,i} \|_2^2 \big| \leq C_3  \| \xb_{l-1,i} \|_2^2 \cdot \sqrt{\frac{\log(nL/\delta)}{m_l}} 
\end{align*}
for any $i=1,\ldots,n$ and $l=1,\ldots,L$, where $C_3$ is an absolute constant. This completes the proof.
\end{proof}

\begin{corollary}\label{cor:randinit_normfinalbound}
Under the same conditions as Lemma~\ref{lemma:randinit_normconcentration}, with probability at least $1-\delta$, 
\begin{align*}
    \big| \| \xb_{l,i} \|_2  - 1 \big| \leq C l \sqrt{\frac{\log(nL/\delta)}{m}},
\end{align*}
where $m = \min\{ m_1,\ldots, m_L \}$, and $C$ is an absolute constant.
\end{corollary}
\begin{proof}[Proof of Corollary~\ref{cor:randinit_normfinalbound}]
    The result directly follows by Lemma~\ref{lemma:randinit_normconcentration} and induction.
\end{proof}

By Corollary~\ref{cor:randinit_normfinalbound}, we can see that the norms of the inputs are roughly preserved after passing through layers in the ReLU neural network with properly scaled Gaussian random weights. We now proceed to show \ref{ThmResult:randinit_dataseparationbound} in Theorem~\ref{thm:randinit}, which states that the inner product of any two samples, although may not be preserved throughout layers, also share a common upper bound based on the Assumption \ref{assump:separateddata}. The detailed results are given in
Lemmas~\ref{lemma:randinit_positivegaussianproduct}, \ref{lemma:randinit_distancelowerboundinduction} and Corollary~\ref{cor:randinit_distancelowerbound} below. %together shows that the data separation property is preserved

\begin{lemma}\label{lemma:randinit_positivegaussianproduct}
For $\theta > 0$, let $Z_1$, $Z_2$ be two jointly Gaussian random variables with $\EE(Z_1) = \EE(Z_2) = 0$, $\EE(Z_1^2) = \EE(Z_2^2) = 1$ and $\EE(Z_1Z_2) \leq 1- \theta^2/2 $. If $\theta\leq \kappa$ for some small enough absolute constant $\kappa$, then
\begin{align*}
    %\PP( z_1 z_2 \geq 0 ) \geq
    \EE[\sigma(Z_1) \sigma(Z_2)] \leq \frac{1}{2} - \frac{1}{4} \theta^2 + C\theta^3,
\end{align*}
where $C$ is an absolute constant.
\end{lemma}

\begin{proof}[Proof of Lemma~\ref{lemma:randinit_positivegaussianproduct}]
Denote $\alpha = \EE(Z_1 Z_2)$. 
It is clear that the maximum of $\EE[\sigma(Z_1) \sigma(Z_2)]$ is achieved when $\alpha =1- \theta^2/2$, which is the case that gives the highest correlation between $Z_1$ and $Z_2$. Hence it suffices to consider this special case. 
By symmetry we have
\begin{align*}
    \EE[\sigma(Z_1)\sigma(Z_2)] &= \EE(Z_1 Z_2 | Z_1 \geq 0, Z_2 \geq 0) \PP(Z_1 \geq 0, Z_2 \geq 0)\\
    &= \EE(Z_1 Z_2 | Z_1 \leq 0, Z_2 \leq 0) \PP(Z_1 \leq 0, Z_2 \leq 0).
\end{align*}
Therefore,
\begin{align}
    \alpha &= \EE(Z_1 Z_2) \nonumber \\
    &= \EE(Z_1 Z_2  \ind\{ Z_1 \geq 0, Z_2 \geq 0 \} ) + \EE(Z_1 Z_2 \ind\{ Z_1 \leq 0, Z_2 \leq 0 \}) + \EE(Z_1 Z_2 \ind\{ Z_1 \leq 0, Z_2 \geq 0 \}) \nonumber \\
    &\quad  + \EE(Z_1 Z_2 \ind\{ Z_1 \geq 0, Z_2 \leq 0 \})\nonumber \\
    & = 2\cdot \EE[\sigma(Z_1)\sigma(Z_2)] + 2\cdot \EE(Z_1 Z_2 \ind \{ Z_1 \leq 0, Z_2 \geq 0 \}).\nonumber 
\end{align}
Rearranging terms gives
\begin{align}
    \EE[\sigma(Z_1)\sigma(Z_2)] = \frac{1}{2} \alpha - \EE(Z_1 Z_2 \ind \{ Z_1 \leq 0, Z_2 \geq 0 \}) = \frac{1}{2} \alpha + \EE( -Z_1 Z_2 \ind \{ -Z_1 \geq 0, Z_2 \geq 0 \}).\label{eq:randinit_positivegaussianproduct_eq1}
\end{align}
We now show that for small enough $\theta$ we have $\EE( -Z_1 Z_2 \ind \{ -Z_1 \geq 0, Z_2 \geq 0 \}) = O(\theta^3)$. 
Let $(U_1,U_2)$ be a standard Gaussian random vector. Then $(-Z_1,Z_2) \stackrel{d}{=} (U_1, -\alpha U_1 + \sqrt{1-\alpha^2} U_2)$, and
\begin{align*}
    \EE( -Z_1 Z_2 \ind \{ -Z_1 \geq 0, Z_2 \geq 0 \}) = I_1 + I_2,
\end{align*}
where
\begin{align*}
    & I_1 = \frac{\sqrt{1 - \alpha^2}}{2\pi} \int_{0}^{+\infty} \mathrm{d} u_1 \int_{\frac{\alpha u_1}{\sqrt{1 - \alpha^2}}}^{+\infty} \mathrm{d} u_2 \bigg\{ u_1 u_2  \exp\bigg[-\frac{1}{2}(u_1^2 + u_2^2) \bigg] \bigg\},\\
    & I_2 = -\frac{\alpha}{2\pi} \int_{0}^{+\infty} \mathrm{d} u_1 \int_{\frac{\alpha u_1}{\sqrt{1 - \alpha^2}}}^{+\infty} \mathrm{d} u_2 \bigg\{ u_1^2 \exp\bigg[-\frac{1}{2}(u_1^2 + u_2^2) \bigg] \bigg\}.
\end{align*}
For $I_1$, it follows by direct calculation that $I_1 = (1 - \alpha^2)^{3/2} /(2 \pi) = O(\theta^3)$. 
For $I_2$, integration by parts gives
\begin{align*}
    I_2 &= \frac{-\alpha}{2 \pi}\int_{0}^{+\infty} \exp\bigg(-\frac{u_1^2}{2}\bigg) \cdot \frac{\mathrm{d}}{\mathrm{d} u_1}\Bigg[  u_1 \int_{\frac{\alpha u_1}{\sqrt{1 - \alpha^2}}}^{+\infty}  \exp\bigg(-\frac{u_2^2}{2} \bigg) \mathrm{d} u_2   \Bigg] \mathrm{d} u_1 \\
    & = -\alpha \cdot \PP(U_1 \geq 0, -\alpha U_1 +\sqrt{1 - \alpha^2 } U_2 \geq 0) + \frac{\alpha}{2 \pi}\cdot \frac{\alpha}{\sqrt{1 - \alpha^2}} \cdot\int_{0}^{+\infty} u_1\exp\bigg(-\frac{u_1^2}{2(1 - \alpha^2)}\bigg) \mathrm{d} u_1\\
    & = - \alpha \cdot \PP(U_1 \geq 0, -\alpha U_1 +\sqrt{1 - \alpha^2 } U_2 \geq 0) + \frac{1}{2\pi} \alpha^2 \sqrt{1 - \alpha^2}.
\end{align*}
Note that $\{ (u_1,u_2): u_1\geq 0, -\alpha u_1 + \sqrt{1-\alpha^2} u_2 \}$ is a cone in $\RR^2$ with angle $\arccos(\alpha)$. Therefore, by the rotation invariant property of spherical Gaussian density function, we have
\begin{align*}
    \PP(U_1 \geq 0, -\alpha U_1 +\sqrt{1 - \alpha^2 } U_2 \geq 0) = \frac{1}{2\pi}\arccos(\alpha).
\end{align*}
For $\arccos$ function we have the following expansion:
\begin{align*}
    \arccos(1 - x) =   \sqrt{2x} +  \frac{\sqrt{2}}{12}x^{\frac{3}{2}} + O(x^{\frac{5}{2}}).
\end{align*}
Therefore when $\alpha = 1 - \theta^2/2$ for small enough constant $\theta$, we have
\begin{align*}
    \arccos(\alpha) = \theta + \frac{1}{24} \theta^3 + O(\theta^5),
\end{align*}
Therefore we have
\begin{align*}
    I_2 = \frac{-\alpha}{2\pi} \bigg[ \theta + \frac{1}{24} \theta^3 + O(\theta^5) - \alpha \sqrt{1 - \alpha^2}  \bigg] = O(\theta^3).
\end{align*}
Combining the obtained rates with \eqref{eq:randinit_positivegaussianproduct_eq1} completes the proof.
\end{proof}

\begin{lemma}\label{lemma:randinit_distancelowerboundinduction}
Suppose that $\wb_{l,1},\ldots,\wb_{l,m_{l}}$ are generated independently from $N(\mathbf{0},2/m_l\Ib)$. Let $\overline{\xb}_{l-1,i} = \xb_{l-1,i} / \| \xb_{l-1,i} \|_2$, $i= 1,\ldots,n$. For any fixed $i,i'\in \{1,\ldots, n \}$, if $\phi\leq \kappa L^{-1}$ for some small enough absolute constant $\kappa$, then for any $\delta > 0$, if $m = \min\{m_1,\ldots,m_L\} \geq C L^4\phi^{-4} \log(4n^2L/\delta)$ for some large enough absolute constant $C$, with probability at least $1-\delta$,
\begin{align*}
    \| \overline{\xb}_{l,i} - \overline{\xb}_{l,i'} \|_2 \geq [1 - (2L)^{-1} \log(2)]^{l}\phi
\end{align*}
for all $i,j=1,\ldots,n$, $l=1,\ldots,L$.
\end{lemma}
\begin{proof}[Proof of Lemma~\ref{lemma:randinit_distancelowerboundinduction}] 
We first consider any fixed $l\geq 1$. Suppose that $ \| \overline{\xb}_{l-1,i} -  \overline{\xb}_{l-1,i'} \|_2 \geq [1 - (2L)^{-1} \log(2)]^{l-1}\phi $. % where $C$ is the absolute constant in Lemma~\ref{lemma:randinit_positivegaussianproduct}. 
If we can show that under this condition, with high probability
\begin{align*}
    \| \overline{\xb}_{l,i} -  \overline{\xb}_{l,i'} \|_2 \geq [1 - (2L)^{-1} \log(2)]^{l}\phi,
\end{align*}
then the result of the lemma follows by union bound and induction. 
Denote 
$$\phi_{l - 1} =  [1 - (2L)^{-1} \log(2)]^{l-1}\phi.$$ 
Then by assumption we have $\| \overline{\xb}_{l-1,i} - \overline{\xb}_{l-1,i'} \|_2^2 \geq \phi_{l-1}^2$. Therefore $\overline{\xb}_{l-1,i}^\top \overline{\xb}_{l-1,i'} \leq 1 - \phi_{l-1}^2/2$.
It follows by direct calculation that 
\begin{align*}
    \EE \big(\| \xb_{l,i} - \xb_{l,i'} \|_2^2 \big| \xb_{l-1,i},\xb_{l-1,i'} \big) &= \EE \big(\| \xb_{l,i} \|_2^2 + \| \xb_{l,i'} \|_2^2  \big| \xb_{l-1,i},\xb_{l-1,i'} \big) - 2\EE \big( \xb_{l,i}^\top \xb_{l,i'}  \big| \xb_{l-1,i},\xb_{l-1,i'} \big)\\
    &= (\| \xb_{l-1,i} \|_2^2 + \| \xb_{l-1,i'} \|_2^2) - 2\EE \big( \xb_{l,i}^\top \xb_{l,i'}  \big| \xb_{l-1,i},\xb_{l-1,i'} \big).
\end{align*}
By Lemma~\ref{lemma:randinit_positivegaussianproduct} and the assumption that $\phi_{l-1}\leq \phi\leq \kappa$, we have
\begin{align*}
    \EE \big( \xb_{l,i}^\top \xb_{l,i'}  \big| \xb_{l-1,i},\xb_{l-1,i'} \big) &= \EE \Bigg[\sum_{j=1}^{m_l} \sigma(\wb_{l,j}^\top \xb_{l-1,i}) \sigma(\wb_{l,j}^\top \xb_{l-1,i'})   \Bigg| \xb_{l-1,i},\xb_{l-1,i'} \Bigg]\\
    & = \|\xb_{l-1,i} \|_2 \|\xb_{l-1,i'} \|_2 \cdot\EE \Bigg[\sum_{j=1}^{m_l} \sigma(\wb_{l,j}^\top \overline{\xb}_{l-1,i}) \sigma(\wb_{l,j}^\top \overline{\xb}_{l-1,i'}) \Bigg| \xb_{l-1,i},\xb_{l-1,i'} \Bigg]\\
    & \leq \frac{2}{m} \|\xb_{l-1,i} \|_2 \|\xb_{l-1,i'} \|_2 \cdot m\cdot  \bigg( \frac{1}{2} - \frac{1}{4} \phi_{l-1}^2 + C\phi_{l-1}^3 \bigg)\\
    &=  \|\xb_{l-1,i} \|_2 \|\xb_{l-1,i'} \|_2 \cdot  \bigg( 1 - \frac{1}{2} \phi_{l-1}^2 + 2C\phi_{l-1}^3 \bigg).
\end{align*}
Therefore, 
\begin{align}\label{eq:randinit_concentrationsquaredistanceeq1}
    \EE \big(\| \xb_{l,i} - \xb_{l,i'} \|_2^2 \big| \xb_{l-1,i},\xb_{l-1,i'} \big) \geq ( \|\xb_{l-1,i} \|_2 - \|\xb_{l-1,i'} \|_2)^2 +  \|\xb_{l-1,i} \|_2 \|\xb_{l-1,i'} \|_2( \phi_{l-1}^2 - 4C\phi_{l-1}^3).
\end{align}
Condition on $\xb_{l-1,i}$ and $\xb_{l-1,i'}$, by Lemma~5.14 in \cite{vershynin2010introduction} we have
\begin{align*}
    \big\| [\sigma(\wb_{l,j}^\top \xb_{l-1,i}) - \sigma(\wb_{l,j}^\top \xb_{l-1,i'})]^2 \big\|_{\psi_1} &\leq 2 \big[ \big\| [\sigma(\wb_{l,j}^\top \xb_{l-1,i})  \big\|_{\psi_2} +  \big\|  \sigma(\wb_{l,j}^\top \xb_{l-1,i'}) \big\|_{\psi_2} \big]^2\\
    &\leq C_1( \| \xb_{l-1,i} \|_2 + \| \xb_{l-1,i'} \|_2 )^2 /m_l,
\end{align*}
where $C_1$ is an absolute constant. Therefore if $m_l \geq C_2^2\log(4n^2L/\delta)$, similar to the proof of Lemma~\ref{lemma:randinit_normconcentration}, by Bernstein inequality and union bound, with probability at least $1-\delta/(4n^2L)$ we have
\begin{align*}
    \big|\| \xb_{l,i} - \xb_{l,i'} \|_2^2 - \EE \big(\| \xb_{l,i} - \xb_{l,i'} \|_2^2 \big| \xb_{l-1,i},\xb_{l-1,i'} \big)\big| \leq C_2( \| \xb_{l-1,i} \|_2 + \| \xb_{l-1,i'} \|_2 )^2 \cdot \sqrt{\frac{\log(2n^2L/\delta)}{m_l}},
\end{align*}
where $C_2$ is an absolute constant. 
Therefore with probability at least $1-\delta/(4n^2L)$ we have
\begin{align*}
    \| \xb_{l,i} - \xb_{l,i'} \|_2^2 &\geq  ( \|\xb_{l-1,i} \|_2 - \|\xb_{l-1,i'} \|_2)^2 +  \|\xb_{l-1,i} \|_2 \|\xb_{l-1,i'} \|_2( \phi_{l-1}^2 - 4C\phi_{l-1}^3)\\
    &\quad - C_2( \| \xb_{l-1,i} \|_2 + \| \xb_{l-1,i'} \|_2 )^2 \cdot \sqrt{\frac{\log(1/\delta)}{m_l}}.
\end{align*}
By union bound and Lemma~\ref{lemma:randinit_normconcentration}, if $m_r \geq C_3 L^4\phi_l^{-4} \log(4n^2L/\delta)$, $r=1,\ldots,l$ for some large enough absolute constant $C_3$ and $\phi \leq \kappa L^{-1}$ for some small enough absolute constant $\kappa$, then with probability at least $1 - \delta/(2n^2L)$ we have
\begin{align*}
    \| \xb_{l,i} - \xb_{l,i'} \|_2^2 \geq  [1 - (4L)^{-1} \log(2)] \phi_{l-1}^2 \geq [1 - (4L)^{-1} \log(2)]^2 \phi_{l-1}^2.
\end{align*}
Moreover, by Lemma~\ref{lemma:randinit_normconcentration}, with probability at least $1 - \delta/(2n^2L)$ we have
\begin{align*}
    \big| \| \overline{\xb}_{l,i} - \overline{\xb}_{l,i'} \|_2 - \| \xb_{l,i} - \xb_{l,i'} \|_2 \big| &\leq \| \overline{\xb}_{l,i} - \xb_{l,i} \|_2 + \| \overline{\xb}_{l,i'} - \xb_{l,i'} \|_2 \\
    &= \big| 1 - \|\xb_{l,i} \|_2 \big| +  \big| 1 - \|\xb_{l,i'} \|_2 \big| \\
    &\leq (4L)^{-1}\log(2)\cdot \phi_{l-1}^2,
\end{align*}
and therefore with probability at least $1 - \delta/(n^2L)$, we have
\begin{align*}
    \| \overline{\xb}_{l,i} - \overline{\xb}_{l,i'} \|_2 \geq [1 - (2L)^{-1} \log(2)] \phi_{l-1} = [1 - (2L)^{-1} \log(2)]^l \phi.
\end{align*}
Applying union bound and induction over $l=1,\ldots,L$ completes the proof. 
\end{proof}

\begin{corollary}\label{cor:randinit_distancelowerbound}
Under the same conditions in Lemma~\ref{lemma:randinit_distancelowerboundinduction}, with probability at least $1 - \delta$, 
\begin{align*}
    \| \overline{\xb}_{l,i} - \overline{\xb}_{l,i'} \|_2 \geq \phi  / 2
\end{align*}
for all $i,i'=1,\ldots,n$, $l=1,\ldots,L$.
\end{corollary}
\begin{proof}[Proof of Corollary~\ref{cor:randinit_distancelowerbound}] It directly follows by Lemma~\ref{lemma:randinit_distancelowerboundinduction} and the fact that $1 - x/2 \geq e^{-x}$ for $x\in [0,\log(2)]$.
\end{proof}

The following lemma gives the bound of $\hat{y}_i$. It relies on the fact that half of the entries of $\vb$ are $1$'s and the other half are $-1$'s. 

\begin{lemma}\label{lemma:randinit_outputbounded} Suppose that $\{ \Wb_l \}_{l=1}^L$ are generated via Gaussian initialization. Then for any $\delta > 0$, with probability at least $1 - \delta$, it holds that
\begin{align*}
    |\hat{y}_i| = |f_{\Wb} (\xb_i)| \leq C\sqrt{\log( n / \delta)}
\end{align*}
for all $i=1,\ldots,n$, where $C$ is an absolute constant.
\end{lemma}
\begin{proof}[Proof of Lemma~\ref{lemma:randinit_outputbounded}]
Note that half of the entries of $\vb$ are $1$'s and the other half of the entries are $-1$'s. Therefore, without loss of generality, here we assume that $v_1=\cdots=v_{m_L/2} = 1$ and $v_{m_L/2 + 1}=\cdots=v_{m_L} = -1$. Clearly, we have $\EE(\hat{y}_i) = 0$. Moreover, plugging in the value of $\vb$ gives
\begin{align*}
    \hat{y_i} = \sum_{j = 1}^{m_L / 2} [\sigma(\wb_{L,j}^\top \xb_{L-1,i}) - \sigma(\wb_{L,j + m_L/2}^\top \xb_{L-1,i})].
\end{align*}
Apparently, we have $ \| \sigma(\wb_{L,j}^\top \xb_{L-1,i}) - \sigma(\wb_{L,j + m_L/2}^\top \xb_{L-1,i}) \|_{\psi_2} \leq C_1 m_L^{-1/2}$ for some absolute constant $C_1$. Therefore by Hoeffding's inequality and Corollary~\ref{cor:randinit_normfinalbound}, with probability at least $1 - \delta$, it holds that
\begin{align*}
    |\hat{y}_i| \leq C_2 \sqrt{\log( n / \delta)}
\end{align*}
for all $i=1,\ldots,n$.
\end{proof}

We now prove \ref{ThmResult:randinit_activationthreshold} in Theorem~\ref{thm:randinit}, which characterizes the activation pattern of ReLU networks with a parameter $\beta$. We summarize the result in Lemma~\ref{lemma:randinit_activationthreshold}.

\begin{lemma}\label{lemma:randinit_activationthreshold}
Define
\begin{align*}
\cS_{l,i}(\beta) = \{j\in[m_l]: |\la\wb_{l,j},\xb_{l-1,i}\ra|\le \beta \}. 
\end{align*}
For any $\delta>0$, if $m_l \geq C \max \{ [\beta^{-1}\log(nL/\delta)]^{2/3},  L^2 \log(nL/\delta) \}$ for some large enough constant $C>0$, then with probability at least $1 - \delta$, $|\cS_{l,i}(\beta)| \leq 2 m_l^{3/2}\beta$ for all $l=1,\ldots,L$ and $i=1,\ldots,n$. 
\end{lemma}
\begin{proof}[Proof of Lemma~\ref{lemma:randinit_activationthreshold}]
For fixed $l\in \{1,\ldots, L\}$ and $i\in \{1,\ldots,n\}$, define $Z_{j} = \ind\{ |\la\wb_{l,j},\xb_{l-1,i}\ra|\le \beta \}$. Note that by Corollary~\ref{cor:randinit_normfinalbound}, with probability at least $1 - \exp[-\Omega(m/L^2)]$ we have $\| \xb_{l-1,i} \|_2 \geq 1/2 $. Condition on $\xb_{l-1,i}$, we have
\begin{align*}
    \EE(Z_j) = \PP(j\in\cS_{l,i})\le \frac{(2m_l)^{1/2}}{\sqrt{2\pi}}\int_{-\beta}^{\beta}e^{-x^2m_l/2}\mathrm{d}x\le \pi^{-1/2} \beta m_l^{1/2}.
\end{align*}
Then by Bernstein inequality, with probability at least $1 - \exp(-\Omega( m_l^{3/2} \beta))$, 
\begin{align*}
    \sum_{j=1}^{m_l} Z_j \leq 2\cdot  m_l^{3/2} \beta.
\end{align*}
Applying union bound over $l=1,\ldots,L$, $i=1,\ldots,n$ and using the assumption that $m_l^{3/2} \beta \geq C \log(nL/\delta) $ completes the proof.
\end{proof}

We now prove \ref{ThmResult:randinit_matproductnorm}-\ref{ThmResult:randinit_matproductdoublesparsenorm} in Theorem~\ref{thm:randinit}. These results together characterize the Lipschitz continuity of the gradients. To show these results, we utilize similar proof technique we used in Lemma~\ref{lemma:randinit_normconcentration} and combine the resulting bound with covering number arguments. The details are given in Lemmas~\ref{lemma:lipschitzcontinuity_vector}, \ref{lemma:lip_GD_sparse0} and \ref{lemma:lipschitzcontinuity_sparsevector}.

\begin{lemma}\label{lemma:lipschitzcontinuity_vector}
For any $i=1,\ldots,n$ and $1\leq l_1 < l_2\leq L$, 
define
\begin{align*}
    \Lambda_{i,l_1,l_2} = \sup_{\ab\in S^{m_{l_1 - 1} - 1}, \mathrm{\mathbf{b}} \in S^{m_{l_2} - 1}} \mathrm{\mathbf{b}}^\top \Wb_{l_2}^\top \Bigg(\prod_{r= l_1 }^{l_2 - 1} \bSigma_{r,i}\Wb_{r}^\top\Bigg) \ab.
\end{align*}
% \begin{align*}
%     \Lambda_{i,l_1,l_2} = \sup_{\ab\in S^{m_{l_1} - 1}, \mathrm{\mathbf{b}} \in S^{m_{l_2} - 1}} \mathrm{\mathbf{b}}^\top \Wb_{l_2}^\top \bSigma_{l_2-1,i}\Wb_{l_2-1}^\top\cdot \cdots \cdot \bSigma_{l_1+1,i}\Wb_{l_1+1}^\top \ab.
% \end{align*}
If $m \geq C\log(nL^2/\delta)$ for some absolute constant $C$, then with probability at least $1 - \delta$ we have
\begin{align*}
    \Lambda_{i,l_1,l_2} \leq C' L
\end{align*}
for all $i= 1,\ldots,n$ and $1\leq l_1 < l_2\leq L$, where $C'$ is an absolute constant. 
\end{lemma}
\begin{proof}[Proof of Lemma~\ref{lemma:lipschitzcontinuity_vector}]
Denote $l_0 = l_1 -1$, $\delta' = \delta/(nL^2)$ and
\begin{align*}
    g(\ab, \mathrm{\mathbf{b}}) = \mathrm{\mathbf{b}}^\top \Wb_{l_2}^\top \Bigg(\prod_{r= l_1 }^{l_2 - 1} \bSigma_{r,i}\Wb_{r}^\top\Bigg) \ab.
\end{align*}
Let $\cN_1 = \cN[S^{m_{l_0} - 1}, 1/4]$, $\cN_2 = \cN[S^{m_{l_2} - 1}, 1/4]$ be $1/4$-nets covering $S^{m_{l_0} - 1}$ and $S^{m_{l_2} - 1}$ respectively. Then by Lemma~5.2 in \cite{vershynin2010introduction}, we have
\begin{align*}
    |\cN_1|\leq 9^{m_{l_0}},~|\cN_2|\leq 9^{m_{l_2}}.
\end{align*}
A key observation is that, for any $i$, $l$ and any fixed vector $\balpha\in \RR^{m_{l-1}}$, although $\bSigma_{l,i}$ is defined by $\xb_{l-1, i}$ instead of $\balpha$, %condition on $\xb_{l-1, i}$ 
it still holds that
\begin{align*}
    \EE \big( \| \bSigma_{l,i} \Wb_{l}^\top \balpha \|_2^2 \big| \xb_{l-1,i} \big) &= \EE \Bigg[ \sum_{j=1}^{m_l} \ind\big\{ \wb_{l,j}^\top \xb_{l-1,i} > 0 \big\} \big( \wb_{l,j}^\top \balpha \big)^2 \Bigg| \xb_{l-1,i} \Bigg]\\
    &=\EE \Bigg\{ \sum_{j=1}^{m_l} \ind\big\{ \wb_{l,j}^\top \xb_{l-1,i} > 0 \big\} \big[ \wb_{l,j}^\top (\balpha_{\!/\mkern-5mu/\!} + \balpha_{\perp})\big]^2 \Bigg| \xb_{l-1,i} \Bigg\}\\
    & = \sum_{j=1}^{m_l} \EE \big[ \ind\big\{ \wb_{l,j}^\top \xb_{l-1,i} > 0 \big\} \big( \wb_{l,j}^\top \balpha_{\!/\mkern-5mu/\!} \big)^2 \big| \xb_{l-1,i} \big]\\
    & \qquad + \sum_{j=1}^{m_l} \EE \big[ \ind\big\{ \wb_{l,j}^\top \xb_{l-1,i} > 0 \big\} \big( \wb_{l,j}^\top \balpha_{\perp} \big)^2 \big| \xb_{l-1,i} \big]\\
    & =  \| \balpha_{\!/\mkern-5mu/\!} \|_2^2 + \| \balpha_{\perp} \|_2^2\\
    &= \| \balpha \|_2^2,
\end{align*}
where $\balpha_{\mathbin{\!/\mkern-5mu/\!}} = \balpha^\top \xb_{l-1,i}/\| \xb_{l-1,i} \|_2^2\cdot\xb_{l-1,i}$ and $\balpha_{\perp} = \balpha - \balpha_{\!/\mkern-5mu/\!}$. 
Therefore, similar to the proof of Lemma~\ref{lemma:randinit_normconcentration} and Corollary~\ref{cor:randinit_normfinalbound}, with probability at least $1 - \delta'/2$ we have
\begin{align*}
    \Bigg\| \Bigg(\prod_{r= l_1 }^{l_2 - 1} \bSigma_{r,i}\Wb_{r}^\top\Bigg) \hat{\ab} \Bigg\|_2 \leq C_0 L \sqrt{\frac{\log(2\cdot9^{m_{l_2}}/\delta')}{m}}
\end{align*}
for all $\hat{\ab} \in \cN_1$, where $C_0$ is an absolute constant. 
Therefore by Assumption~\ref{assump:m_scaling} and the assumption that $m \geq C\log(1/\delta')$ for some absolute constant $C$, with probability at least $1 - \delta'$, we have
\begin{align*}
    |g(\hat{\ab},\hat{\mathrm{\mathbf{b}}})|  \leq C_1 L\sqrt{ \frac{ \log(2\cdot 9^{m_{l_0} + m_{l_2}} /\delta')}{m}} \leq C_2 L\sqrt{ \frac{ m_{l_0} + m_{l_2} + m}{m}} \leq C_3 L
\end{align*}
for all $\hat{\ab} \in \cN_1$ and $\hat{\mathrm{\mathbf{b}}}\in \cN_2$, where $C_1,C_2,C_3$ are absolute constants. For any $\ab\in S^{m_{l_0} - 1}$ and $\mathrm{\mathbf{b}} \in S^{m_{l_2} - 1}$, there exists $\hat{\ab} \in \cN_1$ and $\hat{\mathrm{\mathbf{b}}}\in \cN_2$ such that $\| \ab - \hat{\ab}\|_2 \leq 1/4$, $\| \mathrm{\mathbf{b}} - \hat{\mathrm{\mathbf{b}}} \|_2\leq 1/4$. Therefore, we have
\begin{align*}
    g(\ab, \mathrm{\mathbf{b}}) &\leq |g(\hat{\ab}, \hat{\mathrm{\mathbf{b}}})| + |g(\ab, \mathrm{\mathbf{b}}) - g(\hat{\ab}, \hat{\mathrm{\mathbf{b}}})|\\
    & \leq C_3 L  + |g(\ab, \mathrm{\mathbf{b}}) - g(\hat{\ab}, \mathrm{\mathbf{b}})| + |g(\hat{\ab}, \mathrm{\mathbf{b}}) - g(\hat{\ab}, \hat{\mathrm{\mathbf{b}}})|.
\end{align*}
Since $g(\ab,\mathrm{\mathbf{b}})$ is a bilinear function, we have
\begin{align*}
    |g(\ab, \mathrm{\mathbf{b}}) - g(\hat{\ab}, \mathrm{\mathbf{b}})| = |g(\ab - \hat{\ab}, \mathrm{\mathbf{b}})| \leq \| \ab - \hat{\ab} \|_2 \bigg|g\bigg(\frac{\ab - \hat{\ab}}{\| \ab - \hat{\ab} \|_2}, \mathrm{\mathbf{b}}\bigg)\bigg| \leq \| \ab - \hat{\ab} \|_2 \Lambda_{i,l_1,l_2} \leq \Lambda_{i,l_1,l_2}/4.
\end{align*}
Similarly, we also have
\begin{align*}
    |g(\hat\ab, \mathrm{\mathbf{b}}) - g(\hat{\ab}, \hat{\mathrm{ \mathbf{b}}})| \leq \Lambda_{i,l_1,l_2}/4.
\end{align*}
Therefore, we have
\begin{align*}
    \Lambda_{i,l_1,l_2} \leq C_4L,
\end{align*}
where $C_4$ is an absolute constant. Applying union bound over all $i=1,\ldots, n$ and all $1\leq l_1 < l_2 \leq L$ completes the proof.
\end{proof}

\begin{lemma}\label{lemma:lip_GD_sparse0}
For any $i=1,\ldots,n$ and $1\leq l \leq L$, If $s \geq C\log(nL/\delta)$ for some absolute constant $C$, then with probability at least $1 - \delta$, we have 
\begin{align*}
\Gamma_{i,l,s} = \sup_{ \substack{\ab \in S^{m_{l - 1} - 1}\\ \| \ab \|_0\leq s}} \vb^\top\Bigg(\prod_{r=l}^{L}\bSigma_{r,i}\Wb_{r}^{\top}\Bigg) \ab \leq C' L\sqrt{s\log(M)}
\end{align*}
for all $i=1,\ldots,n$ and $l=1,\ldots,L$, where $C'$ is an absolute constant.
\end{lemma}

\begin{proof}[Proof of Lemma~\ref{lemma:lip_GD_sparse0}]
%The proof is similar to the proof of Lemma~\ref{lemma:lipschitzcontinuity_vector}. 
Denote $l_0 = l_1 - 1$ and $\delta' = \delta/(nL)$.
Let $\cM$ be a fixed subspace of $\RR^{m_{l_0}}$ with sparsity $s$. Then there are $\binom{m_{l_0}}{s}$ choices of $\cM$. 
let $\cN = \cup_{\cM}\cN(\cM, 1/2)$ be the union of $1/2$-nets covering each choice of $\cM$. Then by Lemma~5.2 in \cite{vershynin2010introduction}, we have
\begin{align*}
    |\cN|\leq \binom{m_{l_0}}{s} 5^{s}.
\end{align*}
Let 
\begin{align*}
    g(\ab) = \frac{\vb^\top}{\| \vb \|_2}  \Bigg(\prod_{r= l_1 }^{L } \bSigma_{r,i}\Wb_{r}^\top\Bigg) \ab,
\end{align*}
then similar to the proof of Lemma~\ref{lemma:lipschitzcontinuity_vector}, with probability at least $1 - \delta'$, 
\begin{align*}
    |g(\hat{\ab})| &\leq C_1 L\sqrt{ \frac{\log\big[5^{s} \cdot \binom{m_{l_0}}{s}/\delta' \big]}{m}}   \leq C_1 L\sqrt{ \frac{\log\big[5^{s} \cdot (e m_{l_0}/s)^s \big] + s}{m}}\leq C_2 L\sqrt{\frac{s\log(m)}{m}}
\end{align*}
for all $\hat{\ab} \in \cN$, where $C_1,C_2$ are absolute constants. For any $\ab\in S^{m_{l_0} - 1}$ with at most $s$ non-zero entries, there exists $\hat{\ab} \in \cN$ such that $\| \ab - \hat{\ab}\|_2 \leq 1/2$. Therefore, we have
\begin{align*}
    g(\ab) &\leq |g(\hat{\ab})| + |g(\ab) - g(\hat{\ab})|\\
    & \leq C_3 L\sqrt{\frac{s\log(m)}{m}}  + |g(\ab) - g(\hat{\ab})|.
\end{align*}
Since $g(\ab)$ is linear and $\| \vb \|_2 = \sqrt{m_L}$, we have
\begin{align*}
    |g(\ab) - g(\hat{\ab})| = |g(\ab - \hat{\ab})| \leq \| \ab - \hat{\ab} \|_2 \bigg|g\bigg(\frac{\ab - \hat{\ab}}{\| \ab - \hat{\ab} \|_2}\bigg)\bigg| \leq m_L^{-1/2} \| \ab - \hat{\ab} \|_2 \Gamma_{i,l,s} \leq m_L^{-1/2} \Gamma_{i,l,s}/2.
\end{align*}
Therefore, we have
\begin{align*}
    \Gamma_{i,l,s} \leq C_3 L\sqrt{s\log(m)},
    \end{align*}
where $C_3$ is an absolute constant. Applying union bound over all $i=1,\ldots, n$ and all $1\leq l  \leq L$ completes the proof.
\end{proof}

\begin{lemma}\label{lemma:lipschitzcontinuity_sparsevector}
For any $i=1,\ldots,n$ and $1\leq l_1 < l_2\leq L$, 
define
\begin{align*}
    \hat{\Lambda}_{i,l_1,l_2,s} = \sup_{ \substack{\ab\in S^{m_{l_1 - 1} - 1}, \mathrm{\mathbf{b}} \in S^{m_{l_2} - 1},\\ \| \ab \|_0,\| \mathrm{\mathbf{b}} \|_0\leq s}} \mathrm{\mathbf{b}}^\top \Wb_{l_2}^\top \Bigg(\prod_{r= l_1 }^{l_2 - 1} \bSigma_{r,i}\Wb_{r}^\top\Bigg) \ab.
\end{align*}
If $s \geq C\log(nL^2/\delta)$ for some absolute constant $C$, then with probability at least $1 - \delta$ we have
\begin{align*}
    \hat{\Lambda}_{i,l_1,l_2,s} \leq C' L\sqrt{\frac{s\log(m)}{m}}
\end{align*}
for all $i=1,\ldots,n$ and $1\leq l_1 < l_2\leq L$, where $C'$ is an absolute constant. 
\end{lemma}

\begin{proof}[Proof of Lemma~\ref{lemma:lipschitzcontinuity_sparsevector}]
The proof is similar to the proof of Lemma~\ref{lemma:lipschitzcontinuity_vector} and Lemma~\ref{lemma:lip_GD_sparse0}. 
Denote $l_0 = l_1 - 1$ and $\delta' = \delta/(nL^2)$.
Let $\cM_1$ and $\cM_2$ be fixed subspaces of $\RR^{m_{l_0}}$ and $\RR^{m_{l_2}}$ with sparsity $s$ respectively. Then there are $\binom{m_{l_0}}{s}$ choices of $\cM_1$ and $\binom{m_{l_2}}{s}$ choices of $\cM_2$. 
let $\cN_1 = \cup_{\cM_1}\cN[\cM_1, 1/4]$, $\cN_2 = \cup_{\cM_1}\cN[\cM_2, 1/4]$ be the union of $1/4$-nets covering each choice of $\cM_1$ and $\cM_2$ respectively. Then by Lemma~5.2 in \cite{vershynin2010introduction}, 
\begin{align*}
    |\cN_1|\leq \binom{m_{l_0}}{s} 9^{s},~|\cN_2|\leq \binom{m_{l_2}}{s} 9^{s}.
\end{align*}
Let 
\begin{align*}
    g(\ab, \mathrm{\mathbf{b}}) = \mathrm{\mathbf{b}}^\top \Wb_{l_2}^\top \Bigg(\prod_{r= l_1 }^{l_2 - 1} \bSigma_{r,i}\Wb_{r}^\top\Bigg) \ab
\end{align*}
% Note that by direct calculation it holds that
% \begin{align*}
%     \EE\big( \| \bSigma_{l,i}\Wb_l^\top \alpha \|_2^2 \big) = \| \alpha \|_2^2
% \end{align*}
% for all $\alpha\in S^{d-1}$. 
Similar to previous proofs, it can be shown that with probability at least $1 - \delta'/2$,
\begin{align*}
    \Bigg\| \Bigg(\prod_{r= l_1 }^{l_2 - 1} \bSigma_{r,i}\Wb_{r}^\top\Bigg) \hat{\ab} \Bigg\|_2 \leq C_0 L \sqrt{\frac{\log(2\cdot9^{m_{l_0}}/\delta')}{m}}
\end{align*}
for all $\hat{a}\in \cN_1$, where $C_0$ is an absolute constant. Therefore with probability at least $1 - \delta'$, we have
\begin{align*}
    |g(\hat{\ab},\hat{\mathrm{\mathbf{b}}})| &\leq C_1 L\sqrt{ \frac{\log\big[9^{2s} \cdot \binom{m_{l_0}}{s}\cdot \binom{m_{l_2}}{s} /\delta' \big]}{m}}  \\
    & \leq C_1 L\sqrt{ \frac{\log\big[9^{2s} \cdot (e m_{l_0}/s)^s\cdot (e m_{l_2}/s)^s /\delta' \big]}{m}}
    \\
    &\leq C_2 L\sqrt{ \frac{ s  \log(m_{l_0} m_{l_2}) + s}{m}}\\
    &\leq C_3 L\sqrt{\frac{s\log(m)}{m}}
\end{align*}
for all $\hat{\ab} \in \cN_1$ and $\hat{\mathrm{\mathbf{b}}} \in \cN_2$, where $C_1,C_2,C_3$ are absolute constants. For any $\ab\in S^{m_{l_0} - 1}$ and $\mathrm{\mathbf{b}} \in S^{m_{l_2} - 1}$ with at most $s$ non-zero entries, there exists $\hat{\ab} \in \cN_1$ and $\hat{\mathrm{\mathbf{b}}}\in \cN_2$ such that $\| \ab - \hat{\ab}\|_2 \leq 1/4$, $\| \mathrm{\mathbf{b}} - \hat{\mathrm{\mathbf{b}}} \|_2\leq 1/4$. Therefore, we have
\begin{align*}
    g(\ab, \mathrm{\mathbf{b}}) &\leq |g(\hat{\ab}, \hat{\mathrm{\mathbf{b}}})| + |g(\ab, \mathrm{\mathbf{b}}) - g(\hat{\ab}, \hat{\mathrm{\mathbf{b}}})|\\
    & \leq C_3 L\sqrt{\frac{s\log(m)}{m}}  + |g(\ab, \mathrm{\mathbf{b}}) - g(\hat{\ab}, \mathrm{\mathbf{b}})| + |g(\hat{\ab}, \mathrm{\mathbf{b}}) - g(\hat{\ab}, \hat{\mathrm{\mathbf{b}}})|.
\end{align*}
Since $g(\ab,\mathrm{\mathbf{b}})$ is a bilinear function, we have
\begin{align*}
    |g(\ab, \mathrm{\mathbf{b}}) - g(\hat{\ab}, \mathrm{\mathbf{b}})| = |g(\ab - \hat{\ab}, \mathrm{\mathbf{b}})| \leq \| \ab - \hat{\ab} \|_2 \bigg|g\bigg(\frac{\ab - \hat{\ab}}{\| \ab - \hat{\ab} \|_2}, \mathrm{\mathbf{b}}\bigg)\bigg| \leq \| \ab - \hat{\ab} \|_2 \tilde{\Lambda}_{i,l_1,l_2,s} \leq \hat{\Lambda}_{i,l_1,l_2,s}/4.
\end{align*}
Similarly, we also have
\begin{align*}
    |g(\ab, \mathrm{\mathbf{b}}) - g(\hat{\ab}, \mathrm{\mathbf{b}})| \leq \hat{\Lambda}_{i,l_1,l_2,s}/4.
\end{align*}
Therefore, we have
\begin{align*}
    \hat{\Lambda}_{i,l_1,l_2,s} \leq C_4 L\sqrt{\frac{s\log(m)}{m}},
\end{align*}
where $C_4$ is an absolute constant. Applying union bound over all $i=1,\ldots, n$ and all $1\leq l_1 < l_2 \leq L$ completes the proof.
\end{proof}

We now prove 
\ref{ThmResult:randinit_gradientuniformlowerbound} in Theorem~\ref{thm:randinit}.
In order to prove this result, we first show that the inner products between normalized hidden layer outputs are lower bounded by a constant related to $\mu$.

\begin{lemma}\label{lemma:randinit_gaussianproductmonotone}
For $\theta > 0$, let $Z_1$, $Z_2$ be two jointly Gaussian random variables with $\EE(Z_1) = \EE(Z_2) = 0$, $\EE(Z_1^2) = \EE(Z_2^2) = 1$.
Then
\begin{align*}
    \EE[\sigma(Z_1) \sigma(Z_2)] \geq \frac{1}{2} \EE[\sigma(Z_1) \sigma(Z_2)].
\end{align*}
\end{lemma}

\begin{proof}[Proof of Lemma~\ref{lemma:randinit_gaussianproductmonotone}]
The proof follows by direct calculation. By definition, we have
\begin{align*}
    \EE(Z_1 Z_2) 
    &= \EE(Z_1 Z_2  \ind\{ Z_1 \geq 0, Z_2 \geq 0 \} ) + \EE(Z_1 Z_2 \ind\{ Z_1 \leq 0, Z_2 \leq 0 \}) + \EE(Z_1 Z_2 \ind\{ Z_1 \leq 0, Z_2 \geq 0 \})  \\
    &\quad  + \EE(Z_1 Z_2 \ind\{ Z_1 \geq 0, Z_2 \leq 0 \}) \\
    & = 2\cdot \EE[\sigma(Z_1)\sigma(Z_2)] + 2\cdot \EE(Z_1 Z_2 \ind \{ Z_1 \leq 0, Z_2 \geq 0 \}).
\end{align*}
Rearranging terms gives
\begin{align*}
    \EE[\sigma(Z_1)\sigma(Z_2)] = \frac{1}{2} \EE(Z_1 Z_2) - \EE(Z_1 Z_2 \ind \{ Z_1 \leq 0, Z_2 \geq 0 \}) \geq \frac{1}{2} \EE(Z_1 Z_2).
\end{align*}
This completes the proof.
\end{proof}

\begin{lemma}\label{lemma:innerproductlowerboundinduction}
    Suppose that $\Wb_1,\ldots,\Wb_L$ are generated via Gaussian initialization. Let $\overline{\xb}_{l,i} = \xb_{l,i} / \| \xb_{l,i} \|_2$, $i= 1,\ldots,n$. For any $\delta > 0$, if $m = \geq C L^4\mu^{-4} \log(n^2L/\delta)$ for some large enough absolute constant $C$, then with probability at least $1-\delta$,
\begin{align*}
    \overline\xb_{l,i}^\top \overline\xb_{l,i'} \geq \mu^2 [1 - (2L)^{-1} \log(2)]^l
\end{align*}
for all $i,i'\in \{1,\ldots, n \}$ and all $l\in [L]$.
\end{lemma}

\begin{proof}[Proof of Lemma~\ref{lemma:innerproductlowerboundinduction}] %We first prove by induction that with high probability, $\xb_{l,i}^\top \xb_{l,i} \geq \mu^2 (1 - L)^l$ for all $l\in [L]$. 
We prove the result by induction. 
Suppose that $\overline\xb_{l,i}^\top \overline\xb_{l,i} \geq \mu^2 (1 - (2L)^{-1}\log(2))^l$. Then by Lemma~\ref{lemma:randinit_gaussianproductmonotone}, we have
\begin{align*}
    \EE(\xb_{l,i}^\top \xb_{l,i'} | \xb_{l-1,i}, \xb_{l-1,i'} ) \geq \frac{1}{2} \EE [ (\Wb_{l}^\top\xb_{l-1,i})^\top  ( \Wb_l \xb_{l-1,i'}) ]  = \xb_{l-1,i}^\top \xb_{l-1,i'}.
\end{align*}
Condition on $\xb_{l-1,i}$ and $\xb_{l-1,i'}$, we have
\begin{align*}
    \big\| \sigma(\wb_{l,j}^\top \xb_{l-1,i})\cdot \sigma(\wb_{l,j}^\top \xb_{l-1,i'}) \big\|_{\psi_1} &\leq C_1 \big\| [\sigma(\wb_{l,j}^\top \xb_{l-1,i})  \big\|_{\psi_2} \cdot  \big\|  \sigma(\wb_{l,j}^\top \xb_{l-1,i'})]^2 \big\|_{\psi_2} \\
    &\leq C_2 \| \xb_{l-1,i} \|_2 \| \xb_{l-1,i'} \|_2  /m_l,
\end{align*}
where $C_1,C_2$ are absolute constants. Note that we have $[1 - (2L)^{-1} \log(2)]^l \geq [1 - (2L)^{-1} \log(2)]^L \geq 1/2$. Therefore if $m \geq C L^4 \mu^{-4} \log(n^2L/\delta)$ for some large enough constant $C$, then by Bernstein inequality, union bound and \ref{ThmResult:randinit_normbounds}, with probability at least $1-\delta$ we have
\begin{align*}
    \xb_{l,i}^\top \xb_{l,i'} &\geq \xb_{l-1,i}^\top \xb_{l-1,i'} -  C_3 \| \xb_{l-1,i} \|_2 \| \xb_{l-1,i'} \|_2 \sqrt{ \frac{\log(n^2L / \delta)}{m_l}}\\
    &= \| \xb_{l-1,i} \|_2 \| \xb_{l-1,i'} \|_2  \overline\xb_{l-1,i}^\top \overline\xb_{l-1,i'} -  C_3 \| \xb_{l-1,i} \|_2 \| \xb_{l-1,i'} \|_2 \sqrt{ \frac{\log(n^2L / \delta)}{m_l}}.
\end{align*}
Therefore
\begin{align*}
    \overline\xb_{l,i}^\top \overline\xb_{l,i'} &\geq \xb_{l-1,i}^\top \xb_{l-1,i'} -  C_3 \| \xb_{l-1,i} \|_2 \| \xb_{l-1,i'} \|_2 \sqrt{ \frac{\log(n^2L / \delta)}{m_l}}\\
    &= \| \xb_{l,i} \|_2^{-1} \| \xb_{l,i'} \|_2^{-1} \| \xb_{l-1,i} \|_2 \| \xb_{l-1,i'} \|_2 \Bigg[ \overline\xb_{l-1,i}^\top \overline\xb_{l-1,i'} -  C_3 \sqrt{ \frac{\log(n^2L / \delta)}{m_l}} \Bigg] \\
%    & \geq \| \xb_{l-1,i} \|_2 \| \xb_{l-1,i'} \|_2 \big\{ \overline\xb_{l-1,i}^\top \overline\xb_{l-1,i'} - \mu^2\cdot [1 - (2L)^{-1} \log(2)]^{l-1}\cdot (2L)^{-1} \log(2) \big\} \\
    & \geq \mu^2 [1 - (2L)^{-1} \log(2)]^{l-1}\cdot [1 - (4L)^{-1} \log(2)] - \mu^2 [1 - (2L)^{-1} \log(2)]^{l-1}\cdot (4L)^{-1} \log(2)\\
    & \geq \mu^2 [1 - (2L)^{-1} \log(2)]^l.
\end{align*}
This completes the proof.
\end{proof}

\begin{corollary}\label{corollary:innerproductlowerbound}
    Suppose that $\Wb_1,\ldots,\Wb_L$ are generated via Gaussian initialization. Let $\overline{\xb}_{l,i} = \xb_{l,i} / \| \xb_{l,i} \|_2$, $i= 1,\ldots,n$. For any $\delta > 0$, if $m \geq C L^4\mu^{-4} \log(n^2L/\delta)$ for some large enough absolute constant $C$, then with probability at least $1-\delta$,
\begin{align*}
    \overline\xb_{l,i}^\top \overline\xb_{l,i'} \geq \mu^2 /2.
\end{align*}
for all $i,i'\in \{1,\ldots, n \}$ and all $l\in [L]$.
\end{corollary}
\begin{proof}[Proof of Corollary~\ref{corollary:innerproductlowerbound}]
It directly follows by Lemma~\ref{lemma:innerproductlowerboundinduction} and the fact that $1 - x/2 \geq e^{-x}$ for $x\in [0,\log(2)]$.
\end{proof}

%%%%%%%%%%%%%%%%%%%%%%%%%%%%%%%%%%%%%%%%%%%%%%%%%%%%%%%%%%%%%%%%%%%%%%%%%%%%%%%%%%%%%%%%%%%%%%%%%%%%%%%%%%%%%%%%%%%%%%%%%%%%%%%%%%%%%%%%

\begin{lemma}\label{lemma:li2018}
Let $\zb_1,\ldots,\zb_n\in S^{d-1}$ be $n$ unit vectors and $y_1,\ldots,y_n \in\{-1,1\}$ be the corresponding labels. Assume that  for any $i\neq j$ such that $y_i\neq y_j$, $\|\zb_i-\zb_j\|_2\ge \tilde\phi$ and $\zb_i^\top \zb_j \geq \tilde\mu^2$ for some $\tilde\phi,\tilde{\mu} > 0$. 
For any $\ab = (a_1,\ldots,a_n)^\top \in \RR_+^{n}$, let
%Given positive constants $a_1,a_2,\dots, a_n$, let $\bar a = \max_{i\in[n]}a_i$, and 
$\hb(\wb) = \sum_{i=1}^n a_iy_i\sigma'(\la\wb,\zb_i\ra)\zb_i$ where $\wb\sim N(\mathbf{0},\Ib)$ is a Gaussian random vector. If $\tilde\phi \leq \tilde\mu/2$, then there exist absolute constants $C,C'>0$ such that
\begin{align*}
    \PP\big[\|\hb(\wb)\|_2\ge C \| \ab \|_\infty \big]\ge C'\tilde\phi/n.
\end{align*}
\end{lemma}
\begin{proof}
Without loss of generality, assume that $a_1 = \| \ab \|_\infty$. Since $ \| \zb_1 \|_2 = 1 $, we can construct an orthonormal matrix $\Qb = [ \zb_1, \Qb' ] \in \RR^{d\times d}$. Let $\ub = \Qb^\top \wb \sim N(\mathbf{0}, \Ib)$ be a standard Gaussian random vector. Then we have
\begin{align*}
    \wb = \Qb \ub = u_1 \zb_1 + \Qb' \ub',
\end{align*}
where $\ub' := (u_2,\ldots,u_d)^\top$ is independent of $u_1$. 
% Then the Gaussian random vector $\wb$ can be decomposed as follows,
% \begin{align*}
% \wb = \wb_{\!/\mkern-5mu/\!} + \wb_{\bot},
% \end{align*}
% where $\wb_{\!/\mkern-5mu/\!} = \la\wb,\zb_1\ra \zb_1$ and $\la\wb_{\bot},\zb_1\ra = 0$. 
We define the following two events based on a parameter $\gamma\in (0,1]$:
\begin{align*}
    \cE_1(\gamma) = \big\{|u_1|\le \gamma \big\},~
    \cE_2(\gamma) = \big\{ |\la\Qb'\ub',\zb_i\ra|\ge \gamma \text{ for all } \zb_i \mbox{ such that } \|\zb_i - \zb_1\|_2\ge \tilde\phi \big\}.
\end{align*}
Let $\cE(\gamma) = \cE_1(\gamma) \cap \cE_2(\gamma) $. 
We first give lower bound for $\PP(\cE) = \PP(\cE_1) \PP(\cE_2)$. Since $u_1$ is a standard Gaussian random variable, we have
\begin{align*}
    \PP(\cE_1) = \frac{1}{\sqrt{2\pi }}\int_{-\gamma}^{\gamma} \exp\bigg( -\frac{1}{2}x^2 \bigg) \mathrm{d} x \geq \sqrt{\frac{2}{\pi e}} \gamma.
\end{align*}
Moreover, by definition, for any $i=1,\ldots, n$ we have
\begin{align*}
    \la \Qb'\ub' , \zb_i \ra \sim N\big[ 0, 1 - (\zb_1^\top \zb_i)^2 \big].
\end{align*}
Let $\cI = \{i: \|\zb_i-\zb_1\|_2\ge \tilde\phi\}$. By the assumption that $\phi \leq \tilde\mu/2$, for any $i \in \cI$, we have 
$$
-1 + \tilde\phi^2/2\leq -(1 - \tilde\mu^2) + \tilde\mu^2 \leq \la \zb_i , \zb_1 \ra \leq 1 - \tilde\phi^2/2,
$$
and if $\tilde\phi^2 \leq 2$, then
\begin{align*}
    1 - (\zb_1^\top \zb_i)^2 \geq \tilde\phi^2 - \tilde\phi^4/4 \geq \tilde\phi^2 / 2.
\end{align*}
Therefore for any $i\in \cI$, 
\begin{align*}
    \PP[ | \la \Qb'\ub' , \zb_i \ra | < \gamma ] %= \PP \bigg[ \bigg| \frac{\la \Qb'\ub' , \zb_i \ra}{ 1 - (\zb_1^\top \zb_i)^2} \bigg| < \frac{\gamma}{ 1 - (\zb_1^\top \zb_i)^2} \bigg] 
    = \frac{1}{\sqrt{2\pi }}\int_{-[1 - (\zb_1^\top \zb_i)^2]^{-1/2}\gamma}^{[1 - (\zb_1^\top \zb_i)^2]^{-1/2} \gamma} \exp\bigg( -\frac{1}{2}x^2 \bigg) \mathrm{d} x 
    \leq  \sqrt{\frac{2}{\pi}} \frac{\gamma}{[1 - (\zb_1^\top \zb_i)^2]^{1/2}}\leq \frac{2}{\sqrt{\pi}}\gamma \tilde\phi^{-1}
    .
\end{align*}
By union bound over $\cI$, we have
\begin{align*}
    \PP(\cE_2) = \PP[ | \la \Qb'\ub' , \zb_i \ra | \geq \gamma, i\in \cI ] \geq 1 -  \frac{2}{\sqrt{\pi}} n \gamma \tilde\phi^{-1}.
\end{align*}
Therefore we have
\begin{align*}
    \PP(\cE) \geq \sqrt{\frac{2}{\pi e}} \gamma \cdot \bigg( 1 - \frac{2}{\sqrt{\pi}} n \gamma \tilde\phi^{-1} \bigg).
\end{align*}
Setting $\gamma = \sqrt{\pi} \tilde\phi / (4n)$, we obtain $\PP(\cE) \geq \tilde\phi / ( \sqrt{32e} n)$. 
Now let  $\cI' = [n]\setminus (\cI\cup\{1\})$.
Then conditioning on event $\cE$, we have
\begin{align}\label{eq:grad_lower_formula}
\hb(\wb) &= \sum_{i=1}^n a_iy_i\sigma'(\la\wb,\zb_i\ra)\zb_i\notag\\
&= a_1y_1\sigma'(u_1)\zb_1 + \sum_{i\in\cI}a_iy_i\sigma'\big(u_1\la\zb_1,\zb_i\ra+\la\Qb'\ub',\zb_i\ra\big)\zb_i%\notag\\
%&\qquad 
+ \sum_{i\in\cI' } a_iy_i\sigma'\big(u_1\la\zb_1,\zb_i\ra+\la\Qb'\ub,\zb_i\ra\big)\zb_i\notag\\
&=a_1y_1\sigma'(u_1)\zb_1 + \sum_{i\in\cI}a_iy_i\sigma'\big(\la\Qb'\ub',\zb_i\ra\big)\zb_i + \sum_{i\in\cI' } a_iy_i\sigma'\big(u_1\la\zb_1,\zb_i\ra+\la\Qb'\ub',\zb_i\ra\big)\zb_i,
\end{align}
where the last equality follows from the fact that conditioning on event $\cE$, for all $i\in \cI$, it holds that $|\la\Qb'\ub',\zb_i\ra|\ge  |u_1| \ge |u_1 \la\zb_1,\zb_i\ra|$. We then consider two cases: $u_1 > 0$ and $u_1 < 0$, which occur equally likely conditioning on $\cE$. %with probability exactly $1/2$ respectively. 
%Since the term $\sum_{i\in\cI}a_iy_i\sigma'(\la \Qb'\ub' ,\xb_i\ra)\zb_i$ is independent of $u_1$,  we have% the following according to \eqref{eq:grad_lower_formula}
Therefore we have
\begin{align*}
\PP\bigg[\|\hb(\wb)\|_2\ge \inf_{{u_1^{(1)}>0,  u_1^{(2)}<0}}\max\big\{\big\|\hb(u_1^{(1)}\zb_1 + \Qb'\ub' )\big\|_2,\big\|\hb(u_1^{(2)}\zb_1 + \Qb'\ub')\big\|_2\big\} \bigg| \cE\bigg]\ge 1/2 .
\end{align*}
By the inequality  $\max\{\|\ab\|_2,\|\bbvec\|_2\}\ge \|\ab-\bbvec\|_2/2$, we have
\begin{align}\label{eq:prob_gradient_lower}
\PP\bigg[\|\hb(\wb)\|_2\ge \inf_{{u_1^{(1)}>0,  u_1^{(2)}<0}} \big\|\hb(u_1^{(1)}\zb_1 + \Qb'\ub' ) - \hb(u_1^{(2)}\zb_1 + \Qb'\ub') \big\|_2 / 2 \bigg| \cE\bigg]\ge 1/2 .
\end{align}
For any $u_1^{(1)} > 0$ and $u_1^{(2)} < 0$, denote $\wb_1 = u_1^{(1)}\zb_1 + \Qb'\ub'$, $\wb_2 = u_1^{(2)}\zb_1 + \Qb'\ub'$. We now proceed to give lower bound for $\|\hb(\wb_1) - \hb(\wb_2)\|_2$. 
By \eqref{eq:grad_lower_formula}, we have
%$\wb_1$ and $\wb_2$ satisfying $\la\wb_1,\xb_1\ra> 0$ and $\la\wb_2,\xb_1\ra<0$,
\begin{align}\label{eq:two_grad_difference}
\hb(\wb_1) - \hb(\wb_2) = a_1y_1\zb_1 + \sum_{i\in\cI'} a_i' y_i\zb_i ,
\end{align}
where
\begin{align*}
    a_i' = a_i\big[\sigma'\big( u_1^{(1)}\la\zb_1,\zb_i\ra+\la\Qb'\ub',\zb_i\ra\big)-\sigma'\big(u_1^{(2)}\la\zb_1,\zb_i\ra+\la\Qb'\ub',\zb_i\ra\big)\big].
\end{align*}
Note that for all $i\in \cI'$, we have $y_i = y_1$ and $\la\zb_1,\zb_i\ra\ge 1-\tilde\phi^2/2\ge 0$. Therefore, since $u_1^{(1)} > 0 > u_1^{(2)} $, we have 
\begin{align*}
\sigma'\big(u_1^{(1)}\la\zb_1,\zb_i\ra+\la\Qb'\ub',\zb_i\ra\big)-\sigma'\big(u_1^{(2)}\la\zb_1,\zb_i\ra+\la\Qb'\ub',\zb_i\ra\big)\ge 0.    
\end{align*}
Therefore $a_i' \geq 0$ for all $i\in \cI'$ and 
%Then we can further rewrite \eqref{eq:two_grad_difference} as follows,
\begin{align*}
\hb(\wb_1) - \hb(\wb_2) = a_1y_1\zb_1 + \sum_{i\in\cI'} a_i'y_1\zb_i = y_1\bigg(a_1\zb_1+\sum_{i\in\cI'} a_i'\zb_i\bigg),
\end{align*}
% where 
% \begin{align*}
% a_i' = a_i\big[\sigma'\big(\la\wb_1,\zb_1\ra\la\zb_1,\zb_i\ra+\la\wb_\bot,\xb_i\ra\big)-\sigma'\big(\la\wb_2,\zb_1\ra\la\zb_1,\zb_i\ra+\la\wb_\bot,\xb_i\ra\big)\big]\ge 0.
% \end{align*}
We have shown that $\la\zb_i,\zb_1\ra \geq 0$ for all $i\in\cI'$. Therefore we have 
\begin{align*}
\|\hb(\wb_1) - \hb(\wb_2)\|_2\ge 
\Bigg\| y_1\Bigg(a_1\zb_1+\sum_{i\in\cI'} a_i'\zb_i\Bigg) \Bigg\|_2 \geq \Bigg\la a_1\zb_1+\sum_{i\in\cI'} a_i'\zb_i,\zb_1 \Bigg\ra \geq  a_1.
\end{align*}
Since the inequality above holds for any $u_1^{(1)} > 0$ and $u_1^{(2)} < 0$, taking infimum gives
\begin{align}\label{eq:difference_gradient_norm}
\inf_{u_1^{(1)}>0,u_1^{(2)}<0}\|\hb(\wb_1) - \hb(\wb_2)\|_2\geq a_1.
\end{align}
Plugging \eqref{eq:difference_gradient_norm} back to \eqref{eq:prob_gradient_lower}, we obtain
\begin{align*}
\PP\big[\|\hb(\wb)\|_2\ge a_1/2 \big|\cE\big]\ge 1/2, 
\end{align*}
%where we neglect the minimization operation since \eqref{eq:difference_gradient_norm} holds for any $\wb_1$ and $\wb_2$. Then applying 
Since $a_1 = \| \ab \|_\infty $ and $\PP(\cE) \ge \tilde\phi/(\sqrt{32e}n)$, we have
\begin{align*}
\PP\big[\|\hb(\wb)\|_2\ge C \| \ab \|_\infty \big]\ge C'\tilde\phi/n,
\end{align*}
where $C$ and $C'$ are absolute constants. This completes the proof.
\end{proof}

%%%%%%%%%%%%%%%%%%%%%%%%%%%%%%%%%%%%%%%%%%%%%%%%%%%%%%%%%%%%%%%%%%%%%%%%%%%%%%%%%%%%%%%%%%%%%%%%%%%%%%%%%%%%%%%%%%%%%%%%%%%%%%%%%%%%%%%%

\begin{lemma}\label{lemma:li2018_uniform}
There exist absolute constants $C,C',C'',C''' > 0$ such that, if $m\geq C n^2\phi^{-1}\log(4n)$, then
with probability at least $1- \exp(-C' m_L \phi /n )$, for any $\ab = (a_1,\ldots,a_n)^\top \in \RR_+^n$, there exist at least $ C'' m_L \phi / n$ nodes that satisfy
\begin{align*}
    \Bigg\|\frac{1}{n}\sum_{i=1}^n a_i y_i \sigma'(\la\wb_{L,j},\xb_{L-1,i}\ra)\xb_{L-1,i}\Bigg\|_2\ge C'''\|\ab \|_\infty/n
\end{align*}
\end{lemma}
\begin{proof}[Proof of Lemma~\ref{lemma:li2018_uniform}]
%The proof is based on the following lemma reformulated from a result given in \citet{li2018learning}.

% \begin{lemma}[Lemma A.3 in \citet{li2018learning}] \label{lemma:li2018}
% Let $\zb_1,\ldots,\zb_n\in S^{d-1}$ be $n$ unit vectors such that $\| \zb_i - \zb_j \|_2 \geq \phi$ for any $i\neq j$, $i,j\in \{1,\ldots, n\}$. Given constants $a_1,a_2,\dots, a_n$, let $\bar a = \max_{i\in[n]}|a_i|$, and $\hb(\wb) = \sum_{i}^n a_i\sigma'(\la\wb,\zb_i\ra)\zb_i$ where $\wb\sim N(\mathbf{0},\Ib)$ is a random vector. Then there exist absolute constants $C,C'>0$  such that
% \begin{align*}
%     \PP\big[\|\hb(\wb)\|_2\ge C\bar a\phi/n\big]\ge C'\phi/n.
% \end{align*}
% \end{lemma}
For any given $j\in \{ 1,\ldots, m_L \}$ and $\hat{\ab}$ with $\| \hat{\ab} \|_\infty = 1$, by Lemma~\ref{lemma:li2018}, we have
\begin{align*}
\PP\Bigg[\Bigg\|\frac{1}{n}\sum_{i=1}^n \hat{a}_i y_i \sigma'(\la\wb_{L,j},\xb_{L-1,i}\ra)\xb_{L-1,i}\Bigg\|_2\ge \frac{C_1}{n} \Bigg]\ge \frac{C_2\phi}{n},
\end{align*}
where $C_1,C_2>0$ are absolute constants. Let $S_{\infty,+}^{n-1} = \{ \ab\in \RR_+^{n}: \| \ab \|_{\infty} = 1 \}$, and $\cN = \cN[S_{\infty,+}^{n-1}, C_1/(4n)]$ be a $C_1/(4n)$-net covering $S_{\infty,+}^{n-1}$ in $\ell_{\infty}$ norm. Then we have
\begin{align*}
    |\cN| \leq (4n /  C_1)^n. 
\end{align*}
For $j=1,\ldots, m_L$, define
\begin{align*}
    Z_j = \ind\Bigg[ \Bigg\|\frac{1}{n}\sum_{i=1}^n \hat{a}_i y_i \sigma'(\la\wb_{L,j},\xb_{L-1,i}\ra)\xb_{L-1,i}\Bigg\|_2\ge \frac{C_1}{n} \Bigg].
\end{align*}
Let $p_\phi = C_2\phi/n$. Then by Bernstein inequality and union bound, with probability at least $1 - \exp[-C_3 m_L p_\phi + n \log(4n/C_1)] \geq 1 - \exp(C_4 m_L \phi / n)$, we have
\begin{align}\label{eq:li2018_uniform_eq1}
    \frac{1}{m_L} \sum_{j=1}^{m_L} Z_j \geq p_\phi / 2,
\end{align}
where $C_3 ,C_4$ are absolute constants. For any $\ab \in S_{\infty,+}^{n-1}$, there exists $\hat{\ab}\in \cN$ such that 
\begin{align*}
    \| \ab - \hat{\ab} \|_{\infty} \leq C_1 / (4n).
\end{align*}
Therefore, we have
\begin{align}
&\Bigg\vert\Bigg\|\frac{1}{n}\sum_{i=1}^n a_i y_i\sigma'(\la\wb_{L,j},\xb_{L-1,i}\ra)\xb_{L-1,i}\Bigg\|_2 - \Bigg\|\frac{1}{n}\sum_{i=1}^n \hat{a}_i y_i \sigma'(\la\wb_{L,j},\xb_{L-1,i}\ra)\xb_{L-1,i}\Bigg\|_2\Bigg\vert \nonumber \\
&\le \Bigg\|\frac{1}{n}\sum_{i=1}^n a_i y_i \sigma'(\la\wb_{L,j},\xb_{L-1,i}\ra)\xb_{L-1,i}-\frac{1}{n}\sum_{i=1}^n \hat{a}_i y_i \sigma'(\la\wb_{L,j},\xb_{L-1,i}\ra)\xb_{L-1,i}\Bigg\|_2 \nonumber \\
&\le \frac{2}{n}\sum_{i=1}^n|a_i-\hat{a}_i| \nonumber \\
&\le \frac{C_1}{2n}.\label{eq:li2018_uniform_eq2}
\end{align}
By \eqref{eq:li2018_uniform_eq1} and \eqref{eq:li2018_uniform_eq2}, it is clear that with probability at least  $1 - \exp(C_4 m_L \phi / n)$, for any $\ab \in S_{\infty,+}^{n-1}$,  there exist at least $m_L p_\phi / 2$ nodes on layer $L$ that satisfy
\begin{align*}
    \Bigg\|\frac{1}{n}\sum_{i=1}^n a_i y_i \sigma'(\la\wb_{L,j},\xb_{L-1,i}\ra)\xb_{L-1,i}\Bigg\|_2\ge \frac{C_1 }{2n}.
\end{align*}
This completes the proof.
\end{proof}

%% file: appendix2.tex
\section{Proof of Theorem \ref{thm:perturbation}}
It is clear that the bound on $\tilde \Wb_l$ can be proved trivially, since by triangle inequality we have $\|\tilde \Wb\|_2\le \|\Wb\|_2 + \|\tilde \Wb - \Wb\|_2\le C$ for some large enough absolute constant $C$. In the following we directly use this result without referring to Theorem \ref{thm:perturbation}.
Similar to the proof of Theorem \ref{thm:randinit}, we split the rest results into the following lemmas and prove them separately. %Throughout this section, we always assume that $\{\Wb_{l}\}_{l=1}^L$ are generated via Gaussian  initialization and the results of Theorem~\ref{thm:randinit} hold. Therefore we omit this assumption in the lemmas we provide in this section.

\begin{lemma}\label{lemma:perturbationmatrixnormbound}
Suppose that $\Wb_1,\ldots,\Wb_L$ are generated via Gaussian initialization, and all results \ref{ThmResult:randinit_normbounds}-\ref{ThmResult:randinit_gradientuniformlowerbound} in Theorem~\ref{thm:randinit} hold. 
For $\tau > 0$, let $ \tilde\Wb_1,\ldots, \tilde\Wb_L$ with $\|\tilde \Wb_l-\Wb_l\|_2\le \tau$, $l=1,\ldots,L$ be the perturbed matrices. Let $\tilde \bSigma_{l,i}$, $ l=1,\ldots,L$, $i=1,\ldots,n$ be diagonal matrices satisfying  $ \| \tilde \bSigma_{l,i} - \bSigma_{l,i} \|_0 \leq s$ and $|(\tilde \bSigma_{l,i} - \bSigma_{l,i})_{jj}|,|(\tilde \bSigma_{l,i})_{jj}| \leq 1$ for all $l=1,\ldots,L$, $i=1,\ldots,n$ and $j=1,\ldots,m_l$. If $\tau, \sqrt{s\log(M)/m} \leq \kappa L^{-3}$ for some small enough absolute constant $\kappa$, then
\begin{align*}
    \Bigg\| \prod_{r=l_1}^{l_2}\tilde \bSigma_{r,i}\tilde \Wb_r^\top \Bigg\|_2 \leq C L 
\end{align*}
for any $1 \leq l_1 < l_2 \leq L $, where $C$ is an absolute constant. 
\end{lemma}
\begin{proof}
Note that we have
\begin{align*}
    \tilde{\bSigma}_{r,i} \tilde{\Wb}_{r}^\top =  \bSigma_{r,i} W_{r}^\top + (\tilde{\bSigma}_{r,i} - \bSigma_{r,i}) \Wb_r^\top + \tilde{\bSigma}_{r,i}(\tilde{\Wb}_r - \Wb_r)^\top.
\end{align*}
Therefore, let $\cA_{r,i} = \{\bSigma_{r,i} \Wb_{r}^\top, (\tilde{\bSigma}_{r,i} - \bSigma_{r,i}) \Wb_r^\top, \tilde{\bSigma}_{r,i}(\tilde{\Wb}_r - \Wb_r)^\top  \}$, $r=l_1,\ldots,l_2$, then we have
\begin{align*}
    \prod_{r=l_1}^{l_2}\tilde \bSigma_{r,i}\tilde \Wb_r^\top = \sum_{\Ab_{l_1,i}\in \cA_{l_1,i},\ldots, \Ab_{l_2,i}\in \cA_{l_2,i}} \Bigg( \prod_{r=l_1}^{l_2} \Ab_{r,i} \Bigg).
\end{align*}
For the rest of the proof, we denote by $|\bSigma|$ the diagonal matrix with absolute values of elements of $\bSigma$ on the corresponding entries. For each sequence $\Ab_{l_1,i},\ldots,\Ab_{l_2,i}$, denote
\begin{align*}
    \hat{\bSigma}_{r,i} = \left\{ 
    \begin{array}{ll}
        | \tilde{\bSigma}_{r,i} - \bSigma_{r,i} |, & \text{if }r\geq 2 \text{ and }\Ab_{r-1,i} =(\tilde{\bSigma}_{r-1,i} - \bSigma_{r-1,i}) \Wb_{r-1}^\top, \\
        \Ib, & \text{otherwise}.
    \end{array}
    \right.
\end{align*}
Then we have
\begin{align*}
    \prod_{r=l_1}^{l_2} \Ab_{r,i}  = \prod_{r=l_1}^{l_2} \Ab_{r,i} \hat{\bSigma}_{r,i} .
\end{align*}
When $\Ab_{r,i} = \bSigma_{r,i} \Wb_{r}$ for all $r= l_1,\ldots,l_2$, then the bound of $\| \prod_{r=l_1}^{l_2} \Ab_{r,i} \|_2 $ is given by \ref{ThmResult:randinit_matproductnorm} in Theorem~\ref{thm:randinit}. For all the other terms in the expansion, we consider sequences of the form $ \Bb_{r_2,i} (\prod_{r = r_1+1}^{r_2 - 1} \bSigma_{r,i} \Wb_r^\top) \Bb_{r_1,i} $, where 
\begin{align*}
&\Bb_{r_2,i} \in \{ (\tilde{\bSigma}_{r_2,i} - \bSigma_{r_2,i}) \Wb_{r_2}^\top, \tilde{\bSigma}_{r_2,i}(\tilde{\Wb}_{r_2} - \Wb_{r_2})^\top  \}, \\
&\Bb_{r_1,i} \in \{ | \tilde{\bSigma}_{r_1,i} - \bSigma_{r_1,i} |, \tilde{\bSigma}_{r_1,i}(\tilde{\Wb}_{r_1} - \Wb_{r_1})^\top  \}.  
\end{align*}
By \ref{ThmResult:randinit_matproductdoublesparsenorm} in Theorem~\ref{thm:randinit}, there exists an absolute constant $C_1$ such that $ \| \Bb_{r_2,i} (\prod_{r = r_1+1}^{r_2 - 1} \bSigma_{r,i} \Wb_r^\top) \Bb_{r_1,i} \|_2 $ with different choices of $\Bb_{r_1,i}$ and $\Bb_{r_2,i}$ have the following bounds:
\begin{enumerate}
    \item If $\Bb_{r_1,i} = | \tilde{\bSigma}_{r_1,i} - \bSigma_{r_1,i} | $, $\Bb_{r_2,i} = (\tilde{\bSigma}_{r_2,i} - \bSigma_{r_2,i}) \Wb_{r_2}^\top$, then  $ \| \Bb_{r_2,i} (\prod_{r = r_1+1}^{r_2 - 1} \bSigma_{r,i} \Wb_r^\top) \Bb_{r_1,i} \|_2 \leq C_1 L \sqrt{s\log(M)/m}$.
    \item If $\Bb_{r_1,i} =  \tilde{\bSigma}_{r_1,i}(\tilde{\Wb}_{r_1} - \Wb_{r_1})^\top $, $\Bb_{r_2,i} =  \tilde{\bSigma}_{r_2,i}(\tilde{\Wb}_{r_2} - \Wb_{r_2})^\top $, then  $ \| \Bb_{r_2,i} (\prod_{r = r_1+1}^{r_2 - 1} \bSigma_{r,i} \Wb_r^\top) \Bb_{r_1,i} \|_2 \leq C_1 L \tau^2$
    \item Otherwise, $ \| \Bb_{r_2,i} (\prod_{r = r_1+1}^{r_2 - 1} \bSigma_{r,i} \Wb_r^\top) \Bb_{r_1,i} \|_2 \leq C_1 L \tau$.
\end{enumerate}
For any fixed sequence $\Ab_{l_1,i} \in \cA_{l_1,i},\ldots, \Ab_{l_2,i} \in \cA_{l_2,i}$, let 
\begin{align*}
    p_1 = \big|\{ r:\Ab_{r,i} = \tilde{\bSigma}_{r,i}(\tilde{\Wb}_r - \Wb_r)^\top \}\big|,p_2 = \big|\{ r:\Ab_{r,i} = (\tilde{\bSigma}_{r,i} - \bSigma_{r,i})\Wb_r^\top \}\big|,p_3 = \big|\{ r:\Ab_{r,i} = \bSigma_{r,i}\Wb_r^\top \}\big|.
\end{align*}
Then by the discussion above we see that the bound of $\| \prod_{r=l_1}^{l_2} \Ab_{r,i} \|_2 $ has a term $ (C_1 L \tau)^{p_1} $ granted by the matrices of the form $\tilde{\bSigma}_{r,i}(\tilde{\Wb}_r - \Wb_r)^\top $. In addition, if $p_2 > p_1+1$, then the bound also has a term $C_1L\sqrt{s\log(M)/m}$ with power at least $p_2 - p_1 - 1$. Note that when $p_1 = p_2 = 0$, we still have $\| \prod_{r=l_1}^{l_2} \Ab_{r,i} \|_2 \leq C_2 L $ for some absolute constant $C_2$. Therefore, 
\begin{align*}
    \Bigg\| \prod_{r=l_1}^{l_2}\tilde \bSigma_{r,i}\tilde \Wb_r^\top \Bigg\|_2 &\leq C_2 L + \sum_{ \substack{ p_1 + p_2 + p_3 = l_2 - l_1 + 1\\ p_1,p_2,p_3 \geq 0}} \frac{(l_2 - l_1 + 1)!}{p_1! p_2! p_3!}\cdot (C_1 L\tau)^{p_1} \cdot \big( C_1 L \sqrt{s\log(m)/m} \big)^{\max\{p_2 - p_1 - 1,0\}} \\
    &\leq C_2 L + \sum_{ \substack{ p_1 + p_2 + p_3 = l_2 - l_1 + 1\\ p_1,p_2,p_3 \geq 0}} \frac{(l_2 - l_1 + 1)!}{p_1! p_2! p_3!}\cdot L^{ -2 p_1 }\cdot L^{ -2 \max\{p_2 - p_1 - 1, 0 \} } \\
    &\leq C_2 L + \sum_{ \substack{ p_1 + p_2 + p_3 = l_2 - l_1 + 1\\ p_1,p_2,p_3 \geq 0}} \frac{(l_2 - l_1 + 1)!}{p_1! p_2! p_3!}\cdot L^{ -2 \max\{p_2 - 1 , p_1 \} }\\
    &\leq C_2 L + L\cdot \sum_{ \substack{ p_1 + p_2 + p_3 = l_2 - l_1 + 1\\ p_1,p_2,p_3 \geq 0}} \frac{(l_2 - l_1 + 1)!}{p_1! p_2! p_3!}\cdot L^{ - (p_1 + p_2) }\cdot 1^{p_3}\\
    &\leq C_2 L + L\cdot (1 + 2L^{-1})^L \\
    &\leq (C_2 + e^2) L .
\end{align*}
This completes the proof.
\end{proof}

\begin{lemma}\label{lemma:perturbationvecmatrixsparsenormbound}
Suppose that $\Wb_1,\ldots,\Wb_L$ are generated via Gaussian initialization, and all results \ref{ThmResult:randinit_normbounds}-\ref{ThmResult:randinit_gradientuniformlowerbound} in Theorem~\ref{thm:randinit} hold. 
For $\tau > 0$, let $ \tilde\Wb_1,\ldots, \tilde\Wb_L$ with $\|\tilde \Wb_l-\Wb_l\|_2\le \tau$, $l=1,\ldots,L$ be the perturbed matrices. Let $\tilde \bSigma_{l,i}$, $ l=1,\ldots,L$, $i=1,\ldots,n$ be diagonal matrices satisfying  $ \| \tilde \bSigma_{l,i} - \bSigma_{l,i} \|_0 \leq s$ and $|(\tilde \bSigma_{l,i} - \bSigma_{l,i})_{jj}|,|(\tilde \bSigma_{l,i})_{jj}| \leq 1$ for all $l=1,\ldots,L$, $i=1,\ldots,n$ and $j=1,\ldots,m_l$. If $\tau, \sqrt{s\log(M)/m} \leq \kappa L^{-3}$ for some small enough absolute constant $\kappa$, then
\begin{align*}
     \vb^\top \Bigg(\prod_{r=l}^{L}\tilde \bSigma_{r,i}\tilde \Wb_r^\top \Bigg) \ab \leq C [L^2\tau\sqrt{M} + L \sqrt{s\log(M)}].
\end{align*}
for any $\ab\in \RR^{m_{l - 1}}$ satisfying $\| \ab \|_2 = 1$, $\| \ab \|_0 \leq s$ and any $1 \leq  l \leq L $, where $C$ is an absolute constant. 
\end{lemma}
\begin{proof}
The proof is similar to the proof of Lemma~\ref{lemma:perturbationmatrixnormbound}. Again, 
% Note that we have
% \begin{align*}
%     \tilde{\bSigma}_{r,i} \tilde{W}_{r}^\top =  \bSigma_{r,i} W_{r}^\top + (\tilde{\bSigma}_{r,i} - \bSigma_{r,i}) \Wb_r^\top + \tilde{\bSigma}_{r,i}(\tilde{\Wb}_r - \Wb_r)^\top.
% \end{align*}
% Therefore, 
let $\cA_{r,i} = \{\bSigma_{r,i} W_{r}^\top, (\tilde{\bSigma}_{r,i} - \bSigma_{r,i}) \Wb_r^\top, \tilde{\bSigma}_{r,i}(\tilde{\Wb}_r - \Wb_r)^\top  \}$, $r=l,\ldots,L$, then we have
\begin{align*}
    \vb^\top \Bigg( \prod_{r=l}^{L}\tilde \bSigma_{r,i}\tilde \Wb_r^\top \Bigg) \ab = m_L^{1/2}\cdot \sum_{\Ab_{l,i}\in \cA_{l,i},\ldots, \Ab_{L,i}\in \cA_{L,i}} m_L^{-1/2}\vb^\top \Bigg( \prod_{r=l}^{L} \Ab_{r,i} \Bigg) \ab.
\end{align*}
Similar to the proof of Lemma~\ref{lemma:perturbationmatrixnormbound}, we denote by $|\bSigma|$ the diagonal matrix with absolute values of elements of $\bSigma$ on the corresponding entries. For each sequence $\Ab_{l,i},\ldots,\Ab_{L,i}$, denote
\begin{align*}
    \hat{\bSigma}_{r,i} = \left\{ 
    \begin{array}{ll}
        | \tilde{\bSigma}_{r,i} - \bSigma_{r,i} |, & \text{if }r\geq 2 \text{ and }\Ab_{r-1,i} =(\tilde{\bSigma}_{r-1,i} - \bSigma_{r-1,i}) \Wb_{r-1}^\top, \\
        \Ib, & \text{otherwise}.
    \end{array}
    \right.
\end{align*}
Then we have
\begin{align*}
    \prod_{r=l}^{L} \Ab_{r,i}  = \prod_{r=l}^{L} \Ab_{r,i} \hat{\bSigma}_{r,i} .
\end{align*}
When $\Ab_{r,i} = \bSigma_{r,i} \Wb_{r}$ for all $r= l,\ldots,L$, then the bound of $m_L^{-1/2}\vb^\top \big( \prod_{r=l}^{L} \Ab_{r,i} \big) \ab $ is given by \ref{ThmResult:randinit_matproductnorm} in Theorem~\ref{thm:randinit}. For all the other terms in the expansion, we consider sequences of the form $ \Bb_{r_2,i} (\prod_{r = r_1+1}^{r_2 - 1} \bSigma_{r,i} \Wb_r^\top) \Bb_{r_1,i} $, where 
\begin{align*}
&\Bb_{r_2,i} \in \{ (\tilde{\bSigma}_{r_2,i} - \bSigma_{r_2,i}) \Wb_{r_2}^\top, \tilde{\bSigma}_{r_2,i}(\tilde{\Wb}_{r_2} - \Wb_{r_2})^\top , m_L^{-1/2}\vb^\top \}, \\
&\Bb_{r_1,i} \in \{ | \tilde{\bSigma}_{r_1,i} - \bSigma_{r_1,i} |, \tilde{\bSigma}_{r_1,i}(\tilde{\Wb}_{r_1} - \Wb_{r_1})^\top ,\ab \}.
\end{align*}
By \ref{ThmResult:randinit_matproductdoublesparsenorm} in Theorem~\ref{thm:randinit}, there exists an absolute constant $C_1$ such that $ \| \Bb_{r_2,i} (\prod_{r = r_1+1}^{r_2 - 1} \bSigma_{r,i} \Wb_r^\top) \Bb_{r_1,i} \|_2 $ with different choices of $\Bb_{r_1,i}$ and $\Bb_{r_2,i}$ have the following bounds:
\begin{enumerate}
    \item If $\Bb_{r_1,i} \in \{ | \tilde{\bSigma}_{r_1,i} - \bSigma_{r_1,i} |, \ab \} $, $\Bb_{r_2,i} \in \{ (\tilde{\bSigma}_{r_2,i} - \bSigma_{r_2,i}) \Wb_{r_2}^\top, m_L^{-1/2} \vb^\top \}$, then  $ \| \Bb_{r_2,i} (\prod_{r = r_1+1}^{r_2 - 1} \bSigma_{r,i} \Wb_r^\top) \Bb_{r_1,i} \|_2 \leq C_1 L \sqrt{s\log(M)/m}$.
    \item If $\Bb_{r_1,i} =  \tilde{\bSigma}_{r_1,i}(\tilde{\Wb}_{r_1} - \Wb_{r_1})^\top $, $\Bb_{r_2,i} =  \tilde{\bSigma}_{r_2,i}(\tilde{\Wb}_{r_2} - \Wb_{r_2})^\top $, then  $ \| \Bb_{r_2,i} (\prod_{r = r_1+1}^{r_2 - 1} \bSigma_{r,i} \Wb_r^\top) \Bb_{r_1,i} \|_2 \leq C_1 L \tau^2$
    \item Otherwise, $ \| \Bb_{r_2,i} (\prod_{r = r_1+1}^{r_2 - 1} \bSigma_{r,i} \Wb_r^\top) \Bb_{r_1,i} \|_2 \leq C_1 L \tau$.
\end{enumerate}
% For any fixed sequence $\Ab_{l_1,i} \in \cA_{l_1,i},\ldots, \Ab_{l_2,i} \in \cA_{l_2,i}$, let 
% \begin{align*}
%     p_1 = \big|\{ r:\Ab_{r,i} = \tilde{\bSigma}_{r,i}(\tilde{\Wb}_r - \Wb_r)^\top \}\big|,p_2 = \big|\{ r:\Ab_{r,i} = (\tilde{\bSigma}_{r,i} - \bSigma_{r,i})\Wb_r^\top \}\big|,p_3 = \big|\{ r:\Ab_{r,i} = \bSigma_{r,i}\Wb_r^\top \}\big|.
% \end{align*}
% Then by the discussion above we see that the bound of $\| \prod_{r=l_1}^{l_2} \Ab_{r,i} \|_2 $ has a term $ (C_1 L \tau)^{p_1} $ granted by the matrices of the form $\tilde{\bSigma}_{r,i}(\tilde{\Wb}_r - \Wb_r)^\top $. In addition, if $p_2 > p_1$, then the bound also has a term $C_1L\sqrt{s\log(M)/m}$ with power at least $p_2 - p_1$. Note that when $p_1 = p_2 = 0$, we still have $\| \prod_{r=l_1}^{l_2} \Ab_{r,i} \|_2 \leq C_2 L $ for some absolute constant $C_2$. Therefore, 
Let $\alpha = \max\{ \tau , \sqrt{s\log(m)/m} \}$
Then similar to the proof of Lemma~\ref{lemma:perturbationmatrixnormbound}, by \ref{ThmResult:randinit_vecmatproductsparsenorm} in Theorem~\ref{thm:randinit} we have
\begin{align}\label{eq:perturbationvecmatrixsparsenormbound1}
    m_L^{-1/2}\vb^\top \Bigg( \prod_{r=l_1}^{l_2}\tilde \bSigma_{r,i}\tilde \Wb_r^\top \Bigg) \ab \leq C_2 L\sqrt{s\log(M)/m} + I_1 + I_2,
\end{align}
where $C_2$ is an absolute constant, and 
\begin{align*}
    & I_1 = \sum_{p_1 = 1}^{L - l + 1} \sum_{p_2 = 0}^{L - l + 1 - p_1} \binom{L-l+1}{p_1} \binom{L-l+1 - p_1}{p_2} (C_1L\tau)^{p_1} (C_1L\sqrt{s\log(M)/m})^{\max\{p_2+1-p_1,0\}},\\
    & I_2 = \sum_{p_2 = 1}^{L-l+1} \binom{L-l+1}{p_2} (C_1 L \sqrt{s\log(M)/m})^{p_2 + 1}.
\end{align*}
For $I_1$, we have
\begin{align*}
    I_1 &\leq C_1 L \tau \cdot \sum_{p_1 = 1}^{L - l + 1} \sum_{p_2 = 0}^{L - l + 1 - p_1} \binom{L-l+1}{p_1} \binom{L-l+1 - p_1}{p_2} (C_1L\tau)^{p_1 - 1} (C_1L\sqrt{s\log(M)/m})^{\max\{p_2+1-p_1,0\}}\\
    &\leq C_1 L \tau \cdot \sum_{p_1 = 1}^{L - l + 1} \sum_{p_2 = 0}^{L - l + 1 - p_1} \binom{L-l+1}{p_1} \binom{L-l+1 - p_1}{p_2} L^{-2(p_1 - 1)} L^{-2\max\{p_2+1-p_1,0\}}\\
    & \leq C_1 L \tau \sum_{p_1 = 1}^{L - l + 1} \sum_{p_2 = 0}^{L - l + 1 - p_1} \binom{L-l+1}{p_1} \binom{L-l+1 - p_1}{p_2}  L^{-2\max\{p_2, p_1 - 1\}}\\
    & \leq C_1 L^2 \tau \sum_{p_1 = 1}^{L - l + 1} \sum_{p_2 = 0}^{L - l + 1 - p_1} \binom{L-l+1}{p_1} \binom{L-l+1 - p_1}{p_2}  L^{-(p_1 + p_2)}\cdot 1^{L-l+1-p_1 - p_2}\\
    &\leq  C_1 L^2 \tau \cdot (1 + 2/L)^L\\
    &\leq C_1 e^2 L^2 \tau.
\end{align*}
For $I_2$, we have
\begin{align*}
    I_2 \leq C_1 L \sqrt{s\log(M)/m}\cdot   \sum_{p_2 = 1}^{L-l+1} \binom{L-l+1}{p_2} L^{-2p_2} \leq C_1 e L \sqrt{s\log(M)/m}.
\end{align*}
Plugging the bounds of $I_1$ and $I_2$ into \eqref{eq:perturbationvecmatrixsparsenormbound1} completes the proof.
\end{proof}

The following lemma is inspired by a similar result given by \cite{allen2018convergence}.
\begin{lemma}\label{lemma:perturbationoutputdifference}
Suppose that $\Wb_1,\ldots,\Wb_L$ are generated via Gaussian initialization, and all results \ref{ThmResult:randinit_normbounds}-\ref{ThmResult:randinit_gradientuniformlowerbound} in Theorem~\ref{thm:randinit} hold. 
Let $\tilde{\Wb} = (\tilde{\Wb}_1,\ldots, \tilde{\Wb}_L)$,  $\hat{\Wb} = (\hat{\Wb}_1 ,\ldots, \hat{\Wb}_L)$ be two collections of weight matrices satisfying $\| \tilde{\Wb}_l - \Wb_l \|_2,\| \hat{\Wb}_l - \Wb_l \|_2 \leq \tau$, $l=1,\ldots,L$. Let $\bSigma_{l,i},\tilde\bSigma_{l,i},\hat\bSigma_{l,i}$ and $\xb_{l,i},\tilde\xb_{l,i},\hat\xb_{l,i}$ be the binary matrices and hidden layer outputs at the $l$-th layer with parameter matrices $\Wb,\tilde{\Wb},\hat{\Wb}$ respectively. 
Let $C_0,C_0'$ be the absolute constants in the bounds of $ \big\| \prod_{r=l_1}^{l_2}\tilde \bSigma_{r,i}\tilde \Wb_r^\top \big\|_2$, $1\leq l_1 < l_2 \leq L$ and $\| \Wb_{l} \|_2$, $l=1,\ldots,L$ given in Lemma~\ref{lemma:perturbationmatrixnormbound} and \ref{ThmResult:randinit_normbounds} in Theorem~\ref{thm:randinit} respectively. 
Then it holds that
%let $\tilde \bSigma_{1,i} = \text{Diag}\big(\ind\big\{\tilde\Wb_1\xb_i>0\big\}\big)$, $\tilde \bSigma_{l,i} = \text{Diag}\big[\ind\big\{\tilde \Wb_l^\top(\prod_{r=1}^{l-1}\tilde \bSigma_{r,i}\tilde \Wb_r^\top)\big\}\xb_i>0\big]$, and $\hat \bSigma_{1,i} = \text{Diag}\big(\ind\big\{\hat\Wb_1\xb_i>0\big\}\big)$, $\hat \bSigma_{l,i} = \text{Diag}\big[\ind\big\{\hat \Wb_l^\top(\prod_{r=1}^{l-1}\hat \bSigma_{r,i}\hat \Wb_r^\top)\big\}\xb_i>0\big]$. Denote $\tilde \xb_{l,i} = \big(\prod_{r=1}^{l}\tilde \bSigma_{r,i}\tilde \Wb_r^\top\big)\xb_i$, $\hat \xb_{l,i} = \big(\prod_{r=1}^{l}\hat \bSigma_{r,i}\hat \Wb_r^\top\big)\xb_i$.
% Then there exists diagonal matrices $\check{\Db}_{l,i}$ such that \begin{itemize}
%     \item $|(\bSigma_{l,i} + \check{\bSigma}_{l,i})_{j,j}| \leq 1$ and $|(\check{\bSigma}_{l,i)}_{j,j}|\leq 1$ for all $j = 1,\ldots,m_j$, $l=1,\ldots, L$ and $i=1,\ldots, n$.
%     \item $\| \check{\bSigma}_{l,i} \|_0 = $
% \end{itemize}
\begin{itemize}
    \item $\| \hat{\xb}_{l,i} - \tilde{\xb}_{l,i} \|_2 \leq C L \cdot \sum_{r=1}^l \| \hat{\Wb}_{r} - \tilde{\Wb}_{r} \|_2$,
    \item $ \| \hat{\bSigma}_{l,i} - \tilde{\bSigma}_{l,i} \|_0 \leq C' L^{4/3}\tau^{2/3} m_l $,
\end{itemize}
for all $l = 1,\ldots, L$ and $i = 1,\ldots,n$, where $C = 2(C_0 \lor 1)$ and $C' = 8 C^{2/3} C_0'^{2/3} $.
\end{lemma}
\begin{proof}
For all $i\in \{1,\ldots,n\}$ and $l\in\{1,\ldots,L\}$, 
we prove the following stronger results: 
\begin{align*}
    &\| \hat{\xb}_{l,i} - \tilde{\xb}_{l,i} \|_2 \leq C L \cdot \sum_{r=1}^l \| \hat{\Wb}_{r} - \tilde{\Wb}_{r} \|_2,~ \| \tilde{\xb}_{l,i} - {\xb}_{l,i} \|_2,\| \hat{\xb}_{l,i} - {\xb}_{l,i} \|_2  \leq C L^2 \tau, \\
    & \| \tilde{\bSigma}_{l,i} - {\bSigma}_{l,i} \|_0,\| \hat{\bSigma}_{l,i} - {\bSigma}_{l,i} \|_0,\| \hat{\bSigma}_{l,i} - \tilde{\bSigma}_{l,i} \|_0 \leq C' L^{4/3} \tau^{2/3} m_r.
\end{align*}
We prove the results above by induction in $l$. 
Suppose that for $r= 1,\ldots, l-1$ it holds that 
\begin{align*}
    &\| \hat{\xb}_{r,i} - \tilde{\xb}_{r,i} \|_2 \leq C L \cdot \sum_{r'=1}^r \| \hat{\Wb}_{r'} - \tilde{\Wb}_{r'} \|_2,~ \| \tilde{\xb}_{r,i} - {\xb}_{r,i} \|_2,\| \hat{\xb}_{r,i} - {\xb}_{r,i} \|_2  \leq C L^2 \tau, \\
    & \| \tilde{\bSigma}_{r,i} - {\bSigma}_{r,i} \|_0,\| \hat{\bSigma}_{r,i} - {\bSigma}_{r,i} \|_0,\| \hat{\bSigma}_{r,i} - \tilde{\bSigma}_{r,i} \|_0 \leq C' L^{4/3} \tau^{2/3} m_r.
\end{align*}
We first prove the bounds for the diagonal matrices. %that $\| \hat{\bSigma}_{l,i} - \tilde{\bSigma}_{l,i} \|_0 \leq C' L^{4/3} \tau^{2/3} m_l$. 
Since $\| \hat{\bSigma}_{l,i} - \tilde{\bSigma}_{l,i} \|_0 \leq \| \hat{\bSigma}_{l,i} - \bSigma_{l,i} \|_0 + \| \tilde{\bSigma}_{l,i} - \bSigma_{l,i} \|_0 $, it suffices to show that $\| \hat{\bSigma}_{l,i} - \bSigma_{l,i} \|_0, \| \tilde{\bSigma}_{l,i} - \bSigma_{l,i} \|_0 \leq  C'/2 L^{4/3} \tau^{2/3} m_l$. We show that $\| \tilde{\bSigma}_{l,i} - \bSigma_{l,i} \|_0 \leq C'/2 L^{4/3} \tau^{2/3} m_l$. Then the same bound for $\| \hat{\bSigma}_{l,i} - \bSigma_{l,i} \|_0$ follows from the exact same proof. To show $\| \tilde{\bSigma}_{l,i} - \bSigma_{l,i} \|_0 \leq C'/2 L^{4/3} \tau^{2/3} m_l$, it suffices to give upper bound for the number of sign changes between vectors $\tilde\Wb_{l}\tilde\xb_{l-1,i}$ and $\Wb_{l}\xb_{l-1,i}$. 
We characterize their difference as follows:
\begin{align*}
    \tilde \Wb_{l}^{\top}\tilde \xb_{l-1,i} - \Wb_l^\top\xb_{l-1,i}  = (\tilde \Wb_l^\top-\Wb_l^\top)\xb_{l-1,i} + \tilde\Wb_l^\top(\tilde \xb_{l-1,i}-\xb_{l-1,i}).  
\end{align*}
Note that we have $\|\tilde \Wb_{l}^\top - \Wb_l^\top\|_2\le \tau$, $\|\xb_{l-1,i}\|_2\le 2$ and $\|\tilde \Wb_l^\top\|_2\le 2C_0'$, and by definition we have $C\geq 2, \overline{c} \geq 1$. Therefore
\begin{align*}
\|\tilde \Wb_l^\top \tilde \xb_{l-1,i}-\Wb_l^\top\xb_{l-1,i}\|_2 &=\big\|(\tilde \Wb_l^\top-\Wb_l^\top)\xb_{l-1,i}+\tilde\Wb_l^\top(\tilde \xb_{l-1,i}-\xb_{l-1,i})\big\|_2\notag\\
&\le \|\tilde \Wb_l^\top - \Wb_l^\top\|_2\|\xb_{l-1,i}\|_2 + \|\tilde \Wb_l^\top\|_2\|\tilde{\xb}_{l-1,i}-\xb_{l-1,i}\|_2\notag\\
&\le 2\tau +  C \overline{c} L^2 \tau \notag\\
&\le 2 C C_0' L^2 \tau.
\end{align*}
Let $\beta >0 $ be a parameter, and 
\begin{align*}
\cS_{l,i}(\beta) = \{j: j\in[m_l], |\la\wb_{l,j},\xb_{l-1,i}\ra|\le \beta \}
\end{align*}
be the set of indices such that the absolute values of the corresponding entries of $\Wb_l^\top \xb_{l-1,i}$ are bounded by $\beta$. Denote
\begin{align*}
    & s_{l,i}^{(1)}(\beta) = | \{ j\in \cS_{l,i}(\beta): (\tilde{\wb}_{l,j}^\top \xb_{l-1,i}) \cdot (\wb_{l,j}^\top \xb_{l-1,i}) < 0 \} |, \\
    & s_{l,i}^{(2)}(\beta) = | \{ j\in \cS^c_{l,i}(\beta): (\tilde{\wb}_{l,j}^\top \xb_{l-1,i}) \cdot (\wb_{l,j}^\top \xb_{l-1,i}) < 0 \} |.
\end{align*}
%$S_{l,i}$ can be further divided into the following two parts based on a positive parameter $\beta$,
Then we have
\begin{align*}
    s_{l,i} = s_{l,i}^{(1)}(\beta)+s_{l,i}^{(2)}(\beta).
\end{align*}
For $s_{l,i}^{(1)}(\beta)$, 
we directly use the upper bound 
$s_{l,i}^{(1)}(\beta)\le |\cS_{l,i}(\beta)| \leq 2m_l^{3/2}\beta$ given by \ref{ThmResult:randinit_activationthreshold} in Theorem~\ref{thm:randinit}. 
We now focus on the upper bound of $s_{l,i}^{(2)}(\beta)$.
It is clear that if the sign of node $j$ changes, we must have
\begin{align*}
\big\vert\la\tilde\wb_{l,j},\tilde\xb_{l-1,i}\ra-\la\wb_{l,j},\xb_{l-1,i}\ra\big\vert \ge \beta.
\end{align*}
This further implies that 
\begin{align*}
s_{l,i}^{(2)}\beta^2\le \|\tilde \Wb_l^\top \tilde \xb_{l-1,i}-\Wb_l^\top\xb_{l-1,i}\|_2^2\le 4 C^2 C_0'^2 L^4 \tau^2.
\end{align*}
Therefore,  we have the following upper bound of $\|\tilde \bSigma_{l,i}-\bSigma_{l,i}\|_0$:
\begin{align*}
\|\tilde \bSigma_{l,i}-\bSigma_{l,i}\|_0\le s_{l,i}^{(1)}(\beta)+s_{l,i}^{(2)}(\beta)\le 2m_l^{3/2}\beta + \frac{4 C^2 C_0'^2 L^4 \tau^2}{\beta^2}.
\end{align*}
Setting $\beta = 2 C^{2/3} C_0'^{2/3} L^{4/3} \tau^{2/3} m_l^{-1/2}$, we obtain
\begin{align*}
\|\tilde \bSigma_{l,i}-\bSigma_{l,i}\|_0\le 8 C^{2/3} C_0'^{2/3} L^{4/3} \tau^{2/3} m_l.
\end{align*}
This completes the proof of
$$
\| \tilde{\bSigma}_{r,i} - {\bSigma}_{r,i} \|_0,\| \hat{\bSigma}_{r,i} - {\bSigma}_{r,i} \|_0,\| \hat{\bSigma}_{r,i} - \tilde{\bSigma}_{r,i} \|_0 \leq C' L^{4/3} \tau^{2/3} m_r.
$$
Now, combining bounds above and the inductive assumption on the bounds of $\|\tilde \bSigma_{r,i}-\hat\bSigma_{r,i}\|_0$, $r=1,\ldots, l-1$, we show that $\| \hat{\xb}_{l,i} - \tilde{\xb}_{l,i} \|_2 \leq C L \cdot \sum_{r=1}^l \| \hat{\Wb}_{r} - \tilde{\Wb}_{r} \|_2$. 
For $l=1,\ldots,L$, we define binary matrix $\check\bSigma_{l,i}$ as follows:
\begin{align*}
    (\check\bSigma_{l,i})_{jj}:= (\hat\bSigma_{l,i} - \tilde\bSigma_{l,i})_{jj} \cdot \frac{\tilde\wb_{l,j}^\top \tilde\xb_{l-1,i}}{ \hat\wb_{l,j}^\top \hat\xb_{l-1,i} - \tilde\wb_{l,j}^\top \tilde\xb_{l-1,i} },~j=1,\ldots,m_l.
\end{align*}
% \begin{align*}
%     (\check\bSigma_{l,i})_{jj} := \left\{ \begin{array}{ll}
%         0, & \text{if }(\hat\bSigma_{l,i} - \tilde\bSigma_{l,i})_{jj} = 0,\\
%         \frac{\tilde\wb_{l,j}^\top \tilde\xb_{l-1,i}}{ \tilde\wb_{l,j}^\top \tilde\xb_{l-1,i} - \hat\wb_{l,j}^\top \hat\xb_{l-1,i} }, & \text{otherwise.}
%     \end{array}
%     \right. 
% \end{align*}
It then follows by definition that $| (\hat\bSigma_{l,i} + \check\bSigma_{l,i})_{jj} |,| (\check\bSigma_{l,i})_{jj} | \leq 1$ for all $j=1,\ldots,m_l$, and
\begin{align*}
    \hat\xb_{l,i} - \tilde\xb_{l,i} 
    &= (\hat\bSigma_{l,i} + \check\bSigma_{l,i}) ( \hat\Wb_l^\top \hat\xb_{l-1,i} - \tilde\Wb_l^\top  \tilde\xb_{l-1,i}) \\
    &= (\hat\bSigma_{l,i} + \check\bSigma_{l,i}) \hat\Wb_l^\top (  \hat\xb_{l-1,i} - \tilde\xb_{l-1,i}) + 
    (\hat\bSigma_{l,i} + \check\bSigma_{l,i}) ( \hat\Wb_l^\top  - \tilde\Wb_l^\top  )\tilde\xb_{l-1,i} \\
    &= \cdots\\
    &= \sum_{r = 1}^l \Bigg[\prod_{t=r+1}^l(\hat\bSigma_{t,i} + \check\bSigma_{t,i}) \hat\Wb_t^\top\Bigg] (\hat\bSigma_{r,i} + \check\bSigma_{r,i})  ( \hat\Wb_r^\top  - \tilde\Wb_r^\top  )\tilde\xb_{r-1,i}
\end{align*}
and therefore by Lemma~\ref{lemma:perturbationmatrixnormbound}, we have
\begin{align*}
    \| \hat\xb_{l,i} - \tilde\xb_{l,i}  \|_2 \leq 2 C_0 L \cdot \sum_{r=1}^l \| \hat{\Wb}_{r} - \tilde{\Wb}_{r} \|_2.
\end{align*}
With the exact same proof, we have
\begin{align*}
    &\| \hat\xb_{l,i} - \xb_{l,i}  \|_2 \leq 2 C_0 L \cdot \sum_{r=1}^l \| \hat{\Wb}_{r} - {\Wb}_{r} \|_2 \leq 2C_0 L^2 \tau,\\
    &\| \tilde\xb_{l,i} - \xb_{l,i}  \|_2 \leq 2 C_0 L \cdot \sum_{r=1}^l \| \tilde{\Wb}_{r} - {\Wb}_{r} \|_2 \leq 2C_0 L^2 \tau.
\end{align*}
This completes the proof.
\end{proof}

\begin{corollary}\label{corollary:upperbound_signs_forall}
Let $\Wb,\tilde{\Wb}$ be the collections of Gaussian initialized and perturbed weight matrices respectively. Define
\begin{align*}
\tilde \cS_L = \{j\in[m_L]: \text{there exists } i\in[n] \mbox{ such that } (\tilde\bSigma_{L,i} - \bSigma_{L,i})_{jj}\neq 0 \}.
\end{align*}
Then under the same assumptions as Lemma~\ref{lemma:perturbationoutputdifference}, it holds that 
\begin{align*}
|\tilde \cS_L|\le Cn L^{4/3}\tau^{2/3} m_L,
\end{align*}
where $C$ is an absolute constant.
\end{corollary}
\begin{proof}[Proof of Corollary~\ref{corollary:upperbound_signs_forall}]
The result directly follows by the bound of $\| \tilde\bSigma_{L,i} - \bSigma_{L,i} \|_0$ given by Lemma~\ref{lemma:perturbationoutputdifference}.
% This result follows by Lemma~\ref{lemma:bound_difference_sigmali}. Denote
% \begin{align*}
%     \tilde{\cS}_{L,i} = \{j:j\in[m_L], (\tilde\bSigma_{L,i} - \bSigma_{L,i})_{jj}\neq 0 \}.
% \end{align*}
% Then we have
% \begin{align*}
%     \tilde \cS_L = \bigcup_{i=1}^{n} \tilde{\cS}_{L,i}.
% \end{align*}
% Therefore
% \begin{align*}
%     |\tilde \cS_L| \leq \sum_{i=1}^n |\tilde{\cS}_{L,i}|\le3n\cdot 4^Lm_L\tau^{2/3}.
% \end{align*}
% This completes the proof.
\end{proof}

%%%%%%%%%%%%%%%%%%%%%%%%%%%%%%%%%%%%%%%%%%%%%%%%%%%%%%%%%%%%%%%%%%%%

\begin{lemma}\label{lemma:grad_norm_nonlinear}
Suppose that $\Wb_1,\ldots,\Wb_L$ are generated via Gaussian initialization, and all results \ref{ThmResult:randinit_normbounds} to \ref{ThmResult:randinit_gradientuniformlowerbound} hold. If $\|\tilde \Wb_l - \Wb_l\|_2\le \tau = O\big(\phi^{3/2}n^{-3}L^{-2}\big)$ for all $l$, then there exists an absolute constant $C$ such that
\begin{align*}
\|\nabla_{\Wb_L}[ L_S(\tilde \Wb)]\|_F^2\ge C\frac{m_L\phi}{n^5}\bigg(\sum_{i=1}^n\ell'(y_i\tilde y_i)\bigg)^2.
\end{align*}
\end{lemma}
\begin{proof}[Proof of Lemma \ref{lemma:grad_norm_nonlinear}]
%It is worthy noting that the major proof of Lemma \ref{lemma:grad_norm_nonlinear} follows from the proof of Lemma A.3 in \citet{li2018learning}, which provides an lower bound for the norm of the partial derivatives with respect to the weight vector of one particular hidden node for one-hidden-layer neural network. 
Define 
\begin{align*}
\gb_j = \frac{1}{n}\sum_{i=1}^n \ell'(y_i\tilde  y_i)y_i\vb_j\sigma'(\la\wb_{L,j},\xb_{L-1,i}\ra)\xb_{L-1,i},
\end{align*}
where $\tilde y_i = f_{\tilde \Wb}(\xb_i)$ denotes the output of the network using the perturbed weight matrices.
By \ref{ThmResult:randinit_gradientuniformlowerbound} in Theorem~\ref{thm:randinit}, the inequality
\begin{align*}
    \|\gb_j\|_2\ge C_1\max_{i}|\ell'(y_i\tilde y_i)|
\end{align*}
holds for at least $C_2m_L\phi/n$ nodes, where $C_1,C_2>0$ are positive absolute constants. Moreover, we
rewrite the gradient $\nabla_{\Wb_{L,j}} L_S(\tilde \Wb)$ as follows:
\begin{align*}
\nabla_{\Wb_{L,j}} L_S(\tilde \Wb)  = \frac{1}{n}\sum_{i=1}^n \ell'(y_i\tilde  y_i)y_i\vb_j\sigma'(\la\tilde\wb_{L,j},\tilde\xb_{L-1,i}\ra)\tilde\xb_{L-1,i}.
\end{align*}
Let $b_{i,j} = \ell'(y_i\tilde y_i)y_i\vb_j$, we have
\begin{align*}
& \|\gb_j\|_2-\|\nabla_{\Wb_{L,j}}L_S(\tilde\Wb)\|_2  \\
&\le \bigg\|\frac{1}{n}\sum_{i=1}^n b_{i,j}\big(\sigma'(\la\tilde\wb_{L,j},\tilde\xb_{L-1,i}\ra)\tilde\xb_{L-1,i}-\sigma'(\la\wb_{L,j},\xb_{L-1,i}\ra)\xb_{L-1,i}\big)\bigg\|_2\\
&\le\bigg\|\frac{1}{n}\sum_{i=1}^nb_{i,j}\Big[\big(\sigma'(\la\tilde\wb_{L,j},\tilde\xb_{L-1,i}\ra)-\sigma'(\la\wb_{L,j},\xb_{L-1,i}\ra\big)\xb_{L-1,i} + \sigma'(\la\tilde\wb_{L,j},\tilde\xb_{L-1,i}\ra)(\tilde\xb_{L-1,i}-\xb_{L-1,i}) \Big]\bigg\|_2.
\end{align*}
According to Lemma \ref{lemma:perturbationoutputdifference}, the number of nodes satisfying $\sigma'(\la\tilde\wb_{L,j},\tilde\xb_{L-1,i}\ra)-\sigma'(\la\wb_{L,j},\xb_{L-1,i}\ra \neq 0$ for at least one $i$ is at most $C_3nL^{4/3}\tau^{2/3}m_L$, where $C_3$ is an absolute constant. For the rest of the nodes in this layer, we have
\begin{align*}
\|\gb_j\|_2-\|\nabla_{\Wb_{L,j}}L_S(\tilde\Wb)\|_2&\le \bigg\|\frac{1}{n}\sum_{i=1}^nb_{i,j} \sigma'(\la\tilde\wb_{L,j},\tilde\xb_{L-1,i}\ra)(\tilde\xb_{L-1,i}-\xb_{L-1,i}) \bigg\|_2\\
&\le \frac{1}{n}\sum_{i=1}^nC_4L^2\tau|b_{i,j}|\\
&\le C_4L^2\tau\max_i|\ell'(y_i\tilde y_i)|,
\end{align*}
where $C_4$ is an absolute constant, the first inequality holds since these nodes satisfy $\sigma'(\la\tilde\wb_{L,j},\tilde\xb_{L-1,i}\ra)-\sigma'(\la\wb_{L,j},\xb_{L-1,i}\ra = 0$ for all $i$, the second inequality follows from Lemma \ref{lemma:perturbationoutputdifference} and triangle inequality. Let
\begin{align*}
\tau \le \bigg(\frac{C_2\phi}{2C_3n^{2}L^{4/3}}\bigg)^{3/2}\wedge \frac{C_1}{2L^2C_4} = O\big(\phi^{3/2}n^{-3}L^{-2}\big).
\end{align*}
Note that we have at least $C_2m_L\phi/n$ nodes satisfying $\|\gb_j\|_2\ge C_1\max_i|\ell'(y_i\tilde y_i)|$, thus there are at least $C_2m_L\phi/n - C_3nL^{4/3}\tau^{2/3}m_L  = C_2m_L\phi/(2n)$ nodes satisfying 
\begin{align*}
\|\nabla_{\Wb_{L,j}}L_S(\tilde\Wb)\|_2\ge C_1\max_i|\ell'(y_i\tilde y_i)| - C_4L^2\tau\max_i|\ell'(y_i\tilde y_i)|\ge  \frac{C_1\max_i|\ell'(y_i\tilde y_i)|}{2}.
\end{align*}
Therefore, 
\begin{align*}
\|\nabla_{\Wb_L} L_S(\tilde\Wb)\|_F^2 &= \sum_{j=1}^{m_L}\|\nabla_{\Wb_{L,j}}L_S(\tilde\Wb)\|_2^2\\
&\ge \frac{C_2\phi m_L}{2n}\bigg(\frac{C_1 \max_{i}|\ell'(y_i\hat y_i^{(k)})y_i\vb_j|}{2}\bigg)^2\\ &\ge \frac{C_2C_1^2\phi m_L}{8n^5}\bigg(\sum_{i=1}^n\ell'(y_i\hat y_i^{(k)})\bigg)^2,
\end{align*}
where the last inequality follows from the fact that $\ell'(\cdot)<0$ and $|y_i\vb_j| = 1$. Let $C = C_2C_1^2/8$, we complete the proof.
\end{proof}

% \begin{lemma}\label{pertub_lemma_2}
% \begin{align*}
%  \bigg\|\Wb_L^\top\bigg(\prod_{r=l+1}^{L-1}\bSigma_{r,i}\Wb^\top_r\bigg)\bigg\|_2\le L^{1/2}(\mbox{or any polynomial function of L})  
%  \end{align*}
% \end{lemma}

% \begin{lemma}
% \begin{align*}
%  \bigg\|\prod_{r=1}^{l}(\bSigma_{r,i}+\tilde \bSigma_{r,i})(\Wb_{r}^\top+\tilde\Wb_{r}^\top)\bigg\|_2\le L^{1/2}(\mbox{or any polynomial function of L})  
%  \end{align*}
% \end{lemma}

% Then, we have the following based on the \eqref{eq:bound_last_layer},
% \begin{align*}
% \|\Wb_L^\top\tilde\xb_{L-1,i}\|_2\le \frac{\sqrt{s\log(M)L^2}}{M^{1/2}}+ L\sum_{l=1}^{L-1} u_{\phi,l},
% \end{align*}
% where $u_{\phi,l}$ denotes an upper bound of $\|\tilde \Wb_l\|_F$ and we use the fact that $\|\xb_i\|_2 = 1$.
% Therefore, following \eqref{eq:difference_last_layer} we have
% \begin{align}\label{eq:upperbound_difference_lastlayer}
% \| (\Wb^\top_{L}+\tilde\Wb^\top_L)(\xb_{L-1,i}+\tilde \xb_{L-1,i}) - \Wb^\top_{L}\xb_{L-1,i}\|_2&\le \frac{\sqrt{s\log(M)L^2}}{M^{1/2}}+ L\sum_{i=1}^{L-1}u_{\phi,l} + 2\sqrt{L}u_{\phi,L} \notag\\
%  &\le \frac{\sqrt{s\log(M)L^2}}{M^{1/2}}+ 3L^2u_\phi ,
% \end{align}
% where the first inequality is by the fact that $\|\xb_{L-1,i}+\tilde \xb_{L-1,i}\|_2\le 2\|\xb_{L-1,i}\|\le 2\sqrt{L}$, and the last inequality follows from the definition $u_\phi = \max_{l\in[L]}u_{\phi,l}$ and fact that $\sqrt{L}\le L^2$.

\begin{lemma}\label{lemma:upper_grad}
If $\|\tilde \Wb_l - \Wb_l\|_2\le \tau$, there exists an absolute constant $C$ such that the following bounds hold on the norm of the partial gradient $\nabla_{\Wb_l} [L_S(\tilde \Wb)]$ and stochastic partial gradient $\tilde \Gb_l$:
\begin{align*}
\big\|\nabla_{\Wb_l}[L_S(\tilde \Wb)]\big\|_2 \le -C\frac{L^2M^{1/2}}{n}\sum_{i=1}^n\ell'(y_i\tilde y_i) \mbox{ and } \big\|\tilde \Gb_l\big\|_2 \le -C\frac{L^2M^{1/2}}{B}\sum_{i\in\cB}\ell'(y_i\tilde y_i),
\end{align*}
where $\tilde y_i = f_{\tilde \Wb}(\xb_i)$, $B = |\cB|$ denotes the minibatch size.
\end{lemma}
\begin{proof}[Proof of Lemma \ref{lemma:upper_grad}]
For the training example $(\xb_i,y_i)$, let $\tilde y_i = f_{\tilde \Wb}(\xb_i)$, the gradient $\nabla_{\Wb_l}\ell(y_i\tilde y_i)$ can be written as follows,
\begin{align*}
\nabla_{\Wb_l}\ell(y_i\tilde y_i) &= \ell'(y_i\tilde y_i)y_i\nabla_{\Wb_l}[f_{\tilde\Wb}(\xb_i)] \\
&=\ell'(y_i\tilde y_i)y_i\tilde\xb_{l-1,i}\vb^\top\Bigg(\prod_{r=l+1}^L\tilde\bSigma_{r,i}\tilde\Wb_l^{\top}\Bigg)\tilde\bSigma_{l,i}.
\end{align*}
Note that by Lemma \ref{lemma:perturbationmatrixnormbound}, there exists an absolute constant $C_1$ such that $\|\prod_{r=l_1}^{l_2}\tilde \bSigma_{r,i}\tilde \Wb_r\|_2\le C_1L$. Hence, we have the following upper bound on $\big\|\nabla_{\Wb_l}\ell(y_i\tilde y_i)\big\|_2$,
\begin{align*}
\big\|\nabla_{\Wb_l}\ell(y_i\tilde y_i)\big\|_2 &\le -\ell'(y_i\tilde y_i) \Bigg\|\prod_{r=1}^{l-1}\tilde \bSigma_{r,i}\tilde\Wb_{r}^\top\xb_i\Bigg\|_2 \Bigg\|\prod_{r=l+1}^L\tilde\bSigma_{r,i}\tilde\Wb_l^{\top}\Bigg\|_2\|\vb\|_2\\
&\le -\ell'(y_i\tilde y_i)C_1^2L^2M^{1/2},
\end{align*}
where the last inequality follows from the fact that $\|\vb\|_2 = m_L^{1/2}\le M^{1/2}$.
Moreover, we have the following for $\nabla_{\Wb_l}[L_S(\tilde \Wb)]$:
\begin{align*}
\big\|\nabla_{\Wb_l}[L_S(\tilde \Wb)]\big\|_2 = \Bigg\|\frac{1}{n}\sum_{i=1}^n \nabla_{\Wb_l}\ell(y_i\tilde y_i)\Bigg\|_2\le \frac{1}{n}\sum_{i=1}^{n}\big\|\nabla_{\Wb_l}\ell(y_i\tilde y_i)\big\|_2\le -\frac{C_1^2L^2M^{1/2}}{n}\sum_{i=1}^n\ell'(y_i\tilde y_i).
\end{align*}
Similarly, regarding the stochastic gradient $\tilde\Gb_l$, we have
\begin{align*}
\big\|\tilde\Gb_{l}\big\|_2 = \Bigg\|\frac{1}{B}\sum_{i\in\cB}\nabla_{\Wb_l}\ell(y_i\tilde y_i)\Bigg\|_2\le \frac{1}{B}\sum_{i\in\cB}\big\|\nabla_{\Wb_l}\ell(y_i\tilde y_i)\big\|_2\le -\frac{C_1^2L^2M^{1/2}}{B}\sum_{i\in\cB}\ell'(y_i\tilde y_i).
\end{align*}
This completes the proof.
\end{proof}

%% file: Appendix3.tex
\section{Proof of Technical Lemmas in Section~\ref{sec:proof of main theory}}

% \subsection{Proof of Lemma \ref{lemma:perturbation_GD}}
% \begin{proof}[Proof of Lemma \ref{lemma:perturbation_GD}]
% We aim to show that within $T/\eta$ steps, the iterates of gradient descent must be in the perturbation region with radius $\tau$. %We prove this lemma by contradiction.
% For any $l = 1,\ldots,L$, 
% assume that $\Wb_l^{(k)}$ is the first iterate which is not in the perturbation region, then for all $t< k$ we have $\| \Wb_l^{(t)} - \Wb_l\|_2\leq \tau $. Therefore we can take advantage of \ref{item:grad_upperbound} in Theorem \ref{thm:perturbation} to derive the upper bound of the gradient,
% \begin{align}\label{eq:zou12345}
% \|\nabla_{\Wb_l}[L_S(\tilde \Wb)]\|_F\le -\frac{4^Lm_L}{n}\sum_{i=1}^n\ell'(y_iy_i^{(k)})\le 4^{L}M^{1/2}\rho,    
% \end{align}
% where the last inequality follows from Assumption \ref{assump:Lipschitz}. Then we have
% \begin{align*}
% \|\Wb_l^{(k)} - \Wb_l^{(0)}\|_2\le \eta\sum_{t=0}^{k-1}\|\nabla_{\Wb_l}[L_S(\Wb^{(t)})]\|_2\le 4^{L}M^{1/2}\rho k\eta.
% \end{align*}
% Note that $k\eta\le T$, thus it follows that $\|\Wb_l^{(k)} - \Wb_l^{(0)}\|_2\le 4^LM^{1/2}\rho T\le \tau$. Similarly we can prove this result for all $l$. Thus the iterate $\Wb^{(k)}$ is in the preset perturbation region, which conflicts with the assumption. This completes the proof.

% \end{proof}

\subsection{Proof of Lemma \ref{lemma:gd_converge}}
\begin{lemma}\label{lemma:support1}
Consider two positive constants $a$ and $b$. For any $p\in[0,1/2)\cup(1/2,1]$, the following inequality holds:
\begin{align*}
\frac{a-b}{b^{2p}}\ge \frac{a^{1-2p}-b^{1-2p}}{1-2p}.
\end{align*}
\end{lemma}
% \begin{proof}[Proof of Lemma \ref{lemma:support1}]
% Note that we aim to prove
% \begin{align}\label{eq:temp0001}
% \frac{a-b}{b^{2p}}\ge \frac{a^{1-2p}-b^{1-2p}}{1-2p},
% \end{align}
% which is equivalent to the following when $p>1/2$
% \begin{align*}
% (2p-1)(b-a)\le  a^{1-2p}b^{2p} - b.
% \end{align*}
% Note that $f(x) = x^{2p-1}$ is a concave function, thus it follows that $b^{2p-1}-a^{2p-1}\ge (2p-1)b^{2p-2}(b-a)$, which yields
% \begin{align*}
% (2p-1)(b-a)\le b-a^{2p-1}b^{2-2p} \le \frac{b^{2p-1}}{a^{2p-1}}\big(b-a^{2p-1}b^{2-2p}\big) = a^{1-2p}b^{2p} - b,
% \end{align*}
% where the second inequality is due to $b\ge a$ and $2p-1>0$. 
% When $p < 1/2$, \eqref{eq:temp0001} is equivalent to 
% \begin{align*}
%  (2p-1)(b-a)\ge  a^{1-2p}b^{2p} - b.  
% \end{align*}
% Based on the convexity of function $f(x) = x^{2p-1}$, we have
% \begin{align*}
% (2p-1)(b-a)\le b^{2p-1}a^{2-2p}-a \ge b-a^{2p-1}b^{2-2p}\ge \frac{b^{2p-1}}{a^{2p-1}}\big(b^{2p-1}a^{2-2p}-a\big) = a^{1-2p}b^{2p} - b ,   
% \end{align*}
% where the second inequality is due to $b\ge a$ and $2p-1<0$. 
% \end{proof}

% \begin{lemma}\label{lemma:boun d_difference_xl}
% For any $l\in\{1,\ldots,L\}$ and $i\in\{1,\ldots,n\}$, if $\tau\le 3.5\times 8^{-L}$, then 
% \begin{align*}
% \|\xb_{l,i}^{(k+1)} - \xb_{l,i}^{(k)} \|_2\le-\frac{2^{5L}M^{1/2}\eta}{7n}\sum_{i=1}^n \ell'(y_i\hat y_i^{(k)}).
% \end{align*}
% \end{lemma}

\begin{lemma}\label{lemma:bound_delta_i}
Let $\Wb_1^{(0)},\dots,\Wb_L^{(0)}$ be generated via Gaussian random initialization. Let $\Wb^{(k)} = \{\Wb_l^{(k)}\}_{l=1,\dots, L}$ be the $k$-th iterate in the gradient descent. Assume all iterates are in the perturbation region centering at $\Wb^{0}$ with radius $\tau$, i.e.,  $\|\Wb_{l}^{(k)} - \Wb_l^{(0)}\|_2\le \tau$ holds for any $k\le K$ and $l \in[L]$, where $K$ is the maximum iteration number. Assume all results in Theorem \ref{thm:perturbation} hold, there exist absolute constants $\overline{C}$, $\underline{C}'$ and $\underline{C}''$ such that the following upper and lower bounds on $\Delta_i^{(k)} = y_i(\hat y_i^{(k+1)}-\hat y_i^{(k)})$ hold:
\begin{align}
&\Delta_{i}^{(k)}\le -\frac{\overline{C}L^4M\eta}{n}\sum_{i=1}^n\ell'(y_i\hat y_i^{(k)});\label{eq:upper_delta} \\
&\Delta_{i}^{(k)}\ge \frac{\underline{C}'L^{17/3}\tau^{1/3}M\eta\cdot\sqrt{\log(M)}+\underline{C}''L^9M^2\eta^2}{n}\sum_{i=1}^n\ell'(y_i\hat y_i^{(k)}) + \sum_{l=1}^L y_i u_{l,i},\label{eq:lower_delta}
\end{align}
where
\begin{align*}
u_{l,i} = -\eta \vb^\top\bigg(\prod_{r=l+1}^L\bSigma_{r,i}^{(k)}\Wb_{r}^{(k)\top}\bigg) \bSigma_{l,i}^{(k)}\big(\nabla_{\Wb_l}[L_S(\Wb^{(k)})]\big)^\top\xb_{l-1,i}^{(k)}.
\end{align*}
\end{lemma}

% \begin{lemma}\label{lemma:lip_GD_sparse}
% With high probability
% \begin{align*}
% \bigg\|\vb^\top\bigg(\prod_{r=l+1}^L\bSigma_{r,i}^{(k)}\Wb_{r}^{(k)\top}\bigg)\tilde \bSigma_{l,i}^{(k+1)}\bigg\|_2\le \sqrt{s\log(M)}.
% \end{align*}
% \end{lemma}

% \CC{Now we focus on studying the convergence performance of gradient descent on training ReLU networks. Note that when the step size is infinitesimal, the gradient descent degenerates to gradient flow, which satisfies the following ordinary differential equation
% \begin{align*}
%  \frac{d L_S(\Wb(t))}{dt}= -\|\nabla_\Wb L_S(\Wb(t))\|^2  = -\sum_{l=1}^L \|\nabla_{\Wb_l} L_S(\Wb(t))\|_F^2.
% \end{align*}
% Since our focus is on the last layer, thus we further relax the above equality as
% \begin{align*}
% \frac{d L_S(\Wb(t))}{dt}\le - \|\nabla_{\Wb_L}L_S(\Wb(t))\|_F^2,
% \end{align*}
% which is sufficient to make the loss function decrease through the lens of the gradient flow. Now we return to the gradient descent, it is clear that the when the loss function is smooth, we can still guarantee that  the loss function keeps decreasing in each iteration of the gradient descent. However, the objective for training ReLU network is highly nonsmooth and nonconvex, which cannot be easily analyzed using the conventional optimization tools. Fortunately, if it can be guaranteed that during the training process, the objective function value decrease dominates the error introduced by the nonsmooth part, the convergence rate of gradient descent can still be established. Thus we have the following lemma.}
\begin{proof}[Proof of Lemma \ref{lemma:gd_converge}]
Note that $\ell(x)$ is $\lambda$-smooth, thus the following holds for any $\Delta$ and $x$,
\begin{align*}
\ell(x+\Delta) \leq \ell(x)+\ell'(x)\Delta + \frac{\lambda}{2}\Delta^2.
\end{align*}
Then we have the following upper bound on $L_S(\Wb^{(k+1)}) - L_S(\Wb^{(k)})$,
\begin{align}\label{eq:loss_decrease_onestep}
L_S(\Wb^{(k+1)}) - L_S(\Wb^{(k)}) &= \frac{1}{n}\sum_{i=1}^n \Big[ \ell\big(y_i\hat y_i^{(k+1)}\big) - \ell\big(y_i\hat y_i^{(k)}\big)\Big]  \notag\\
&\le \frac{1}{n}\sum_{i=1}^n \Big[ \ell'(y_i\hat y_i^{(k)})\Delta_i^{(k)} + \frac{\lambda}{2}(\Delta_i^{(k)})^2\Big],
\end{align}
where $\Delta_i^{(k)} = y_i\big(\hat y_i^{(k+1)} -\hat y_i^{(k)} \big)$.
By Lemma \ref{lemma:bound_delta_i}, we know that there exist constants $C_1$ and $C_2$ such that
\begin{align}\label{eq:one_step_descrease_GD1}
L_S(\Wb^{(k+1)})-L_S(\Wb^{(k)}) &\le \frac{
C_1L^{17/3}\tau^{1/3}M\sqrt{\log(M)}\eta+C_2L^9M^2\eta^2}{n^2}\bigg(\sum_{i=1}^n \ell'(y_i\hat y_i^{(k)})\bigg)^2 \notag\\
&\qquad+\frac{1}{n}\sum_{l=1}^L\sum_{i=1}^n\ell'(y_i\hat y_i^{(k)})y_iu_{l,i}.
\end{align}
Moreover, using the definition of $u_{l,i}$ in Lemma \ref{lemma:bound_delta_i}, we have
\begin{align*}
\frac{1}{n}\sum_{i=1}^n\ell'(y_i\hat y_i^{(k)})y_i \ub_{l,i} &= - \frac{\eta}{n} \sum_{i=1}^n\ell'(y_i\hat y_i^{(k)}) y_i\vb^\top\bigg(\prod_{r=l+1}^L\bSigma_{r,i}^{(k)} \Wb_{r}^{(k)\top}\bigg) \bSigma_{l,i}^{(k)}\big(\nabla_{\Wb_l}[L_S(\Wb^{(k)})]\big)^\top\xb_{l-1,i}^{(k)}\\
&= -\frac{\eta}{n}\sum_{i=1}^n \ell'(y_i\hat y_i^{(k)}) y_i\vb^\top \bigg(\prod_{r=l+1}^L\bSigma_{r,i}^{(k)}\Wb_{r}^{(k)\top}\bigg)\bSigma_{l,i}^{(k)}\\
&\qquad\cdot\bigg(\frac{1}{n}\sum_{j=1}^n \ell'(y_j\hat y_j^{(k)})y_j\xb_{l-1,j}^{(k)} \vb^\top \bigg(\prod_{r=l+1}^L\bSigma_{r,j}^{(k)}\Wb_{r}^{(k)\top}\bigg)\bSigma_{l,i}^{(k)}\bigg)^\top \xb_{l-1,i}^{(k)},\\
&=-\frac{\eta}{n^2} \bigg\|\sum_{i=1}^n\ell'(y_i\hat y_i^{(k)})y_i\xb_{l-1,i}^{(k)}\vb^\top\bigg(\prod_{r=l+1}^L\bSigma_{r,i}^{(k)} \Wb_{r}^{(k)\top}\bigg) \bSigma_{l,i}^{(k)}\bigg\|_F^2\\
&= -\eta \|\nabla_{\Wb_l}[L_S(\Wb^{(k)})]\|_F^2.
\end{align*}
% We are going to take advantage of  \ref{lemma:bound_delta_i} to bound  the R.H.S. of \eqref{eq:loss_decrease_onestep}. Specifically, we use the lower bound \eqref{eq:lower_delta} to bound the first term on the R.H.S. of \eqref{eq:loss_decrease_onestep} since $\ell'(y_i\hat y_i^{(k)})$ is negative, and then use the upper bound \eqref{eq:upper_delta} to bound term $(\Delta_i^{(k)})^2$.
% Regarding the term $\ub_{l,i}$ in \eqref{eq:lower_delta}, we have
% where the last equality follows from the formula of the partial gradient  over $\Wb_l$.
% Let $\ab_{l,i} = \ell'(y_i\hat y_i^{(k)})y_i\big(\vb^\top\big(\prod_{r=l+1}^L\bSigma_{r,i}^{(k)}\Wb_{r}^{(k)\top}\big)\bSigma_{l,i}^{(k)}\big)^\top$, it follows that
% \begin{align*}
% \frac{1}{n}\sum_{i=1}^n\ell'(y_i\hat y_i^{(k)})y_i\vb^\top \ub_{l,i} &= -\frac{\eta}{n^2}\sum_{i=1}^n\ab_{l,i}^\top\bigg(\sum_{j=1}^n\xb_{l-1,j}^{(k)}\ab_{l,j}^\top\bigg)^\top\xb_{l-1,i}^{(k)}\\
% &=-\frac{\eta}{n^2}\sum_{i=1}^n\sum_{j=1}^n \ab_{l,i}^\top\ab_{l,j}\xb_{l-1,j}^{(k)\top}\xb_{l-1,i}^{(k)}\\
% &=-\frac{\eta}{n^2}\text{Tr}\bigg[\bigg(\sum_{i=1}^n \ab_{l,i}\xb_{l-1,i}^{(k)\top}\bigg)\bigg(\sum_{i=1}^n \ab_{l,i}\xb_{l-1,i}^{(k)\top}\bigg)^\top\bigg]\\
% &=-\frac{\eta}{n^2}\bigg\|\sum_{i=1}^n \ab_{l,i}\xb_{l-1,i}^{(k)\top}\bigg\|_F^2\\
% &=-\eta \|\nabla_{\Wb_l}L_S(\Wb^{(k)})\|_F^2.
% \end{align*}
Then, plugging the above result into \eqref{eq:one_step_descrease_GD1} gives
\begin{align}\label{eq:decrease_onestep2}
&\frac{1}{n}\sum_{i=1}^n\ell'(y_i\hat y_i^{(k)})\Delta_i^{(k)}\notag\\
&\le -\eta\sum_{l=1}^L\|\nabla_{\Wb_l}L_S(\Wb^{(k)})\|_F^2 + \frac{
C_1L^{17/3}\tau^{1/3}M\sqrt{\log(M)}\eta+C_2L^9M^2\eta^2}{n^2}\bigg(\sum_{i=1}^n \ell'(y_i\hat y_i^{(k)})\bigg)^2.
\end{align}
Note that by Lemma \ref{lemma:grad_norm_nonlinear}, we know that there exists a constant $c_0$ such that
\begin{align*}
\|\nabla_{\Wb_L}L_S(\Wb^{(k)})\|_F^2\ge \frac{c_0 m\phi}{n^5}\bigg(\sum_{i=1}^n\ell'(y_i\hat y_i^{(k)})\bigg)^2.
\end{align*}
We only take advantage of the gradient of the weight matrix in the last hidden layer to make loss function decrease. Thus, substituting the above inequality into \eqref{eq:decrease_onestep2}, we obtain
\begin{align*}
&L_S(\Wb^{(k+1)})-L_S(\Wb^{(k)})\\
&\le \bigg(-\eta\frac{c_0m\phi/n^3-C_1L^{17/3}\tau^{1/3}M\sqrt{\log(M)}}{n^2}+\eta^2\frac{C_2L^9M^2}{n^2}\bigg)\bigg(\sum_{i=1}^n\ell'(y_i\hat y_i^{(k)})\bigg)^2.
\end{align*}
Then we set
\begin{align*}
\tau \le \bigg(\frac{c_0m\phi}{4C_1L^{17/3}n^{3}M\sqrt{\log(M)}}\bigg)^3 = \tilde O(n^{-9}L^{-17}\phi^9),
\end{align*}
and 
\begin{align*}
\eta\le \frac{c_0m\phi }{4C_2n^3L^9M^2} = O(n^{-3}L^{-9}M^{-1}\phi),
\end{align*}
which leads to,
\begin{align}\label{eq:decrease_onestep3}
L_S(\Wb^{(k+1)})-L_S(\Wb^{(k)})\le -\eta\frac{c_0m\phi}{2n^5} \bigg(\sum_{i=1}^n\ell'(y_i\hat y_i^{(k)})\bigg)^2.   
\end{align}
According to Assumption \ref{assump:derivative_loss}, we know that
\begin{align*}
-\sum_{i=1}^n\ell'(y_i\hat y_i^{(k)})\ge \min\bigg\{\alpha_0, \sum_{i=1}^n \alpha_1\ell^p(y_i\hat y_i^{(k)})\bigg\}\ge \min\big\{\alpha_0, n^pL_S^p(\Wb^{(k)})\big\}.
\end{align*}
Note that $\min\{a,b\}\ge 1/(1/a+1/b)$, we have the following by plugging the above inequality into \eqref{eq:decrease_onestep3} 
\begin{align*}
L_S(\Wb^{(k+1)}) - L_S(\Wb^{(k)})&\le -\eta\min\bigg\{\frac{c_0m\phi\alpha_0^2}{2n^5},\frac{c_0m\phi\alpha_1^2}{2n^{5-2p}}L_S^{2p}(\Wb^{(k)})\bigg\} \\
&\le -\eta\bigg(\frac{2n^5}{c_0m\phi\alpha_0^2}+\frac{2n^{5-2p}}{c_0m\phi\alpha_1^2L_S^{2p}(\Wb^{(k)})}\bigg)^{-1}.
\end{align*}
Rearranging terms gives
\begin{align}\label{eq:GD_temp11111}
\frac{2n^5}{c_0m\phi\alpha_0^2}\big(L_S(\Wb^{(k+1)}) - L_S(\Wb^{(k)})\big)+\frac{2n^{5-2p}\big(L_S(\Wb^{(k+1)}) -L_S(\Wb^{(k)})\big)}{c_0m\phi\alpha_1^2L_S^{2p}(\Wb^{(k)})}\le -\eta.
\end{align}
When $p\neq 1/2$, we have the following by Lemma \ref{lemma:support1}
\begin{align*}
\frac{2n^5}{c_0m\phi\alpha_0^2}\big(L_S(\Wb^{(k+1)}) - L_S(\Wb^{(k)})\big)+\frac{2n^{5-2p}\big(L_S^{1-2p}(\Wb^{(k+1)}) -L_S^{1-2p}(\Wb^{(k)})\big)}{c_0m\phi\alpha_1^2(1-2p)}\le -\eta.
\end{align*}
Then taking telescope sum over $k$ and rearranging terms  give
\begin{align}\label{eq:pneq1/2}
k\eta&\le \frac{2n^5}{c_0m\phi\alpha_0^2}\big(L_S(\Wb^{(0)})-L_S(\Wb^{(k)})\big)+\frac{2n^{5-2p}\big(L_S^{1-2p}(\Wb^{(0)}) -L_S^{1-2p}(\Wb^{(k)})\big)}{(1-2p)c_0m\phi\alpha_1^2}\notag\\
&\le\frac{2n^5}{c_0m\phi\alpha_0^2}L_S(\Wb^{(0)})+\frac{2n^{5-2p}\big(L_S^{1-2p}(\Wb^{(0)}) -L_S^{1-2p}(\Wb^{(k)})\big)}{(1-2p)c_0m\phi\alpha_1^2}.
\end{align}
When $p=1/2$, we have 
\begin{align*}
\frac{L_S(\Wb^{(k+1)})-L_S(\Wb^{(k)})}{L_S(\Wb^{(k)})}\ge \log\big(L_S(\Wb^{(k+1)})\big) - \log\big(L_S(\Wb^{(k)})\big),
\end{align*}
which implies the following from
\eqref{eq:GD_temp11111},
\begin{align*}
\frac{2n^5}{c_0m\phi\alpha_0^2}\big(L_S(\Wb^{(k+1)}) - L_S(\Wb^{(k)})\big)+\frac{2n^{5-2p}\big(\log\big(L_S(\Wb^{(k+1)}\big)) -\log\big(L_S(\Wb^{(k)})\big)\big)}{c_0m\phi\alpha_1^2}\le -\eta.
\end{align*}
Similarly, taking telescope sum and rearranging terms give
\begin{align}\label{eq:peq1/2}
k\eta&\le\frac{2n^5}{c_0m\phi\alpha_0^2}L_S(\Wb^{(0)})+\frac{2n^{5-2p}\big(\log\big(L_S(\Wb^{(0)})\big) -\log\big(L_S(\Wb^{(k)})\big)\big)}{c_0m\phi\alpha_1^2}.
\end{align}
% The following proofs are established based on two regimes according to Assumption \ref{assump:derivative_loss}: $\alpha_0 = 0$ and $\alpha_0 > 0$.
% For $\alpha_0 = 0$, we have
% \begin{align*}
% \sum_{i=1}^n\ell'(y_i\hat y_i^{(k)})\le\sum_{i=1}^n\alpha_1\ell^p(y_i\hat y_i^{(k)})\le \alpha_1 n^p L_S^p(\Wb^{(k)}) .  
% \end{align*}
% Therefore, \eqref{eq:decrease_onestep3} further implies
% \begin{align}\label{eq:decrease_onestep4}
% L_S(\Wb^{(k+1)})-L_S(\Wb^{(k)})\le -\eta\frac{c_0m\phi}{2n^5} \bigg(\sum_{i=1}^n\ell'(y_i\hat y_i^{(k)})\bigg)^2\le -\eta\frac{c_0m\phi\alpha_1^2}{2n^{5-2p}}L_S^{2p}(\Wb^{(k)})  . 
% \end{align}
% Dividing by $L_S^{2p}(\Wb^{(k)})$ on both sides and applying Lemma \ref{lemma:support1}, we have the following when $p\neq 1/2$,
% \begin{align*}
% \frac{L_S^{1-2p}(\Wb^{(k+1)})-L_S^{1-2p}(\Wb^{(k)})}{1-2p}\le -\eta\frac{c_0m\phi\alpha_1^2}{2n^{5-2p}}.
% \end{align*}
% Telescoping over $k$ yields
% \begin{align}\label{eq:gd_x1p1}
% \frac{c_0m\phi\alpha_1^2}{2n^{7-2p}}k\eta\le\frac{L_S^{1-2p}(\Wb^{(k)})-L_S^{1-2p}(\Wb^{(0)})}{2p-1}.    
% \end{align}
% When $p=1/2$, we have the following from \eqref{eq:decrease_onestep4},
% \begin{align*}
% \frac{L_S(\Wb^{(k+1)})}{L_S(\Wb^{(k)})}\le 1-\frac{c_0m\phi\alpha_1^2\eta}{2n^4},   
% \end{align*}
% which further implies
% \begin{align}\label{eq:gd_x1p2}
% \frac{L_S(\Wb^{(k)})}{L_S(\Wb^{(0)})}\le \bigg(1-\frac{c_0m\phi\alpha_1^2\eta}{2n^4}\bigg)^{k}\le \exp\bigg(-\frac{c_0m\phi\alpha_1^2k\eta}{2n^4}\bigg).  
% \end{align}
Now we need to guarantee that after $K$ gradient descent steps the loss function $L_S(\Wb^{(K)})$ is smaller than the target accuracy $\epsilon$. By \ref{ThmResult:randinit_outputbound}, we know that the output $y_i^{(0)}$ is in the order of $\tilde O(1)$, which implies that the training loss $L_S(\Wb^{(0)}) = \tilde O(1)$ due to the smoothness assumption. Therefore, by \eqref{eq:pneq1/2} and \eqref{eq:peq1/2}, we require the following in terms of $K\eta$,
\begin{align}\label{eq:convergence_gd_keta}
K\eta=\left\{\begin{array}{ll}
\tilde \Omega\big(n^5/(m\phi\alpha_0^2)+n^{5-2p}/(m\phi)\big)& 0\le p <\frac{1}{2}\\
\tilde \Omega\big(n^5/(m\phi\alpha_0^2)\big)+\tilde \Omega(n^4/(m\phi)\big)\cdot \Omega\big(\log(1/\epsilon))& p=\frac{1}{2}\\
\tilde \Omega\big(n^5/(m\phi\alpha_0^2)+n^{5-2p}\epsilon^{1-2p}/(m\phi)\big)& \frac{1}{2}<p \le 1.
\end{array} \right.  
\end{align}
Here we preserve the parameter $\alpha_0$ since it might take on value $\infty$. When $\alpha_0 = \infty$, we can remove the term $n^5/(m\phi\alpha_0^2)$ in \eqref{eq:convergence_gd_keta} since it becomes zero. When $\alpha_0 <\infty$, we treat it as a constant of order $O(1)$. This completes the proof.

\end{proof}

\section{Proof of Lemma \ref{lemma:upperbound_tau} }
\begin{proof}[Proof of Lemma \ref{lemma:upperbound_tau}]
We prove this lemma by induction. Assume the argument holds for any $t< k$, thus according to \eqref{eq:decrease_onestep3}, we have
\begin{align}\label{eq:decrease_onestep_gd_temp1}
L_S(\Wb^{(t+1)}) -L_S(\Wb^{(t)})\le -\frac{c_0m\phi \eta}{2n^{5}}\bigg(\sum_{i=1}^n\ell'(y_i\hat y_i^{(t)})\bigg)^2,
\end{align}
for any $t< k$, where $c_0$ is an absolute constant. Moreover, by Lemma \ref{lemma:upper_grad}, we have
\begin{align*}
\|\Wb_l^{(t+1)}-\Wb_l^{(t)}\|_2&\le \eta \|\nabla_{\Wb_l}L_S(\Wb^{(t)})\|_2\\
&\le -\frac{c_1L^2M^{1/2}\eta}{n}\sum_{i=1}^{n}\ell'(y_i\hat y_i^{(t)}),
\end{align*}
where $c_1$ is an absolute constant.
Therefore, we have
\begin{align*}
\|\Wb_l^{(k)} - \Wb_l^{(0)}\|_2
&\le \eta\sum_{t=0}^{k-1} \big\|\nabla_{\Wb_l}L_S(\Wb^{(t)})\big\|_2\\
&\le \eta\sqrt{k\sum_{t=0}^{k-1}\big\|\nabla_{\Wb_l}L_S(\Wb^{(t)})\big\|_2^2}\notag\\
&\le \eta\sqrt{kc_1^2L^4M\sum_{t=0}^{k-1}\bigg(\frac{1}{n}\sum_{i=1}^n\ell'(y_i\hat y_i^{(t)})\bigg)^2}.
\end{align*}
According to \eqref{eq:decrease_onestep_gd_temp1}, it follows that
\begin{align*}
\|\Wb^{(k)}_l - \Wb_l^{(0)}\|_2&\le  \sqrt{\frac{2k\eta c_1^2L^4Mn^3}{c_0m\phi}\sum_{t=0}^{k-1}\big[L_S(\Wb^{(k)})-L_S(\Wb^{(k+1)})\big]}\\
&\le \sqrt{\frac{2k\eta c_1^2L^4Mn^3}{c_0m\phi}L_S(\Wb^{(0)})}.
\end{align*}
Thus if 
\begin{align*}
k\eta \le \frac{\tau^2 c_0m\phi}{2c_1^2L^4Mn^3L_S(\Wb^{(0)})} = O\big(\tau^2\phi L^{-4}n^{-3}\big),
\end{align*}
it follows that $\|\Wb^{(k)}-\Wb_l^{(0)}\|_2\le \tau$. This completes the proof.
\end{proof}

% \subsection{Proof of Lemma \ref{lemma:perturbation_SGD}}
% \begin{proof}[Proof of Lemma \ref{lemma:perturbation_SGD}]
% The proof is essentially similar to that of Lemma \ref{lemma:perturbation_GD}. It suffices to show that 
% \begin{align*}
% \|\tilde \Gb_l\|_F\le -\frac{4^LM}{B}\sum_{i\in\cB}\ell'(y_i\hat y_i^{(k)})\le 4^LM^{1/2}\rho,
% \end{align*}
% which is the same as the upper bound in \eqref{eq:zou12345}. Thus by following the identical proof technique applied for proving Lemma \ref{lemma:perturbation_SGD}, we are able to complete the proof.

% \end{proof}

\subsection{Proof of Lemma \ref{lemma:sgd_converge}}
The following three lemmas are necessary to prove Lemma \ref{lemma:sgd_converge}.
\begin{lemma}\label{lemma:bound_delta_i_sgd}
Let $\Wb_1^{0},\dots,\Wb_L^{(0)}$ be generated via Gaussian random initialization. Let $\Wb^{(k)} = \{\Wb_l^{(k)}\}_{l=1,\dots, L}$ be the $k$-th iterate in the stochastic gradient descent. Assume there exist two constants $\tau, s>0$ satisfying $\tau\le \kappa L^{-3}$ and $\sqrt{s\log(M)/m}\le \kappa L^{-3}$ for some small enough absolute constant $\kappa$ such that $\|\Wb_{l}^{(k)} - \Wb_l^{(0)}\|_2\le \tau$ and $\|\bSigma_{l,i}^{(k)}-\bSigma_{l,i}^{(0)}\|_0\le s$ hold for any $k\le K$ and $l \in[L]$, where $K$ is the maximum iteration number. If $M = \tilde \Omega (1)$, there exist absolute constants $\overline{C}$, $\underline{C}'$ and $\underline{C}''$ such that the following upper and lower bounds on $\tilde \Delta_i^{(k)} = y_i(\hat y_i^{(k+1)}-\hat y_i^{(k)})$ hold:
\begin{align}
&\tilde \Delta_{i}^{(k)}\le -\frac{\overline{C}L^4M\eta}{B}\sum_{i\in\cB^{(k)}}\ell'(y_i\hat y_i^{(k)});\label{eq:upper_delta_sgd} \\
&\tilde\Delta_{i}^{(k)}\ge \frac{\underline{C}'L^{17/3}\tau^{1/3}M\eta\cdot\sqrt{\log(M)}+\underline{C}''L^9M^2\eta^2}{B}\sum_{i\in\cB^{(k)}}\ell'(y_i\hat y_i^{(k)}) + \sum_{l=1}^L y_i u_{l,i}\label{eq:lower_delta_sgd},
\end{align}
where $\cB^{(k)}$ with $|\cB^{(k)}| = B$ denotes the set of minibatch for stochastic gradient calculation in the $k$-th iteration, and
\begin{align*}
u_{l,i} = -\eta \vb^\top\bigg(\prod_{r=l+1}^L\bSigma_{r,i}^{(k)}\Wb_{r}^{(k)\top}\bigg) \bSigma_{l,i}^{(k)}\Gb_l^{(k)\top}\xb_{l-1,i}^{(k)}.
\end{align*}
where $\Gb_{l}^{(k)}$ denotes the stochastic partial gradient with respect to $\Wb_l^{(k)}$.
\end{lemma}

\begin{lemma}\label{lemma:support2}
Regarding $n$ random variables $u_1,\dots,u_n$ satisfying $\sum_{i=1}^nu_i = 0$. Let $\cB\in[n]$ denote a subset of $[n]$ and $|\cB| = B\le n$, the following holds,
\begin{align*}
\EE\bigg[\bigg(\frac{1}{B}\sum_{i\in\cB}u_i\bigg)^2\bigg]\le \frac{1}{B}\EE\big[u_i^2\big].
\end{align*}
\end{lemma}

\begin{lemma}\label{lemma:azuma}
Under Assumptions \ref{assump:derivative_loss_upper}, with probability at least $1-\delta$, there exists an absolute constant $C$ such that the output of the stochastic gradient descent, i.e., $\Wb^{(k)}$, satisfies
\begin{align*}
\frac{L_S^{1-2p}(\Wb^{(k)})}{1-2p}\le \frac{\EE\big[L_S^{1-2p}(\Wb^{(k)})\big]}{1-2p}+\sqrt{2k\log(1/\delta)}\frac{CnL^4M\eta}{B}.
\end{align*}
when $p \neq 1/2$ and 
\begin{align*}
\log(L_S(\Wb^{(k)})) \le  \EE\big[\log(L_S(\Wb^{(k)}))\big]+\sqrt{2k\log(1/\delta)}\frac{CnL^4M\eta}{B}. 
\end{align*}
\end{lemma}

\begin{proof}[Proof of Lemma \ref{lemma:sgd_converge}]
By Assumption \ref{assump:smooth}, $L_S(\Wb^{(k+1)}) - L_S(\Wb^{(k)})$ can be upper bounded as follows,
\begin{align}\label{eq:loss_decrease_onestep_sgd}
L_S(\Wb^{(k+1)}) - L_S(\Wb^{(k)}) &= \frac{1}{n}\sum_{i=1}^n \Big[ \ell\big(y_i\hat y_i^{(k+1)}\big) - \ell\big(y_i\hat y_i^{(k)}\big)\Big]  \notag\\
&\le \frac{1}{n}\sum_{i=1}^n \Big[ \ell'(y_i\hat y_i^{(k)})\tilde \Delta_i^{(k)} + \frac{\lambda}{2}(\tilde \Delta_i^{(k)})^2\Big],
\end{align}
where $\tilde \Delta_i^{(k)} = y_i\big(\hat y_i^{(k+1)} -\hat y_i^{(k)} \big)$.
Then taking expectation conditioning on $\Wb^{(k)}$ gives
\begin{align*}
\EE\big[L_S(\Wb^{(k+1)})|\Wb^{(k)}\big] - L_S(\Wb^{(k)}) \le \frac{1}{n}\sum_{i=1}^n \Big[ \ell'(y_i\hat y_i^{(k)})\EE\big[\tilde \Delta_i^{(k)}|\Wb^{(k)}\big] + \frac{\lambda}{2}\EE\big[(\tilde \Delta_i^{(k)})^2|\Wb^{(k)}\big]\Big].
\end{align*}
% Note that given the iterate $\Wb^{(k)}$, the randomness of $\tilde \Delta_{i}^{(k)}$ only comes from the choice of samples for computing the stochastic gradient. Therefore, we have
% \begin{align*}
% \EE\big[\tilde \Delta_i^{(k)}|\Wb^{(k)}\big]  = y_i\big(\tilde y_i^{(k)} - y_i^{(k)}\big)\triangleq \Delta_i^{(k)},
% \end{align*}
% where $\tilde y_i^{(k)}$ denotes the output of neural network after one-step gradient descent on $\Wb^{(k)}$. Thus, similar to \eqref{eq:upperbound_delta_1_gd}, it follows that
Note that  the lower bound of $\EE[\tilde \Delta_i^{(k)}|\Wb^{(k)}]$ provided in Lemma \ref{lemma:bound_delta_i_sgd} is identical to \eqref{eq:lowerbound_delta} in Lemma \ref{lemma:bound_delta_i}. Therefore, there exist absolute constants $C_1$ and $C_2$ such that
\begin{align}\label{eq:upperbound_delta_1_sgd}
&\frac{1}{n}\sum_{i=1}^n\ell'(y_i\hat y_i^{(k)})\EE\big[\tilde\Delta_i^{(k)}|\Wb^{(k)}\big]\notag\\
&\le \frac{1}{n}\sum_{l=1}^L\sum_{i=1}^n\ell'(y_i\hat y_i^{(k)})y_iu_{l,i}+\frac{
C_1L^{17/3}\tau^{1/3}M\sqrt{\log(M)}\eta+C_2L^9M^2\eta^2}{n^2}\bigg(\sum_{i=1}^n \ell'(y_i\hat y_i^{(k)})\bigg)^2\notag\\
&\le -\eta\|\nabla_{\Wb_L}[L_S(\Wb^{(k)})]\|_F^2+\frac{
C_1L^{17/3}\tau^{1/3}M\sqrt{\log(M)}\eta+C_2L^9M^2\eta^2}{n^2}\bigg(\sum_{i=1}^n \ell'(y_i\hat y_i^{(k)})\bigg)^2.
\end{align}
In terms of $\EE\big[(\tilde \Delta_i^{(k)})^2|\Wb^{(k)}\big]$, we bound it as follows based on \eqref{eq:upper_delta_sgd} in Lemma \ref{lemma:bound_delta_i_sgd}
\begin{align}\label{eq:zou123}
\EE\big[(\tilde \Delta_i^{(k)})^2|\Wb^{(k)}\big] \le C_2L^{8}M^2\eta^2\EE\bigg[\bigg(\frac{1}{B}\sum_{i\in\cB^{(k)}}\ell'(y_i\hat y_i^{(k)})\bigg)^2|\Wb^{(k)}\bigg],
\end{align}
where $C_2$ is an absolute constant.
By Lemma \ref{lemma:support2}, we have
\begin{align*}
\EE\bigg[\bigg(\frac{1}{B}\sum_{i\in\cB^{(k)}}\ell'(y_i\hat y_i^{(k)})\bigg)|\Wb^{(k)}\bigg] &= \EE\bigg[\bigg(\frac{1}{B}\sum_{i\in\cB^{(k)}}\ell'(y_i\hat y_i^{(k)}) - \frac{1}{n}\sum_{i=1}^n\ell'(y_i\hat y_i^{(k)})\bigg)^2|\Wb^{(k)}\bigg] \\
&\qquad+ \bigg(\frac{1}{n}\sum_{i=1}^n\ell'(y_i\hat y_i^{(k)})\bigg)^2\\
&\le \frac{1}{nB}\sum_{i=1}^n\big[\ell'(y_i\hat y_i^{(k)})\big]^2 +  \bigg(\frac{1}{n}\sum_{i=1}^n\ell'(y_i\hat y_i^{(k)})\bigg)^2\\
&\le \frac{2}{nB}\bigg(\sum_{i=1}^n\ell'(y_i\hat y_i^{(k)})\bigg)^2,
\end{align*}
where the last inequality holds since $(\sum_{i=1}^nz_i)^2\ge \sum_{i=1}^nz_i^2$ for any $z_1,\dots, z_n\ge 0$.
Plugging this into \eqref{eq:zou123}, we obtain
\begin{align}\label{eq:bound_delta_square_sgd}
\EE\big[(\tilde \Delta_i^{(k)})^2|\Wb^{(k)}\big] \le \frac{2C_2L^8M^2\eta^2}{nB}\bigg(\sum_{i=1}^n\ell'(y_i\hat y_i^{(k)})\bigg)^2.
\end{align}
Combining \eqref{eq:bound_delta_square_sgd} and \eqref{eq:upperbound_delta_1_sgd} and then plugging into \eqref{eq:loss_decrease_onestep_sgd}, we have
\begin{align}\label{eq:decrease_onestep2_sgd}
&\EE\big[L_S(\Wb^{(k+1)})|\Wb^{(k)}\big]-L_S(\Wb^{(k)})\notag\\
% &\le -\eta\sum_{l=1}^L\|\nabla_{\Wb_l}L_S(\Wb^{(k)})\|_F^2+\bigg(\frac{2^{9L+3}LC_0\tau^{1/3}M\sqrt{\log(M)}\eta}{n^2}+\frac{2^{10L}M^{2}\eta^2(\rho+\lambda)}{7nB}\bigg) \bigg(\sum_{i=1}^n\ell'(y_i\hat y_i^{(k)})\bigg)^2\notag\\
&\le -\eta\|\nabla_{\Wb_L}L_S(\Wb^{(k)})\|_F^2+\frac{
C_1L^{17/3}\tau^{1/3}M\sqrt{\log(M)}\eta}{n^2}+\frac{C_3L^9M^2\eta^2}{nB} \bigg(\sum_{i=1}^n\ell'(y_i\hat y_i^{(k)})\bigg)^2,
\end{align}
where $C_3$ is an absolute constant, and we use the fact that $B\le n$.
By Lemma \ref{lemma:grad_norm_nonlinear}, there exists a positive constant $c_0$ such that
\begin{align*}
\|\nabla_{\Wb_L}L_S(\Wb^{(k)})\|_F^2\ge \frac{c_0 m_L\phi}{n^5}\bigg(\sum_{i=1}^n\ell'(y_i\hat y_i^{(k)})\bigg)^2.
\end{align*}
Substituting this into \eqref{eq:decrease_onestep2_sgd}, we obtain
\begin{align*}
&\EE\big[L_S(\Wb^{(k+1)})|\Wb^{(k)}\big]-L_S(\Wb^{(k)})\\
&\le \bigg[-\eta\bigg(\frac{c_0m\phi}{n^5}-\frac{C_1L^{17/3}\tau^{1/3}M\sqrt{\log(M)}}{n^2}\bigg)+\eta^2\frac{C_3L^9M^2}{nB}\bigg]\bigg(\sum_{i=1}^n\ell'(y_i\hat y_i^{(k)})\bigg)^2.
\end{align*}
Then we set
\begin{align*}
\tau \le \bigg(\frac{c_0m\phi}{4C_1L^{17/3}n^{3}M\sqrt{\log(M)}}\bigg)^3 = \tilde O(n^{-9}L^{-17}\phi^3),
\end{align*}
and 
\begin{align*}
\eta\le \frac{c_0m\phi B}{4C_3n^4L^9M^2} = O(Bn^{-4}L^{-9}M^{-1}\phi).
\end{align*}
Therefore, after one-step stochastic gradient descent, we have
\begin{align}\label{eq:decrease_onestep3_sgd}
\EE\big[L_S(\Wb^{(k+1)})|\Wb^{(k)}\big]-L_S(\Wb^{(k)})\le-\eta\frac{c_0m\phi}{2n^5} \bigg(\sum_{i=1}^n\ell'(y_i\hat y_i^{(k)})\bigg)^2.   
\end{align}
Similar to the analysis for gradient descent, we have 
\begin{align*}
-\sum_{i=1}^n\ell'(y_i\hat y_i^{(k)})\ge \min\big\{\alpha_0, n^pL_S^p(\Wb^{(k)})\big\},  
\end{align*}
which yields
\begin{align*}
\EE\big[L_S(\Wb^{(k+1)})|\Wb^{(k)}\big]-L_S(\Wb^{(k)})\le-\eta\bigg(\frac{2n^5}{c_0m\phi\alpha_0^2}+\frac{2n^{5-2p}}{c_0m\phi\alpha_1^2L_S^{2p}(\Wb^{(k)})}\bigg)^{-1},
\end{align*}
where we use the fact that $\min\{a,b\}\ge 1/(1/a+1/b)$.
When $p\neq 1/2$, rearranging terms and applying Lemma \ref{lemma:support1} give
\begin{align*}
\frac{2n^5}{c_0m\phi\alpha_0^2}\big(\EE\big[L_S(\Wb^{(k+1)})|\Wb^{(k)}\big] - L_S(\Wb^{(k)})\big)+\frac{2n^{5-2p}\big(\EE\big[L_S(\Wb^{(k+1)})|\Wb^{(k)}\big]^{1-2p} -L_S^{1-2p}(\Wb^{(k)})\big)}{(1-2p)c_0m\phi\alpha_1^2}\le -\eta.
\end{align*}
Using Jensen's inequality $\EE[x]^{1-2p}/(1-2p)\ge \EE[x^{1-2p}]/(1-2p)$ and further taking expectation over $\Wb^{(k)}$, we have
\begin{align*}
\frac{2n^5}{c_0m\phi\alpha_0^2}\big(\EE\big[L_S(\Wb^{(k+1)})\big] - \EE\big[L_S(\Wb^{(k)})\big]\big)+\frac{2n^{5-2p}\big(\EE\big[L_S^{1-2p}(\Wb^{(k+1)})\big] -\EE\big[L_S^{1-2p}(\Wb^{(k)})\big]\big)}{(1-2p)c_0m\phi\alpha_1^2}\le -\eta. 
\end{align*}
Then we take telescope sum over $k$ on both sides, and obtain
\begin{align}\label{eq:pneq1/2_sgd}
k\eta &\le \frac{2n^5}{c_0m\phi\alpha_0^2}L_S(\Wb^{(0)})+   \frac{2n^{5-2p}\big(L_S^{1-2p}(\Wb^{(0)}) -\EE\big[L_S^{1-2p}(\Wb^{(k)})\big]\big)}{(1-2p)c_0m\phi\alpha_1^2}. 
\end{align}
Similar to the proof for gradient descent, when $p=1/2$, it is easy to show that
\begin{align}\label{eq:peq1/2_sgd}
k\eta &\le \frac{2n^5}{c_0m\phi\alpha_0^2}L_S(\Wb^{(0)})+   \frac{2n^{5-2p}\big(\log\big(L_S(\Wb^{(0)})\big) -\EE\big[\log\big(L_S(\Wb^{(k)})\big)\big]\big)}{c_0m\phi\alpha_1^2}.
\end{align}
Set the step size $\eta=O(\phi M^{-1}n^{2p-5})$ and plug the results in Lemma \ref{lemma:azuma} into \eqref{eq:pneq1/2_sgd}, with probability at least $1-\delta$, there exists an absolute constant $c_1$ such that 
\begin{align*}
k\eta - \frac{\sqrt{k\log(1/\delta)}c_1n^{6-2p}L^4\eta}{c_0B\phi}\le \frac{2n^5}{c_0m\phi\alpha_0^2}L_S(\Wb^{(0)})+   \frac{2n^{5-2p}\big(L_S^{1-2p}(\Wb^{(0)}) -L_S^{1-2p}(\Wb^{(k)})\big)}{(1-2p)c_0m\phi\alpha_1^2}
\end{align*}
when $p\neq 1/2$ and 
\begin{align*}
k\eta - \frac{\sqrt{k\log(1/\delta)}c_1n^{6-2p}L^4\eta}{c_0B\phi}\le \frac{2n^5}{c_0m\phi\alpha_0^2}L_S(\Wb^{(0)})+   \frac{2n^{5-2p}\big(\log\big(L_S(\Wb^{(0)})\big) -\log\big(L_S(\Wb^{(k)})\big)\big)}{c_0m\phi\alpha_1^2}, 
\end{align*}
when $p = 1/2$. Then it requires to set $k = \tilde \Omega(n^{12-4p}B^{-2}\phi^{-2})$ to make the L.H.S. of the above two inequalities be lower bounded by $c_2k\eta$, where $c_2\in(0,1)$ is an absolute constant. Moreover, in order to guarantee that $L_S(\Wb^{(k)})\le \epsilon$, it suffices to set the quantity $k\eta$ as follows,
\begin{align}\label{eq:convergence_sgd_keta}
k\eta=\left\{\begin{array}{ll}
\tilde \Omega\big(n^5/(m\phi\alpha_0^2)+n^{5-2p}/(m\phi)\big)& 0\le p <\frac{1}{2}\\
\tilde \Omega\big(n^5/(m\phi\alpha_0^2)\big)+\tilde \Omega(n^4/(m\phi)\big)\cdot \Omega\big(\log(1/\epsilon))& p=\frac{1}{2}\\
\tilde \Omega\big(n^5/(m\phi\alpha_0^2)+n^{5-2p}\epsilon^{1-2p}/(m\phi)\big)& \frac{1}{2}<p \le 1.
\end{array} \right.  
\end{align}
Similar to the proof of Lemma \ref{lemma:gd_converge}, by setting $\alpha_0 = \infty$ or $\alpha_0 = O(1)$, we are able to complete the proof.

\subsection{Proof of Lemma \ref{lemma:upperbound_tau_sgd}}
\begin{proof}[Proof of Lemma \ref{lemma:upperbound_tau_sgd}]
We first prove that during the stochastic gradient descent, the derivative $|\ell'(y_i\hat y_i^{(k)})|$ is upper bounded by some constant for all $i$. According to Assumption \ref{assump:derivative_loss_upper}, it is clear that when $\rho_0<\infty$, $|\ell'(y_i\hat y_i^{(k)})|$ is spontaneously upper bounded by the constant $\rho_0$. Therefore, the rest proofs are established by assuming $\alpha_0 = \rho_0 = \infty$. Then by \eqref{eq:pneq1/2_sgd} and \eqref{eq:peq1/2_sgd}, we have
\begin{align}\label{eq:decrease_kstep_neq}
\frac{\EE\big[L_S^{1-2p}(\Wb^{(k)})\big]}{1-2p} \le \frac{L_S^{1-2p}(\Wb^{(0)})}{1-2p} - \frac{c_0m\phi\alpha_1^2k\eta}{2n^{5-2p}},
\end{align}
when $p\neq 1/2$ and 
\begin{align}\label{eq:decrease_kstep_eq}
\EE\big[\log\big(L_S(\Wb^{(k)})\big)\big] \le \log\big(L_S(\Wb^{(0)})\big) - \frac{c_0m\phi\alpha_1^2k\eta}{2n^{5-2p}},
\end{align}
when $p = 1/2$.
We tackle these two cases separately.

\textbf{Case $p> 1/2$: }
Then by Lemma \ref{lemma:azuma}, with probability at least $1-\delta$, there exists an absolute constant $c_1$ such that the following holds
\begin{align*}
\frac{L_S^{1-2p}(\Wb^{(k)})}{1-2p}\le \frac{\EE\big[L_S^{1-2p}(\Wb^{(k)})\big]}{1-2p}+\sqrt{2k\log(1/\delta)}\frac{c_1nL^4m\eta}{B}.
\end{align*}
Combining with \eqref{eq:decrease_kstep_neq}, the following holds with probability at least $1-\delta$
\begin{align}\label{eq:bound_grad_sgd1}
\frac{L_S^{1-2p}(\Wb^{(k)})}{1-2p}&\le \frac{L_S^{1-2p}(\Wb^{(0)})}{1-2p}- \frac{c_0m\alpha_0^2\phi k\eta}{2n^{5-2p}}+\sqrt{2k\log(1/\delta)}\frac{C_2nL^4m\eta}{B}\notag\\
&\le \frac{L_S^{1-2p}(\Wb^{(0)})}{1-2p} + \frac{\log(1/\delta)c_1^2L^8mn^{7-2p}\eta}{c_0B^2\alpha_0^2\phi}.
\end{align}
Note that $L_S(\Wb^{(0)}) = \tilde O(1)$, then we can choose the step size $\eta = O(\phi m^{-1}n^{2p-7}L^{-8}B^{2})$ such that the second term on the R.H.S. of \eqref{eq:bound_grad_sgd1} is upper bounded by $|L_S^{1-2p}(\Wb^{(0)})/(2-4p)|$.  Therefore, with probability at least $1-\delta$,
\begin{align*}
\frac{L_S^{1-2p}(\Wb^{(k)})}{1-2p}\le \max\bigg\{\frac{L_S^{1-2p}(\Wb^{(0)})}{2-4p},\frac{3L_S^{1-2p}(\Wb^{(0)})}{2-4p}\bigg\},
\end{align*}
which is equivalent to
\begin{align*}
L_S(\Wb^{(k)})\le \max\bigg\{\frac{1}{2^{1/(1-2p)}},\bigg(\frac{3}{2}\bigg)^{1/(1-2p)}\bigg\} L_S(\Wb^{(0)}) .
\end{align*}
Since we have $p\in [0,1/2)\cup(1/2,1]$,  and $L_S(\Wb^{(0)}) = \tilde O(1)$, $L_S(\Wb^{(k)})=\tilde O(1)$ holds with probability at least $1-\delta$.

\textbf{Case $p = 1/2$: }
By Lemma \ref{lemma:azuma}, with probability at least $1-\delta$, 
\begin{align}
\log(L_S(\Wb^{(k)})) \le  \EE\big[\log(L_S(\Wb^{(k)}))\big]+\sqrt{2k\log(1/\delta)}\frac{c_1nL^4m\eta}{B}
\end{align}
Combining with \eqref{eq:decrease_kstep_eq}, we can set $\eta = O(\phi m^{-1}n^{2p-7}L^{-8}B^{2})$, and thus 
\begin{align*}
\log(L_S(\Wb^{(k)})) \le \max\{1/2\log(L_S(\Wb^{(0)})),2\log(L_S(\Wb^{(0)}))\}
\end{align*}
holds with probability at least $1-\delta$. Therefore, with probability at least $1-\delta$, we have
\begin{align*}
L_S(\Wb^{(k)})\le   \max\big\{ L_S^2(\Wb^{(0)}), L_S^{1/2}(\Wb^{(0)}) \big\}  = \tilde O(1).
\end{align*}
Now we have proved that $L_S(\Wb^{(k)}) = \tilde O(1)$ with probability at least $1-\delta$ for one particular $k$. Then take union bound, we have $L_S(\Wb^{(k)}) = \tilde O(1)$ for all $k\le K$ with probability at least $1-K\delta$.

Moreover, when $\rho_0 = \infty$, we have
\begin{align*}
-\sum_{i\in\cB^{(k)}}\ell'(y_i\hat y_i^{(k)})\le \rho_1\sum_{i\in\cB^{(k)}}\ell^p(y_i\hat y_i^{(k)})\le \rho_1 n L_S^p(\Wb^{(k)}) = \tilde O(n),
\end{align*}
where the second inequality is due to $ \sum_{i=1}^n \ell^p(x)\le n \big(1/n\sum_{i=1}^n\ell(x)\big)^p$ when $p\le 1$. By Lemma \ref{lemma:upper_grad},  we first assume that $\Wb^{(k)}$ is in the perturbation region centering at $\Wb^{(0)}$ with radius $\tau$, and obtain the following upper bound for the stochastic gradient $\Gb_l^{(k)}$, 
\begin{align*}
\|\Gb_l^{(k)}\|_2\le -\frac{CL^2M^{1/2}}{B}\sum_{i\in\cB^{(k)}}\ell'(y_i\hat y_i^{(k)}) = \tilde O(L^2M^{1/2}n),
\end{align*}
where $C$ is an absolute constant.
Then we finalize the proof based on two cases: $\alpha_0=\infty$ and $\alpha_0<\infty$. When $\alpha_0<\infty$,
since 
\begin{align*}
\|\Wb_l^{(k)} - \Wb_l^{(0)}\|_2\le k\eta \sum_{t=0}^{k-1}\|\Gb_l^{(k)}\|_2,
\end{align*}
by induction it is easy to show that there exists $T =  \tilde O(\tau L^{-2}M^{-1/2}n^{-1})$ such that for any $k$ and $\eta = O(\phi m^{-1}n^{2p-7}L^{-8}B^{2})$ satisfying $k\eta \le T$, it holds that $\|\Wb_l^{(k)}-\Wb^{(0)}\|_2\le \tau$ for any $l$, i.e., $\Wb^{(k)}$ is in the preset perturbation region with radius $\tau$. When $\alpha_0<\infty$, we have $\|\Gb_l^{(k)}\|_2\le -C_1L^2M^{1/2}\rho_0 = O(L^2M^{1/2})$, where $C_1$ is an absolute constant. Thus, it can be show that for any $k$ and $\eta$ satisfying $k\eta\le T = O(\tau L^{-2}M^{-1/2})$, we have $\|\Wb_l^{(k)}-\Wb_l^{(0)}\|_2\le \tau$. This completes the proof.
\end{proof}

 \end{proof}

\section{Proof of Auxiliary Lemmas}

\subsection{Proof of Lemma \ref{lemma:support1}}
\begin{proof}[Proof of Lemma \ref{lemma:support1}]
It is simple to verify that function $f(x) = x^{1-2p}/(1-2p)$ is concave when $p\in[0,1/2)\cup(1/2,1]$. Thus we have
\begin{align*}
\frac{a^{1-2p}-b^{1-2p}}{1-2p} = f(a) - f(b)\le f'(b)(a-b) = b^{-2p}(a-b).
\end{align*}
This completes the proof.
% Note that we aim to prove
% \begin{align}\label{eq:temp0001}
% \frac{a-b}{b^{2p}}\ge \frac{a^{1-2p}-b^{1-2p}}{1-2p},
% \end{align}
% which is equivalent to the following when $p>1/2$
% \begin{align*}
% (2p-1)(b-a)\le  a^{1-2p}b^{2p} - b.
% \end{align*}
% Note that $f(x) = x^{2p-1}$ is a concave function, thus it follows that $b^{2p-1}-a^{2p-1}\ge (2p-1)b^{2p-2}(b-a)$, which yields
% \begin{align*}
% (2p-1)(b-a)\le b-a^{2p-1}b^{2-2p} \le \frac{b^{2p-1}}{a^{2p-1}}\big(b-a^{2p-1}b^{2-2p}\big) = a^{1-2p}b^{2p} - b,
% \end{align*}
% where the second inequality is due to $b\ge a$ and $2p-1>0$. 
% When $p < 1/2$, \eqref{eq:temp0001} is equivalent to 
% \begin{align*}
%  (2p-1)(b-a)\ge  a^{1-2p}b^{2p} - b.  
% \end{align*}
% Based on the convexity of function $f(x) = x^{2p-1}$, we have
% \begin{align*}
% (2p-1)(b-a)\le b^{2p-1}a^{2-2p}-a \ge b-a^{2p-1}b^{2-2p}\ge \frac{b^{2p-1}}{a^{2p-1}}\big(b^{2p-1}a^{2-2p}-a\big) = a^{1-2p}b^{2p} - b ,   
% \end{align*}
% where the second inequality is due to $b\ge a$ and $2p-1<0$. Thus we complete the proof.
\end{proof}

\subsection{Proof of Lemma \ref{lemma:bound_delta_i}}
\begin{proof}[Proof of Lemma \ref{lemma:bound_delta_i}]
The upper bound of $|\Delta_i^{(k)}|$ can be derived straightforwardly. By \ref{item:difference_xli} in Theorem \ref{thm:perturbation}, we know that there exists a constant $C_1$ such that
\begin{align*}
\|\xb_{L,i}^{(k+1)} - \xb_{L,i}^{(k)}\|_2\le C_1L\cdot\sum_{l=1}^L\|\Wb_l^{(k+1)} - \Wb_l^{(k)}\|_2 = C_1L\eta \sum_{l=1}^L \big\|\nabla_{\Wb_l}[L_S(\Wb^{(k)})]\big\|_2.
\end{align*}
By \ref{item:grad_upperbound} in Theorem \ref{thm:perturbation}, we further have
\begin{align}\label{eq:upper_iteratedifference_output_layer}
\|\xb_{L,i}^{(k+1)} - \xb_{L,i}^{(k)}\|_2\le -\frac{C_2L^4M^{1/2}\eta}{n}\sum_{i=1}^n\ell'(y_i\hat y_i^{(k)}),    
\end{align}
where $C_2$ is an absolute constant. Therefore, it follows that
\begin{align*}
|\Delta_i^{(k)}| = |y_i\vb^\top(\xb_{L,i}^{(k+1)} - \xb_{L,i}^{(k)})|\le \|\vb\|_2\|\xb_{L,i}^{(k+1)} - \xb_{L,i}^{(k)}\|_2\le -\frac{C_2L^4M\eta}{n}\sum_{i=1}^n\ell'(y_i\hat y_i^{(k)}),
\end{align*}
where we use the fact that $\|\vb\|_2 \le M^{1/2}$.

In what follows we are going to prove the lower bound of $\Delta_i^{(k)}$. Based on the definition of $\Delta_i^{(k)}$, we have
\begin{align}\label{eq:lowerbound_delta}
\Delta_i^{(k)} &= y_i\vb^\top\bigg( \prod_{l=1}^L \bSigma_{l,i}^{(k+1)}\Wb_{l}^{(k+1)\top}\bigg)\xb_i -y_i\vb^\top\bigg( \prod_{l=1}^L \bSigma_{l,i}^{(k)}\Wb_{l}^{(k)\top}\bigg)\xb_i\notag\\
&=y_i\vb^\top \sum_{l=1}^L \bigg(\prod_{r=l+1}^L\bSigma_{r,i}^{(k)}\Wb_{r}^{(k)\top}\bigg)\big(\bSigma_{l,i}^{(k+1)}\Wb_l^{k+1} - \bSigma_{l,i}^{(k)}\Wb_{l}^{(k)\top}\big)\xb_{l-1,i}^{(k+1)}\notag\\
% &\qquad+y_i\vb^\top \sum_{l=1}^L \bigg(\prod_{r=l+1}^L\bSigma_{r,i}^{(k)}\Wb_{r}^{(k)\top}\bigg) \bSigma_{l,i}^{(k)}(\Wb_{l}^{(k+1)\top}-\Wb_{l}^{(k)\top})\bigg(\prod_{r=1}^{l-1}\bSigma_{r,i}^{(k+1)}\Wb_r^{(k+1)\top}\bigg)\xb_i\notag\\
&=y_i \sum_{l=1}^L \underbrace{\vb^\top\bigg(\prod_{r=l+1}^L\bSigma_{r,i}^{(k)}\Wb_{r}^{(k)\top}\bigg)(\bSigma_{l,i}^{(k+1)} - \bSigma_{l,i}^{(k)})\Wb_{l}^{(k+1)\top}\xb_{l-1,i}^{(k+1)}}_{I_{l,i}^1}\notag\\
&\qquad+y_i \sum_{l=1}^L \underbrace{\vb^\top\bigg(\prod_{r=l+1}^L\bSigma_{r,i}^{(k)}\Wb_{r}^{(k)\top}\bigg) \bSigma_{l,i}^{(k)}(\Wb_{l}^{(k+1)\top}-\Wb_{l}^{(k)\top})\xb_{l-1,i}^{(k+1)}}_{I_{l,i}^2}.
\end{align}
We first tackle $I_{l,i}^1$. Note that
% \begin{align*}
% (\bSigma_{l,i}^{(k+1)} - \bSigma_{l,i}^{(k)})\Wb_{l}^{(k+1)\top}\xb_{l-1,i}^{(k+1)}=  \tilde \bSigma_{l,i}^{(k+1)}(\bSigma_{l,i}^{(k+1)} - \bSigma_{l,i}^{(k)})\Wb_{l}^{(k+1)\top}\xb_{l-1,i}^{(k+1)},
% \end{align*}
% where $\tilde \bSigma_{l,i}^{k+1}$ denotes a diagonal matrix whose entries are the absolute value of the corresponding entries in $(\bSigma_{l,i}^{(k+1)} - \bSigma_{l,i}^{(k)})$.
% Moreover, note that the nonzero entries in $(\bSigma_{l,i}^{(k+1)} - \bSigma_{l,i}^{(k)})$ represent the nodes in the $l$-th layer whose signs are changed in the $k$-th iteration. 
% %Without loss of generality, we assume the $j$-th node whose sign is changed in the $k$-th iteration, thus 
% If the sign of the $j$-th node is changed in the $k$-th iteration, then the corresponding output %of the $j$-th node after iteration 
% must be smaller than $|\big(\Wb_{l}^{(k+1)\top}\xb_{l,i}^{(k+1)}-\Wb_{l}^{(k)\top}\xb_{l,i}^{(k)}\big)_j|$ before passing through the ReLU activation. Therefore 
\begin{align}\label{eq:zou111}
\big\|(\bSigma_{l,i}^{(k+1)} - \bSigma_{l,i}^{(k)})\Wb_{l}^{(k+1)\top}\xb_{l-1,i}^{(k+1)}\big\|_2&\le \|\Wb_{l}^{(k+1)\top}\xb_{l-1,i}^{(k+1)}-\Wb_{l}^{(k)\top}\xb_{l-1,i}^{(k)}\|_2\notag\\
&\le \|\Wb_l^{(k+1)\top} - \Wb_l^{(k)\top}\|_2\|\xb_{l-1,i}^{(k+1)}\|_2 +\|\Wb_{l}^{(k)}\|_2\|\xb_{l-1,i}^{(k+1)} - \xb_{l-1,i}^{(k)}\|_2\notag\\
&= \eta\|\nabla_{\Wb_l}[L_S(\Wb^{(k)})]\|_2\|\xb_{l-1,i}^{(k+1)}\|_2 +\|\Wb_{l}^{(k)}\|_2\|\xb_{l-1,i}^{(k+1)} - \xb_{l-1,i}^{(k)}\|_2\notag\\
&\le -\frac{C_3L^4M^{1/2}\eta}{n}\sum_{i=1}^n \ell'(y_i\hat y_i^{(k)}),
\end{align}
where $C_3$ is an absolute constant and the last inequality follows from \ref{item:grad_upperbound} in Theorem \ref{thm:perturbation} and \eqref{eq:upper_iteratedifference_output_layer}.
Moreover, according to Lemma \ref{item:perturbe_lip_sparse}, for any  vector $\ab$ with $\|\ab\|_0 \le s$, the following holds, 
\begin{align*}
\bigg\|\vb^\top\bigg(\prod_{r=l+1}^L\bSigma_{r,i}^{(k)}\Wb_{r}^{(k)\top}\bigg)\ab\bigg\|_2\le C_4L^{5/3}\tau^{1/3}\cdot\sqrt{M\log(M)}\cdot\|\ab\|_2,
\end{align*}
where $C_4$ is an absolute constant. Then let $\ab = (\bSigma_{l,i}^{(k+1)} - \bSigma_{l,i}^{(k)})\Wb_{l}^{(k+1)\top}\xb_{l-1,i}^{(k+1)}$ and apply \eqref{eq:zou111}, we get the following bound of $| I_{l,i}^1|$
\begin{align}\label{eq:bound_I1}
| I_{l,1}^1|\le-\frac{ C_5L^{17/3}\tau^{1/3}\eta\cdot\sqrt{M^2\log(M)}}{n}\sum_{i=1}^n\ell'(y_i\hat y_i^{(k)}),
\end{align}
where $C_5 = C_3\cdot C_4$ is an absolute constant. Then we pay attention to $I_{l,i}^2$. Using $\Wb_l^{(k+1)} - \Wb_l^{(k)} = -\eta \nabla_{\Wb_l}[L_S(\Wb^{(k)})]$, we have
% $\xb_{l-1,i}^{(k+1)}$ by $\xb_{l-1,i}^{(k+1)} = \xb_{l-1,i}^{(k)}+(\xb_{l-1,i}^{(k+1)}-\xb_{l-1,i}^{(k)})$, and obtain
\begin{align}\label{eq:bound_I2}
I_{l,i}^2 &= 
\underbrace{-\eta\vb^\top\bigg(\prod_{r=l+1}^L\bSigma_{r,i}^{(k)}\Wb_{r}^{(k)\top}\bigg) \bSigma_{l,i}^{(k)}\big(\nabla_{\Wb_l}[L_S(\Wb^{(k)})]\big)^\top(\xb_{l-1,i}^{(k+1)}-\xb_{l-1,i}^{(k)})}_{I_{l,i}^3}, \notag\\
&\qquad-  \underbrace{\eta\vb^\top \bigg(\prod_{r=l+1}^L\bSigma_{r,i}^{(k)}\Wb_{r}^{(k)\top}\bigg) \bSigma_{l,i}^{(k)}\big(\nabla_{\Wb_l}[L_S(\Wb^{(k)})]\big)^\top\xb_{l-1,i}^{(k)}}_{I_{l,i}^4}。 
\end{align}
% where we use the fact that $\Wb_l^{(k+1)} - \Wb_l^{(k)} = -\eta\nabla_{\Wb_l}L_S(\Wb^{(k)})$.
%In what follows, we first pay attention to 
We now proceed to bound $\| I_{l,i}^3 \|_2$. Based on Lemma \ref{lemma:upper_grad} and \eqref{eq:upper_iteratedifference_output_layer}, we have
\begin{align}\label{eq:bound_I3}
|I_{l,i}^3|&\le \eta\|\vb\|_2\bigg\|\prod_{r=l+1}^L\bSigma_{r,i}^{(k)}\Wb_{r}^{(k)\top}\bigg\|_2\|\nabla_{\Wb_l}[L_S(\Wb^{(k)})]\|_2\|\xb_{l-1,i}^{(k+1)}-\xb_{l-1,i}^{(k+1)}\|_2\notag\\
&\le \frac{C_8L^7M^{3/2}\eta^2}{n^2}\bigg(\sum_{i=1}^n\ell'(y_i\hat y_i^{(k)})\bigg)^2,
% &\le \frac{2^{9L}M\eta^2}{7n^2}\bigg(\sum_{i=1}^n\ell'(y_i\hat y_i^{(k)})\bigg)^2\notag\\
% &\le -\frac{2^{9L}M \rho\eta^2}{7n}\sum_{i=1}^n\ell'(y_i\hat y_i^{(k)}),
\end{align}
where $C_6$ is an absolute constant.
Moreover, note that $\|\Wb_l^{(k)} - \Wb_l^{(0)}\|_2\le \tau$ holds for all $l$, by \ref{item:difference_xli} in Theorem \ref{thm:perturbation}, we have
\begin{align*}
|\hat y_i^{(k)}| \le \|\vb\|_2\|\xb_{L,i}^{(k)}-\xb_{L,i}^{(0)}\|_2+|\hat y_i^{(0)}|\le C_7\big(L^2M^{1/2}\tau + \sqrt{\log(n/\delta)}\big),
\end{align*}
where the last inequality follows from \ref{ThmResult:randinit_outputbound} in Theorem \ref{thm:randinit}, and $C_7$ is an absolute constant.
Therefore, by Assumption \ref{assump:smooth}, we have
\begin{align*}
|\ell'(y_i\hat y_i^{(k)})|\le|\ell'(0)|+\lambda |\hat y_i^{(k)}| \le C_7\lambda\big(L^2M^{1/2}\tau+\sqrt{\log(n/\delta)}\big) + |\ell'(0)|.
\end{align*}
Assume $M^{1/2}\ge \sqrt{\log{n/\delta}}$ and $\tau \le 1$, plugging the above inequality into \eqref{eq:bound_I3}, we obtain
\begin{align}\label{eq:bound_I3_2}
|I_{l,i}^3|\le -\frac{C_{8}L^9M^2\eta^2}{n}\sum_{i=1}^n\ell'(y_i\hat y_i^{(k)}),
\end{align}
where $C_{10}$ is an absolute constant.
%We keep the term $I_{l,i}^4$ in the final result since it is important to make the loss function decrease in each iteration of gradient descent. Therefore, 
Finally, plugging \eqref{eq:bound_I1}, \eqref{eq:bound_I2} and \eqref{eq:bound_I3_2} into \eqref{eq:lowerbound_delta}, we have
\begin{align*}
\Delta_i^{(k)} &= y_i\vb^\top\sum_{i=1}^L(I_{l,i}^1+I_{l,i}^3+I_{l,i}^4)\\
&\ge \sum_{i=1}^L \big(-|\vb^\top I_{l,i}^1|-\|\vb\|_2\|I_{l,3}\|_2+y_i\vb^\top I_{l,i}^4\big)\\
&\ge \frac{C_7L^{17/3}\tau^{1/3}M\eta\cdot\sqrt{\log(M)}+C_8L^9M^2\eta^2}{n}\sum_{i=1}^n\ell'(y_i\hat y_i^{(k)}) + y_i I_{l,i}^4.
% &\ge -\frac{2^{9L+3}LC_0\tau^{1/3}M\sqrt{\log(M)}\eta}{n}\sum_{i=1}^n\ell'(y_i
% \hat y_i^{(k)})+\frac{2^{9L}LM^{3/2} \rho\eta^2}{7n}\sum_{i=1}^n\ell'(y_i\hat y_i^{(k)})+\sum_{i=1}^Ly_i\vb^\top I_{l,i}^4\\
% &\ge \bigg(\frac{2^{9L+3}LC_0\tau^{1/3}M\sqrt{\log(M)}\eta}{n}+\frac{2^{9L}LM^{3/2} \rho\eta^2}{7n}\bigg)\sum_{i=1}^n\ell'(y_i\hat{ y}_i^{(k)})+\sum_{i=1}^Ly_i\vb^\top I_{l,i}^4.
\end{align*}
Let $\ub_{l,i} = I_{l,i}^4$ we complete the proof.
\end{proof}

\subsection{Proof of Lemma \ref{lemma:bound_delta_i_sgd}}
\begin{proof}[Proof of Lemma \ref{lemma:bound_delta_i_sgd}]
This lemma can be proved by following the same technique for proving Lemma \ref{lemma:bound_delta_i}, while we only need to replace the upper bound of $\|\nabla_{\Wb_l}[L_S(\Wb^{(k)})]\|_2$ with the stochastic gradient $\|\Gb_l^{(k)}\|$ based on \ref{item:grad_upperbound} in Theorem \ref{thm:perturbation}. Since the proof technique of this lemma is essentially identical to that of Lemma \ref{lemma:bound_delta_i}, we omit the detail here.

\end{proof}
% \subsection{Proof of Lemma \ref{lemma:bound_delta_i_sgd}}
% \begin{proof}[Proof of Lemma \ref{lemma:bound_delta_i_sgd}]
% Similar to the proof of Lemma \ref{lemma:boun d_difference_xl}, it is easy to derive the following upper bound of $\|\xb_{l,i}^{(k+1)} - \xb_{l,i}^{(k)}\|_2$ when performing stochastic gradient,
% \begin{align}\label{eq:bound:xl_sgd}
% \|\xb_{l,i}^{(k+1)} - \xb_{l,i}^{(k)}\|_2\le -\frac{2^{5L}M^{1/2}\eta}{7B}\sum_{i\in\cB^{(k)}}\ell'(y_i\hat y_i^{(k)}),    
% \end{align}
% where $B$ denotes the minibatch size and $\cB^{(k)}$ denotes the set of queried indices for computing the stochastic gradient in the $k$-th iteration.
% Then we have the following upper bound of $|\Delta_{i}^{(k)}|$,
% \begin{align*}
% |\tilde \Delta_{i}^{(k)}| = \big\vert y_i\vb^\top (\xb_{L,i}^{(k+1)} - \xb_{L,i}^{(k)})\big\vert \le -\frac{2^{5L}M\eta }{7B}\sum_{i\in\cB^{(k)}}\ell'(y_i\hat y_i^{(k)}).
% \end{align*}
% For the lower bound of $\tilde \Delta_i^{(k)}$, we can prove the results in this lemma by following the same way as the proof of Lemma \ref{lemma:boun d_difference_xl}, while the only difference is that we need to use \eqref{eq:bound:xl_sgd} to replace the previous upper bound of $\|\xb_{l,i}^{(k+1)} - \xb_{l,i}^{(k)}\|_2$. Since the rest proofs are essentially identical, we omit the detail here.  
% \end{proof}
\subsection{Proof of Lemma \ref{lemma:support2}}
\begin{proof}[Proof of Lemma \ref{lemma:support2}]For random variables $u_1,\dots,u_n$, we have
\begin{align*}
    \EE \bigg[\bigg( \frac{1}{B}\sum_{i \in \cB} u_i \bigg)^2\bigg] &= \frac{1}{B^2} \EE\bigg[ \sum_{ i\neq i',\{i,i'\} \in \cB} u_iu_{i'}\bigg] + \frac{1}{B}\EE\big[u_i^2\big]\notag\\
    &= \frac{B-1}{Bn(n-1)} \EE\bigg[\sum_{i\neq i'} u_iu_{i'}\bigg] + \frac{1}{B}\EE\big[u_i^2\big]\notag\\
    &= \frac{B-1}{Bn(n-1)} \EE\bigg[\bigg(\sum_{i=1}^n u_i\bigg)^2\bigg] -\frac{B-1}{B(n-1)}\EE\big[u_i^2\big] + \frac{1}{B}\EE\big[u_i^2\big]\notag\\
    &= \frac{n-B}{B(n-1)}\EE\big[u_i^2\big]\\
    &\le \frac{1}{B}\EE\big[u_i^2\big],
\end{align*}
where the last equality is due to the fact that $\frac{1}{n}\sum_{i=1}^n u_i  = 0$. 
\end{proof}

\subsection{Proof of Lemma \ref{lemma:azuma}}

\begin{proof}[Proof of Lemma \ref{lemma:azuma}]
We prove this Lemma by two cases: $p\neq 1/2$ and $p=1/2$. 

\textbf{Case $p\neq 1/2$:} By \eqref{eq:loss_decrease_onestep_sgd}, we have
\begin{align*}
L_S(\Wb^{(k+1)}) - L_S(\Wb^{(k)})\le \frac{1}{n}\sum_{i=1}^n\bigg[\ell'(y_i\hat y_i^{(k)})\tilde \Delta_i^{(k)}+\frac{\lambda}{2}(\tilde \Delta_i^{(k)})^2\bigg].
\end{align*}
Using the upper bound of $\tilde \Delta_i^{(k)}$ provided in Lemma \ref{lemma:bound_delta_i_sgd}, we have
\begin{align}\label{eq:upperbound_loss_diff}
L_S(\Wb^{(k+1)}) - L_S(\Wb^{(k)})\le \frac{ C_1L^4m\eta}{nB}\bigg(\sum_{i=1}^n\ell'(y_i\hat y_i^{(k)})\bigg)^2\le \frac{C_2nL^4m\eta}{B} L_S^{2p}(\Wb^{(k)}),
\end{align}
where $ C_1$ and $C_2$ are absolute constant and the last inequality follows from the fact that $1/n\sum_{i=1}^nz_i^p\le(1/n\sum_{i=1}^nz_i)^p$ for any $z_1,\dots,z_n\ge0$.
Dividing by $L_S^{2p}(\Wb^{(k)})$ on both sides and applying Lemma \ref{lemma:support1}, 
\begin{align*}
\frac{L_S^{1-2p}(\Wb^{(k+1)})}{1-2p} - \frac{L_S^{1-2p}(\Wb^{(k)})}{1-2p}\le  \frac{C_2nL^4m\eta}{B}.  
\end{align*}
Then by Azuma's inequality, we have with probability at least $1-\delta$
\begin{align*}
\frac{L_S^{1-2p}(\Wb^{(k)})}{1-2p}\le \frac{\EE\big[L_S^{1-2p}(\Wb^{(k)})\big]}{1-2p}+\sqrt{2k\log(1/\delta)}\frac{C_2nL^4M \eta}{B}.
\end{align*}

\textbf{Case $p=1/2$:} 
Similar to \eqref{eq:upperbound_loss_diff}, we can derive the following upper bound on the difference $L_S(\Wb^{(k+1)}) - L_S(\Wb^{(k)})$,
\begin{align*}
L_S(\Wb^{(k+1)}) - L_S(\Wb^{(k)})\le \frac{C_2nL^4m\eta }{B}L_S(\Wb^{(k)}) ,
\end{align*}
which leads to following by using inequality $\log(x)\le x-1$,
\begin{align*}
\log(L_S(\Wb^{(k+1)})) - \log(L_S(\Wb^{(k)}))\le \log\bigg(1 + \frac{C_2nL^4m\eta}{B}\bigg)\le \frac{C_2nL^4m\eta}{B}.
\end{align*}
Then by Azuma's inequality, we have with probability at least $1-\delta$
\begin{align*}
\log(L_S(\Wb^{(k)})) \le  \EE\big[\log(L_S(\Wb^{(k)}))\big]+\sqrt{2k\log(1/\delta)}\frac{C_2nL^4m\eta}{B}.
\end{align*}
This completes the proof.
\end{proof}